\RequirePackage{rotating}

\documentclass{article}

\usepackage{arxiv}

\usepackage[utf8]{inputenc} 
\usepackage[T1]{fontenc}    
\usepackage{hyperref}       
\usepackage{url}            
\usepackage{booktabs}       
\usepackage{amsfonts}       
\usepackage{nicefrac}       
\usepackage{microtype}      
\usepackage{bm}
\usepackage{tabularx}
\usepackage{colortbl}
\usepackage{graphicx}
\usepackage{subcaption}
\usepackage{stfloats}
\usepackage{dsfont}
\usepackage{relsize}

\usepackage{amsmath}
\usepackage{amssymb}
\usepackage{amsthm}
\usepackage{thmtools, thm-restate}
\usepackage{bbm}
\usepackage[numbers]{natbib}

\newcommand{\subheader}{\noindent\textbf}
\newcommand{\D}{\mathcal{D}}
\newcommand{\N}{\mathbb{N}}
\newcommand{\R}{\mathbb{R}}
\newcommand{\vopt}{v_{\operatorname{opt}}}

\newcommand{\vf}{v_{\operatorname{first}}}
\newcommand{\vl}{v_{\operatorname{last}}}
\newcommand{\rij}{r_{i\leftrightarrow j}}

\newcommand{\pexp}{p_\text{Exp}}
\newcommand{\dtd}{\emph{DTD}}
\newcommand{\dtr}{\emph{DTR}}
\newcommand{\dir}{\emph{DIR}}
\newcommand{\did}{\emph{DID}}
\newcommand{\er}{\emph{ER}}
\newcommand{\ed}{\emph{ED}}
\newcommand{\rnd}{\emph{rND}}
\newcommand{\rkl}{\emph{rKL}}
\newcommand{\rrd}{\emph{rRD}}
\newcommand{\awrf}{\emph{AWRF}}
\newcommand{\psp}{\emph{PSP}}

\newcommand{\cross}{$\textcolor{red}{\times}$}
\newcommand{\gcheck}{$\textcolor{green}{\checkmark}$}

\newtheoremstyle{thrm}
  {\topsep} 
  {3pt} 
  {\itshape} 
  {} 
  {\bfseries} 
  {.} 
  {.5em} 
  {} 

\theoremstyle{thrm} 
\newtheorem{theorem}{Theorem}
\newtheorem{proposition}{Proposition}[section]

\newtheoremstyle{style}
  {\topsep} 
  {\topsep} 
  {\itshape} 
  {} 
  {\bfseries} 
  {.} 
  {.5em} 
  {} 

\theoremstyle{style} 
\newtheorem*{definition}{Definition}

\makeatletter
\renewenvironment{proof}[1][\proofname]{\par
  \vspace{-\topsep}
  \pushQED{\qed}%
  \normalfont
  \topsep0pt \partopsep0pt 
  \trivlist
  \item[\hskip\labelsep
        \itshape
    #1\@addpunct{.}]\ignorespaces
}{%
  \popQED\endtrivlist\@endpefalse
  \addvspace{3pt plus 6pt} 
}
\makeatother

\newenvironment{talign}
 {\align}
 {\endalign}
\newenvironment{talign*}
 {\csname align*\endcsname}
 {\endalign}

\title{Properties of Group Fairness Metrics for Rankings}

\author{
  Tobias Schumacher \\
  University of Mannheim\\
  \texttt{tobias.schumacher@uni-mannheim.de} \\
   \And
  Marlene Lutz \\
  University of Mannheim\\
  \texttt{marlene.lutz@uni-mannheim.de} \\
  \And
  Sandipan Sikdar \\
  L3S Research Center \\
  Leibniz University Hannover\\
  \texttt{sikdar@l3s.de} \\
  \And
  Markus Strohmaier \\
  University of Mannheim,\\
  GESIS - Leibniz Institute for the Social Sciences, and\\
  Complexity Science Hub\\
  \texttt{markus.strohmaier@uni-mannheim.de}
}

\begin{document}

\maketitle

\begin{abstract}
In recent years, several metrics have been developed for evaluating group fairness of rankings. Given that these metrics were developed with different application contexts and ranking algorithms in mind, it is not straightforward which metric to choose for a given scenario. 
In this paper, we perform a comprehensive comparative analysis of existing group fairness metrics developed in the context of fair ranking. 
By virtue of their diverse application contexts, we argue that such a comparative analysis is not straightforward. 
Hence, we take an axiomatic approach whereby we design a set of thirteen properties for group fairness metrics that consider different ranking settings. 
A metric can then be selected depending on whether it satisfies all or a subset of these properties. 
We apply these properties on eleven existing group fairness metrics, and through both empirical and theoretical results we demonstrate that most of these metrics only satisfy a small subset of the proposed properties. 
These findings highlight limitations of existing metrics, and provide insights into how to evaluate and interpret different fairness metrics in practical deployment.
The proposed properties can also assist practitioners in selecting appropriate metrics for evaluating fairness in a specific application. 
\end{abstract}

\section{Introduction}

As recommender systems have become nearly ubiquitous in online systems and therefore in our daily lives, the issue of fairness in ranked recommendations has garnered increasing attention in research.
Recent literature has proposed a plethora of approaches \cite{celis2018ranking, celis2020interventions,bower2021individually, Castillo2019,Geyik2019,Gorantla2021,Morik2020,pitoura2021fairness,zehlike2020matching, Zehlike2022a, Zehlike2022b, Yang2020} that aim to ensure that rankings are not only accurate, but also fair. 
At the same time, existing metrics that have been developed to determine the \emph{quality} of rankings \cite{DCG_2002, kendall1938} are not directly applicable or useful for evaluating the \emph{fairness} of rankings. 
As a result, there is a growing body of research~\cite{yang2016measuring,Singh2018,sapiezynski2019quantifying,narasimhan2020pairwise}  that is dedicated to designing and deploying metrics for assessing the fairness of rankings.
In this work, we specifically focus on \emph{group fairness} metrics which are aimed at measuring if groups (where membership is based on the value of a sensitive attribute like race, gender etc.) as a whole receive similar outcomes from a ranking process. 
We assume the existence of a \emph{protected group} which is considered to have been discriminated against historically.

\noindent\textbf{Problem \& Objective.} 
Several group fairness metrics (cf. Section \ref{sec:ex_metrics})  have been developed considering a wide range of different ranking problems and contexts. 
Given their diverse application contexts, it is however unclear how they can be compared or evaluated. 
Hence, when presented with an application context, it is difficult to choose an appropriate group fairness metric. 
For similar evaluation problems in information retrieval, axiomatic analyses have often been preferred in the existing literature~\citep{amigo2018axiomatic, amigo2013general, amigo2019comparison, amigo2020nature, amigo2009comparison}.
To perform an axiomatic analysis, we first define a set of verifiable properties and then characterize the metrics according to whether they satisfy these properties.
Equipped with these properties, we systematically analyze and compare a set of eleven existing metrics for group fairness in rankings.

\noindent \textbf{Approach.} 
We perform an axiomatic analysis considering two different ranking settings, where (i) the entire population of candidates is ranked (e.g., ranking candidates for a job where the total number of candidates are limited), and (ii) only a subset of the population is retrieved to be ranked (reminiscent of a traditional information retrieval setup such as document search). 
We identify properties which are applicable to both settings as well as properties which apply individually to each of the two settings (cf. Section~\ref{sec:background}). We deploy these properties to compare existing group fairness metrics through both theoretical and empirical results.  
To the best of our knowledge, this is the first comprehensive axiomatic analysis of group fairness metrics.

\noindent \textbf{Results.} 
Our experimental and theoretical results demonstrate that most existing metrics for group fairness in rankings only satisfy a subset of our properties.
Particularly for the ranking setting in which the full population is ranked, the corresponding properties are hardly ever satisfied.
In contrast, most existing metrics satisfy the properties for the second setting, where only a subset of the population is ranked.
This points out certain limitations of existing metrics, and suggests a need for careful attention being paid to evaluating the adequacy of metrics for different scenarios.

\noindent \textbf{Contributions.} 
We present and formalize a comprehensive set of properties for the evaluation of group fairness metrics for rankings.  
Through theoretical and empirical results, we demonstrate that most of the existing metrics only satisfy a subset of these properties. 
Our work sheds light on current limitations of metrics for group fairness and can aid practitioners in reflecting about requirements in selecting appropriate fairness metrics for given situations. 

\section{Background}
\label{sec:background}  
In this section, we formally define the ranking problem and present the ranking settings that we consider in this paper. For our analysis, we mainly focus on group fairness metrics, which measure if all groups as a whole receive a similar outcome. Hence, we present a detailed overview of existing group fairness metrics as well.

\subsection{Preliminaries}

\subheader{The Ranking Problem.} 
In typical ranking problems, one is given a candidate population 
$\mathcal{D}$ and a set of queries $\mathcal{Q}$. 
For a given query $q\in\mathcal{Q}$, the goal is to rank a set of candidates $D:=D^q\subseteq\mathcal{D}$ with $n:=n(q):=|D|$.
Each candidate $d \in \D$ can be represented as a tuple $d=(\mathbf{x},y(\mathbf{x}))$  where $\mathbf{x} \in \mathbb{R}^p$ denotes a feature vector of $d$ that determines the match between the candidate and the query, and a ground truth relevance score $y(\mathbf{x}):=y(\mathbf{x},q) \in \mathbb{R}$ that quantifies its qualification with respect to the query $q$.
In practice, the relevance is however often unknown and needs to be estimated by some learning algorithm.
For simplicity, we will also denote the relevance of a candidate $d$ as $y(d)$.

A ranking $r := r(D) := \langle r(1), ..., r(n)\rangle$ is then a total ordering of the candidate set $D$.
We let $r(k)$ denote the candidate $d$ at rank $k$, and $r^{-1}(d)$ denote the position $k$ of candidate $d$ in ranking $r$. 
For a pair of candidates $d, d'\in D$, we say that $d$ is \textit{ranked higher} than $d'$ in $r$, denoted by $d \succ_r d'$, if $r^{-1}(d) < r^{-1}(d')$.
Further, for all candidate sets $D \subseteq \D$ we let $\mathcal{R}(D)$ denote the set of all rankings on $D$.

\subheader{Group Fairness.} In the context of this work, we assume the presence of two groups, a protected group $G_1\not=\emptyset$ and a non-protected group $G_0\not=\emptyset$, in the population $\D$. Each candidate $d\in\mathcal{D}$ belongs to exactly one group, such that $G_0\cup G_1 = \D$ and $G_0 \cap G_1 = \emptyset$. 
For each group $G\in\{G_0,G_1\}$ we let $p_{G} := p_{G}(\D) := \frac{|G|}{|\D|}$ denote the fraction of candidates from this group in the population.
Further, we let $p_\text{groups} := p_\text{groups}(\D) = \left(p_{G_0}(\D), p_{G_1}(\D)\right)\in \R^2$ denote the vector of the group distribution.

\subheader{Ranking Metrics.}
A ranking metric $m$ is a mapping $m: \bigcup_{D\subseteq \D} \mathcal{R}(D) \longrightarrow \mathbb{R}$ 
that assigns each ranking a quality score.
In context of this work, we specifically focus on metrics for group fairness.
For every fairness metric $m$, there exists an \emph{optimal value} $\vopt=\vopt(m)$. A ranking $r$ is considered perfectly fair with respect to $m$ if $m(r) = \vopt$.
Further, without loss of generality, we assume that if a ranking $r$ discriminates against the protected group with respect to $m$, it has to hold that $m(r)<\vopt$, and that generally, lower values indicate higher levels of discrimination against the protected group.
Conversely, not all metrics have their maximum value at $\vopt$, and values $m(r)>\vopt$ would indicate that the non-protected group is disadvantaged (cf. property 1).
We note that for some existing metrics, its optimal value is given by its minimum, however, to obtain a unified setting, we will adapt these metrics here.

\subheader{Generating Rankings.}
In practice rankings are usually generated algorithmically, based on the application at hand. 
Since in this work we focus on evaluating ranking \emph{metrics} via theoretical properties, specific algorithms are however neither applied nor of interest.
Instead, the properties proposed in Section \ref{sec:properties} will consider theoretical sets of rankings where additional mathematical assumptions may be imposed.

\subsection{Two Settings for Ranking}\label{sec:settings}
In existing literature, fair ranking metrics are applied in a diverse range of ranking scenarios that imply different assumptions. In light of this issue, we propose to differentiate between two main settings for rankings, which are described in the following. 

\subheader{Setting 1: Ranking Full Population.} 
In Setting 1 we assume that the candidate set that is to be ranked corresponds to the full population. 
This setting is particularly applicable when the total candidate population is not too large, e.g., when ranking candidates for a job. 
In such applications, the same ranking algorithms may be applied on different populations $\D$, e.g., when considering applications for the same job over multiple years, and thus there is a necessity to assess the behavior of ranking metrics over varying candidate populations.

\subheader{Setting 2: Ranking Subset of Population.} \label{subh: setting_2}
In this setting, we assume that the candidate set $D$ to be ranked for a given query is a strict subset of the full population $\D$, i.e., $D\subsetneq\D$. Further, we assume that $|D|\ll|\D|$, which is a common scenario in information retrieval, where only a limited number of documents can be retrieved and ranked upfront.
In context of group fairness, it is particularly important to consider that the relative proportion of each group $G\in\{G_0,G_1\}$ in the candidate set might strongly differ from its relative proportion $p_G(\D)$ in the full population.

\subsection{Existing Group Fairness Metrics} \label{sec:ex_metrics}
Having introduced the foundations of the fair ranking problem, we will now present an overview of existing metrics as a basis for evaluation.
Most of the subsequent metrics follow the intuition that the highest ranked candidates receive more attention from users than candidates at lower ranks.
The \emph{position bias} captures this intuition by assigning weights to each rank in a way that the highest ranks obtain particularly high weights \cite{Baeza2018}. 
We adopt its common formulation as a logarithmic discount, as it provides a smooth reduction and exhibits favorable theoretical properties \cite{Wang2013}:
\begin{equation}\label{eq:posbias}
    b(k) := \frac{1}{\log_2(k+1)}.
\end{equation}

\subheader{Prefix Metrics.} 
\citet{yang2016measuring} have proposed three metrics that conceptually compare the representation of protected candidates in successive ranking prefixes, i.e., top-$k$ subsets of a ranking, with the representation of the protected group in the whole population. 
The difference in representation is accumulated at discrete cut-off points $k\in I \subseteq \{1,\dots, n\}$, and weighted by the position bias function \eqref{eq:posbias}.
The exact set of cut-off points $I$ is usually determined for a given problem at hand, common choices include linearly increasing cut-offs such as $I = \{10, 20, 30,\dots \}$~\citep{wu2018discrimination, raj2022measuring}, or using every single rank as cut-off~\citep{draws2021, berg2022}.

\emph{Normalized discounted difference} (\emph{rND}) considers the absolute difference between the proportions of the protected group in the top-$k$ ranks and in the overall ranking.
Letting $p_{G}^k(r) := \frac{|\{r(i) : i\leq k\}\cap G|}{k}$ denote the fraction of candidates of a group $G\in\{G_0,G_1\}$ in the top-$k$ positions of the ranking $r$, we define $rND$ as
\begin{equation*}
	rND(r) := \mathlarger{1} - \frac{1}{Z_\text{ND}(D)}\cdot \sum_{k \in I} b(k)\cdot \left|p_{G_1}^k(r) - p_{G_1}\right|,
\end{equation*}
where $Z_\text{ND}(D)$ denotes the maximum value that the sum can take on over the given candidate set $D$.

The second metric, \emph{normalized discounted ratio} (\emph{rRD}), considers the difference between the ratio of protected and non-protected candidates in the top-$k$ ranks and the overall ranking, and is defined as
\begin{equation*}
	rRD(r) := \mathlarger{1} - \frac{1}{Z_\text{RD}(D)}\cdot \sum_{k \in I} b(k)\cdot \left|\tfrac{p_{G_1}^k(r)}{p_{G_0}^k(r)} - \frac{p_{G_1}}{p_{G_0}}\right|,
\end{equation*}
with $Z_\text{RD}$ denoting the maximum value of the sum over the given candidate set $D$.
In practical scenarios, whenever it is the case that $p_{G_0}^k(r) = 0$, the authors set $\frac{p_{G_1}^k(r)}{p_{G_0}^k(r)} := 0$ to avoid division by zero.

Instead of computing absolute differences between ratios of group proportions, \emph{normalized Kullback-Leibler divergence} (\emph{rKL}) applies the Kullback-Leibler divergence $\Delta_\text{KL}$ as a distance metric.
Given two probability vectors $p=(p_1,\dots,p_l),q=(q_1,\dots,q_l)\in\R^l$ with $l\geq 2$, the Kullback-Leibler divergence is defined as
\begin{equation}\label{eq:KL}
     \mathlarger{\Delta_\text{KL}}(p\|q) =  \sum_{i=1}^l p_i\cdot\log\left(\frac{p_i}{q_i}\right).
\end{equation}
Letting $p_\text{groups}^k(r) := \left(p^k_{G_0}(r),p^k_{G_1}(r)\right)$ denote the distribution vector of relative group sizes at rank $k$, \emph{normalized Kullback-Leibler divergence} is then defined as
\begin{equation*}
	rKL(r) := \mathlarger{1} - \frac{1}{Z_\text{KL}(D)} \cdot \sum_{k \in I} b(k)\cdot \mathlarger{\Delta_\text{KL}}\left(p_\text{groups}^k(r)\big\|p_\text{groups}\right),
\end{equation*}
where like in the previous metrics, the normalization factor $Z_\text{KL}(D)$ corresponds to the highest possible value of the right-hand sum on the candidate set $D$. 

We adapted all three metrics from their original proposition \cite{yang2016measuring} by subtracting the normalized sum from 1, so that they fit into our framework where lower values indicate higher levels of discrimination against the protected group.
The optimal and also maximum value of all three prefix metrics $m\in\{RND,RRD,RKL\}$ is then $\vopt(m)=1$, and their minimum value is 0. 
However, these metrics cannot capture whether the protected group is over- or underrepresented (cf. property 1). 

\subheader{Exposure Metrics.}
This family of metrics was introduced by \citet{Singh2018, singh2019policy}, and measures whether \emph{exposure} is fairly distributed between groups.
Exposure denotes the potential attention that a candidate $d$ receives at a specific ranking position $k$, and is quantified by the position bias $b(k)$ as defined in Equation \eqref{eq:posbias}.
Thus, we define the \emph{exposure} of a group $G\in\{G_0,G_1\}$ in a ranking $r$ on a candidate set $D$ as the total position bias of its group members, averaged by the respective group sizes in the full population $\D$:
\begin{equation} \label{eq:group_exposure}
	\operatorname{Exposure}(G | r) := \frac{1}{|G|}\cdot\sum_{d \in G\cap D} b(r^{-1}(d)) . 
\end{equation} 
To measure whether protected and non-protected group receive equal exposure in proportion to their group sizes, we can compute \emph{exposure difference} (\emph{ED}) via
\begin{equation*}
    \emph{ED}(r) := \operatorname{Exposure}(G_1|r) - \operatorname{Exposure}(G_0|r),
\end{equation*}
or alternatively \emph{exposure ratio} (\emph{ER}) as
\begin{equation*}
    \emph{ER}(r) := \frac{\operatorname{Exposure}(G_1|r)}{\operatorname{Exposure}(G_0|r)}.
\end{equation*}
The optimal values are obtained when both groups receive the same exposure, with $\vopt(ED) = 0$ and $\vopt(ER) = 1$. 

The next four metrics also consider ground-truth relevance scores.
\citet{Singh2018, singh2019policy} argue that the treatment or impact of a group $G$ should be proportional to their average relevance $Y(G)$, which is defined as
\begin{equation*}
	Y(G) := \frac{1}{|G|}\cdot \sum_{d \in G} y(d) .
\end{equation*}
To measure the extent to which the protected and non-protected groups are treated disparately, with \emph{treatment} being viewed as allocation of exposure, they define \emph{disparate treatment difference} (\dtd{}) as
\begin{equation*}
	DTD(r) := \frac{\operatorname{Exposure}(G_1|r)}{Y(G_1)} - \frac{\operatorname{Exposure}(G_0|r)}{Y(G_0)} .
\end{equation*}
Analogously, \emph{disparate treatment ratio} (\dtr{}) is defined as
\begin{equation*}
	DTR(r) := \frac{\operatorname{Exposure}(G_1|r)}{\operatorname{Exposure}(G_0|r)}\cdot \frac{Y(G_0)}{Y(G_1)} .
\end{equation*}
The optimal values are given by $\vopt(DTR) = 1$ and $\vopt(DTD) = 0$.
Note that \dtd{} and \dtr{} are equivalent to \ed{} and \er{}, respectively, if the relevance scores of all candidates $d\in \D$ are $y(d)=1$.

Rather than treatment, the next two metrics consider the \emph{impact} of groups in a ranking, which is modeled via the \emph{click-through rate} (\emph{CTR}).
The CTR aggregates the candidate-wise product of relevance and exposure, and is defined as
\begin{equation*}
    CTR(G|r) := \frac{1}{|G|}\cdot \sum_{d \in G\cap D} b(r^{-1}(d)) \cdot y(d).
\end{equation*}
\emph{Disparate impact difference} (\did{}) measures the extent to which the click-through rates of the protected and non-protected group are proportional to their average relevance:
\begin{equation*}
    \emph{DID}(r) := \frac{CTR(G_1|r)}{Y(G_1)} - \frac{CTR(G_0|r)}{Y(G_0)} .
\end{equation*}
Analogously, the \emph{disparate impact ratio} (\dir{}) is defined as
\begin{equation*}
DIR(r) := \frac{CTR(G_1|r)}{CTR(G_0|r)}\cdot \frac{Y(G_0)}{Y(G_1)} .
\end{equation*}
For these metrics, it holds that $\vopt(DID) = 0$ and $\vopt(DIR) = 1$. 
Again, if the relevance scores of all candidates $d\in \D$ are $y(d)=1$, these metrics are equal to \ed{} and \er{}, respectively.

\subheader{Attention-weighted Rank Fairness.}
\citet{sapiezynski2019quantifying} have considered a similar rationale as the exposure metrics, arguing that the relative exposure of a group $G$ should be similar to its relative proportion $p_G$ in the population.
Thus, in their \emph{attention-weighted rank fairness} (\awrf{}) metric, they introduce the probability vector 
$$
\pexp(r) = \tfrac{1}{\sum_{k=1}^n b(k)} \cdot \left(\Sigma_{d \in G_0\cap D}\, b(r^{-1}(d)), \Sigma_{d \in G_1\cap D}\, b(r^{-1}(d)\right) \in\R^2
$$
to model the distribution of exposure per group, which is then compared to the true distribution $p_\text{groups}$ of relative group sizes in the population.
Therefore, they define their metric as 
\[
AWRF_\Delta(r) = \Delta(\pexp(r), p_\text{groups}),
\]
where $\Delta$ denotes a given divergence metric that can be applied on two probability vectors.
While \citet{sapiezynski2019quantifying} consider applying a statistical test to compare similarity of these metrics, \citet{raj2022measuring} propose using some more traditional distance functions such as the Kullback-Leibler divergence $\Delta_\text{KL}$ \eqref{eq:KL}.
In the recent \emph{TREC 2022} fair ranking challenge \cite{trec-fair-ranking-2021}, the Jensen-Shannon divergence $\Delta_\text{JS}$ is used as distance function.
In our analysis, we will also use this distance metric, and thus, define \awrf{} as
$$
AWRF(r) := AWRF_{\Delta_\text{JS}}(r) := \mathlarger{1} - \left[ \tfrac{1}{2}\cdot\mathlarger{\Delta_\text{KL}}\left(\pexp(r)\big\| p_\text{sum}(r)\right) + \tfrac{1}{2}\cdot\mathlarger{\Delta_\text{KL}}\left(p_\text{groups}\big\| p_\text{sum}(r)\right)\right], 
$$
where $p_\text{sum}(r) := \tfrac{1}{2}\cdot\big(\pexp(r)+p_\text{groups}\big)$.
The optimal value of this metric is $\vopt(AWRF) = 1$.
Note that again we adapted the original metric by subtracting it from 1 to make it consistent with our setting.

\subheader{Pairwise Metrics.}
\emph{Pairwise statistical parity} (\psp{})~\cite{narasimhan2020pairwise} measures ranking fairness by considering candidate pairs in which the respective candidates belong to different groups. 
Based on these pairs, it is examined whether, on average, both groups have the same chance of being top ranked.
This metric is only applicable in Setting 1 where the full population is ranked, and is defined as
\[
PSP(r) := \frac{\big|\{(d, d')\in G_0 \times G_1 : d' \succ_r d\}\big| - \big|\{(d, d')\in G_0 \times G_1 : d \succ_r d'\}\big|}{\big| G_0 \times G_1\big|}. 
\]
The optimal value of this metric is given by $\vopt(PSP) = 0$.

\section{Properties of Group Fairness Metrics}\label{sec:properties}

Having established the preliminaries of fair ranking problems, we now turn to formulating a set of properties that we deem useful for evaluating group fairness metrics for rankings.
In that context, we distinguish properties by their applicability to the ranking settings mentioned in Section \ref{sec:settings}.
We begin with properties that are applicable in both settings, before describing properties that apply individually to either Setting 1 or Setting 2.

\subheader{Uniform Relevance.}
Many of the properties that are presented in the following will assume that all candidates in a given population have the same relevance with respect to a given query. We will denote this specific setting as \emph{uniform relevance}. 
Formally, this means that we have $y(d)=1$ for all candidates $d\in \D$. 
We note that this is rather a theoretical assumption that in practice one would not except to hold.
However, when evaluating group fairness metrics for rankings, this assumption allows for assessing behavior of metrics on rankings where higher or lower ranks cannot be justified in terms of relevance, since all candidate have equal relevance.
Any disadvantage would then be based on group membership, and thus properties of group fairness properties are of specific interest in that scenario.
In the following, we will only assume uniform relevance to hold when this is explicitly specified in the definition of a property.

\subheader{Extreme Case Rankings.}
Next to the uniform relevance setting, we will often consider particularly imbalanced rankings, where all elements from one group are ranked higher than any element from the other group.
Such rankings should intuitively yield extreme values in terms of any metric $m$, and are thus of particular interest in context of this work.
Formally, given a population $\D$ and a candidate set $D\subseteq \D$ we define $\mathcal{R}_{\operatorname{first}}(D)$ as the set of all rankings on $D$ where every candidate of the protected group is ranked higher than every candidate of the non-protected group.
Conversely, we define $\mathcal{R}_{\operatorname{last}}(D)$ as the set of all rankings on $D$ where every candidate of the non-protected group is ranked higher than every candidate of the protected group.
When additionally assuming uniform relevance, all metrics presented in Section \ref{sec:ex_metrics} will assign the same value $m(r)$ for all $r\in\mathcal{R}_{\operatorname{first}}(D)$, and the same value $m(r')$ for all $r'\in\mathcal{R}_{\operatorname{last}}(D)$.
Thus, in this setting we can define the extreme ranking values $\vf(D):= \vf(m,D) := m(r), r\in\mathcal{R}_{\operatorname{first}}(D)$ and $\vl(D) := \vl(m,D) := m(r'),r'\in\mathcal{R}_{\operatorname{last}}(D)$.

\begin{figure}
     \centering
     \begin{subfigure}[b]{0.62\textwidth}
         \centering
         \includegraphics[width=\textwidth]{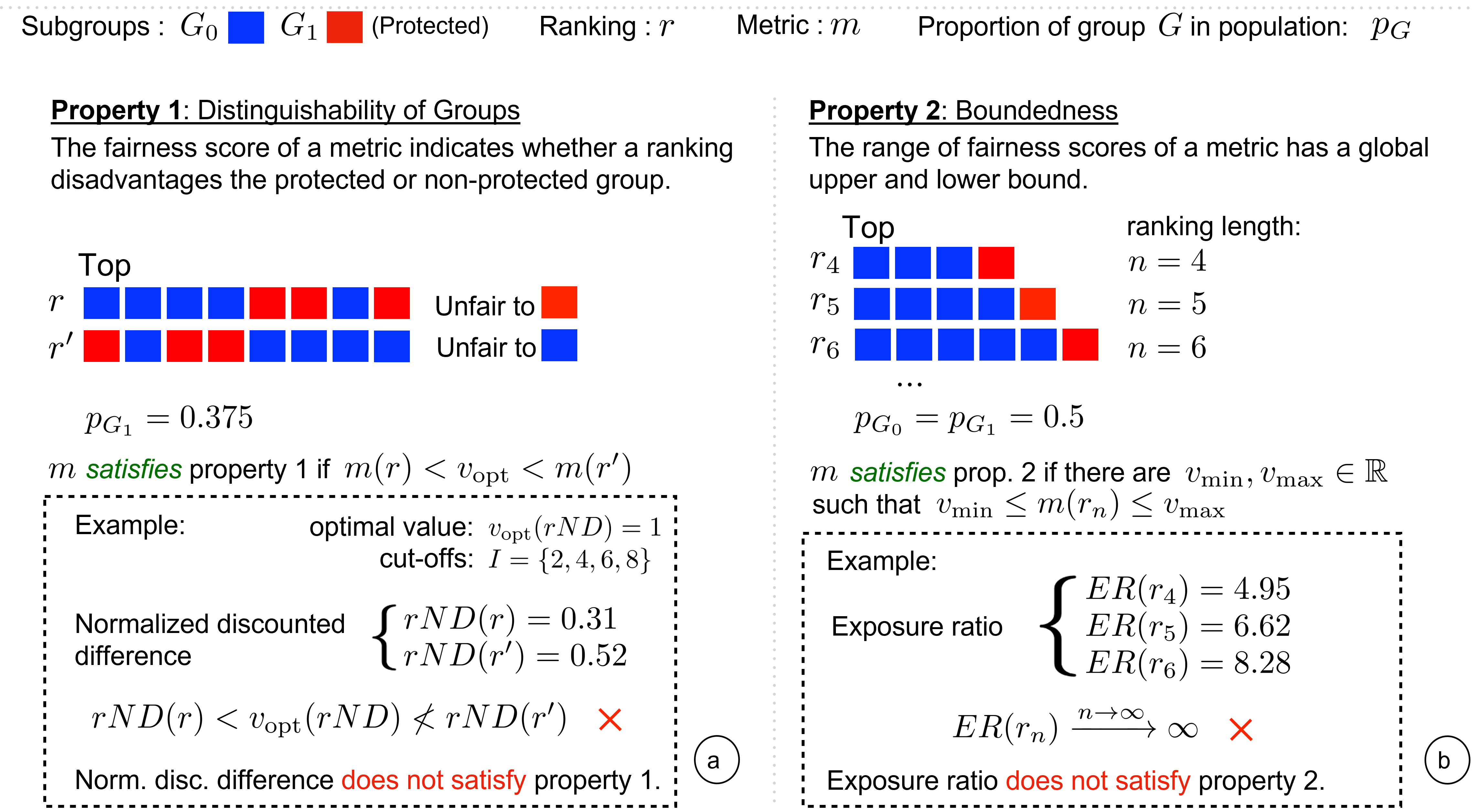}
         \label{fig:1.1}
     \end{subfigure}
     \begin{subfigure}[b]{0.318\textwidth}
         \centering
         \includegraphics[width=\textwidth]{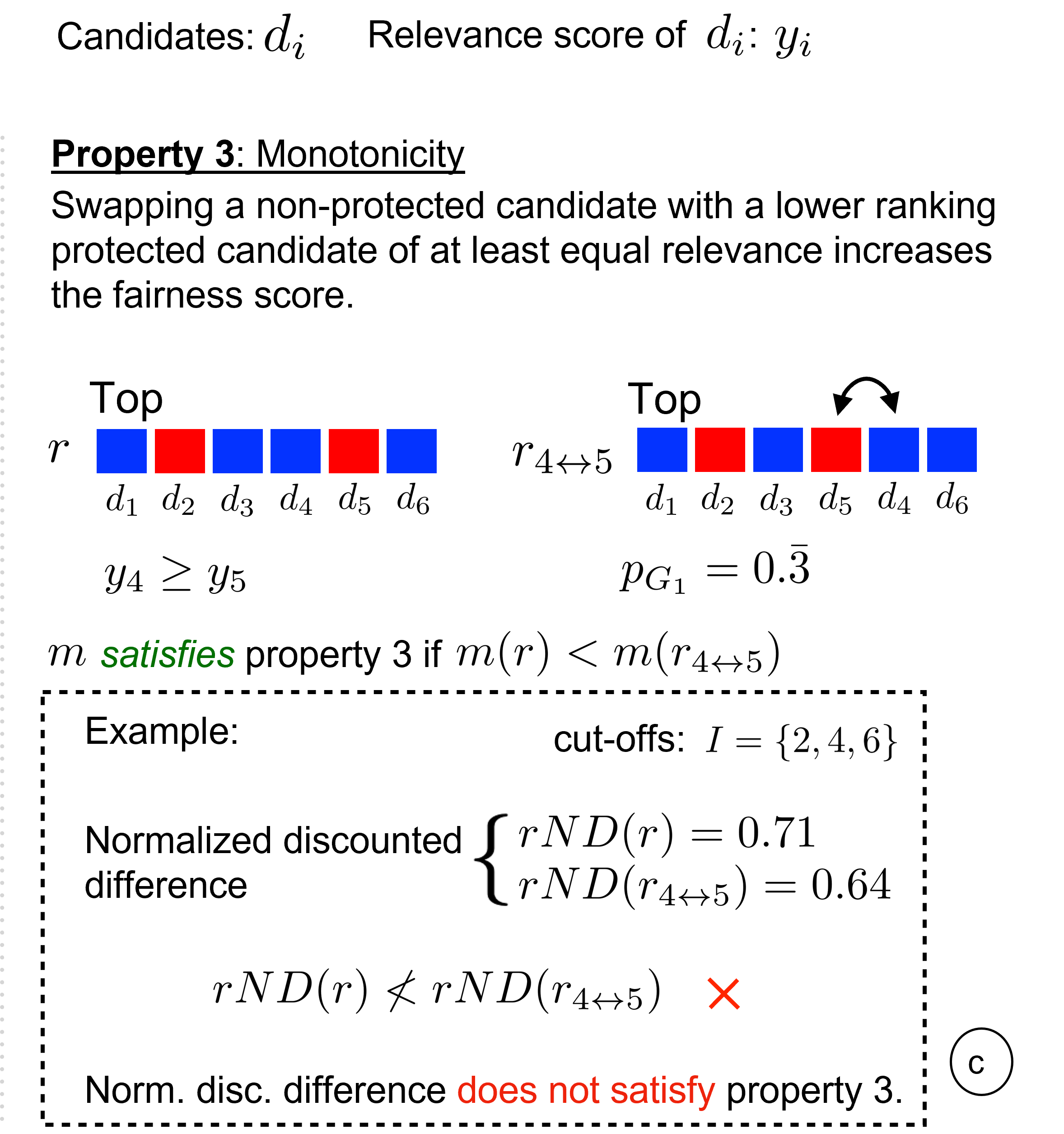}
         \label{fig:1.2}
     \end{subfigure}
     \centering
     \begin{subfigure}[b]{0.295\textwidth}
         \centering
         \includegraphics[width=\textwidth]{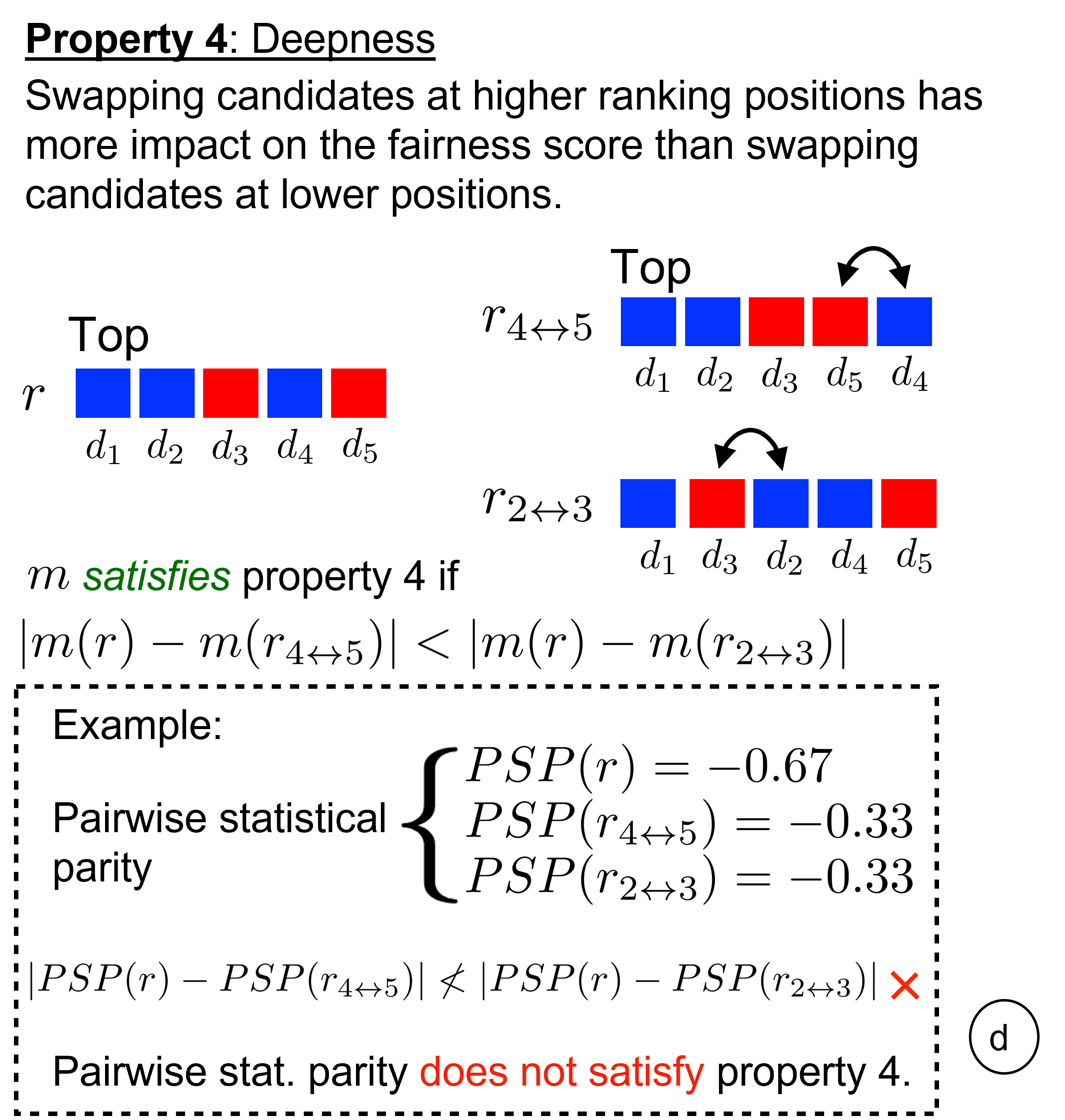}
         \label{fig:1.3}
     \end{subfigure}
     \begin{subfigure}[b]{0.627\textwidth}
         \centering
         \includegraphics[width=\textwidth]{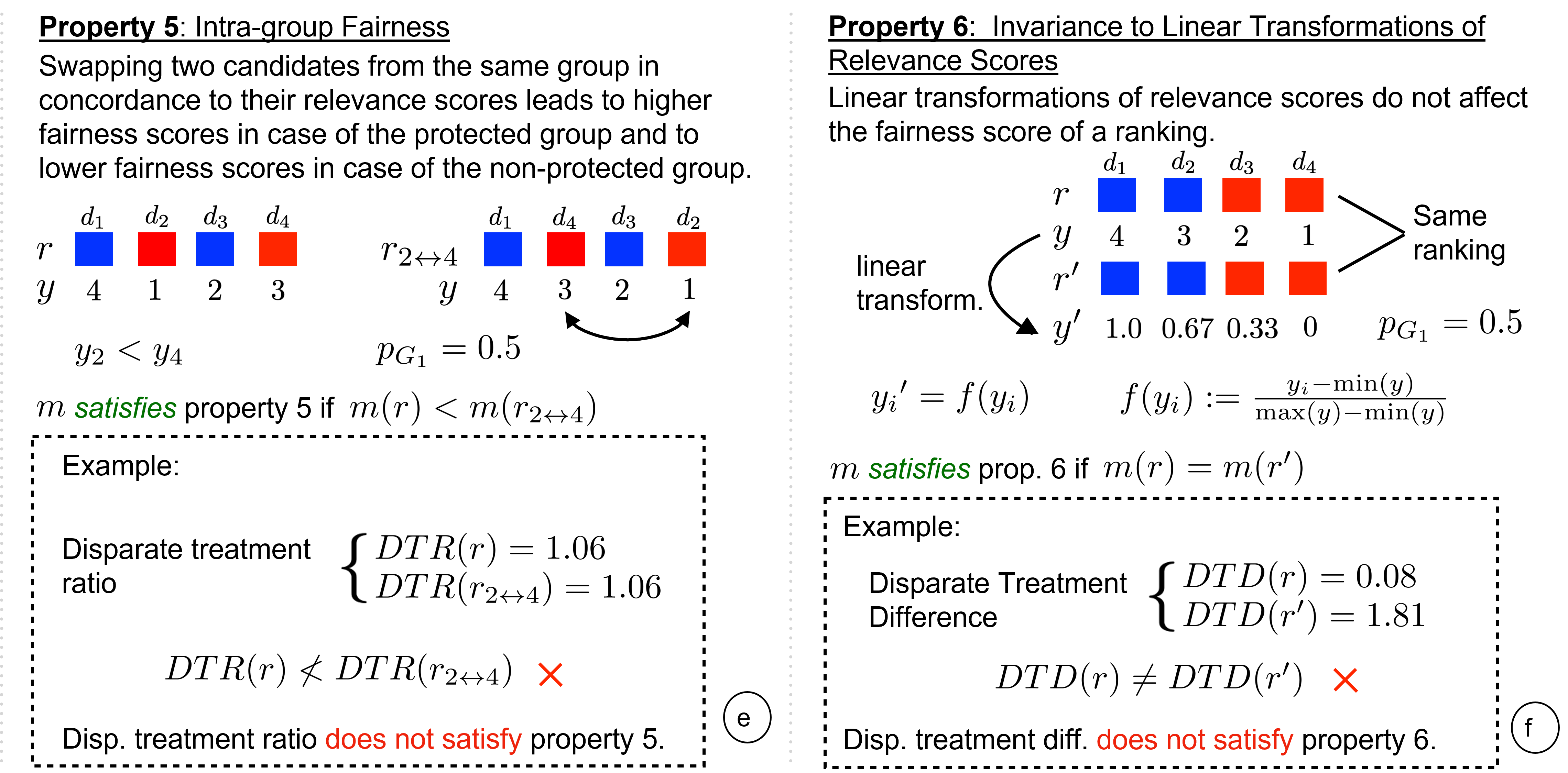}
         \label{fig:1.4}
     \end{subfigure}\vspace{-5mm}
     \caption{\emph{Universal properties for group fairness metrics.} We illustrate six universal properties using exemplary ranking scenarios in panels (a) - (f). In each scenario, a ranking candidate belongs to either a \textcolor{red}{protected} or a \textcolor{blue}{non-protected} group.
     In (a), property 1 requires that a fairness metric $m$ is able to reflect whether the protected group or the non-protected group is disadvantaged. The example shows that normalized discounted difference (\rnd) does not satisfy this criterion, as it assigns values lower than the optimal value $\vopt$ when either group is disadvantaged.
     In (b), property 2 requires that $m$ is bounded. Our example shows a set of rankings where with increasing ranking length, exposure ratio (\er{}) goes to infinity. 
     In (c), property 3 is satisfied if swapping a non-protected candidate with a lower ranking protected ranking of at least equal relevance increases the score of $m$. In the given example, \rnd{} however decreases its value for such a swap.
     In (d), property 4 requires that the value of $m$ should be impacted more when swapping candidates at higher ranks than when swapping at lower ranks. In the given example, pairwise statistical parity (\psp{}) however weighs these swaps the same.
     In (e), a metric $m$ satisfies property 5 if swapping candidates in concordance to their relevance scores increases the score if both are in the protected, and decreases the score if both are in the non-protected group.
     In the given example, disparate treatment ratio \dtr{} is however not affected at all by such a swap.
     In (f), property 6 requires that $m$ is invariant to linear transformations of relevance scores. Here, the rankings $r$ and $r'$ are the same, except that the relevance scores $y'$ in $r'$ are a min-max scaled version of the scores $y$ in $r$. By assigning different scores to $r$ and $r'$, we illustrate that disparate treatment difference (\dtd) does not satisfy property 6.}
     \label{fig:universal} 
\end{figure}

\subsection{Universal Properties for Group Fairness Metrics}
\label{sec:prop_set_1_2}

In the following, we present six properties of group fairness metrics for rankings that can be applied in any of the two specified settings. An illustrative overview of these properties is provided in Figure \ref{fig:universal}.

We often consider swapping single pairs of candidates in rankings.
To formalize such a swap within a given ranking $r\in\mathcal{R}(D)$, $D\subseteq \D$, for any two indices $i,j$ with $1\leq i < j \leq n$ we let $r_{i \leftrightarrow j}$, defined as
\[
r_{i \leftrightarrow j}(k) := \begin{cases}
r(i), &k = j, \\
r(j), &k = i,\\
r(k), &\text{else,}
\end{cases}
\]
denote the ranking resulting from $r$ when swapping the candidates at positions $i$ and $j$.

\subheader{\underline{Property 1: Distinguishability of Groups.}}
The first property is motivated by the notion that a metric should be able to delineate whether the protected or the non-protected group is at a disadvantage.
\begin{definition}
A ranking metric $m$ satisfies \emph{distinguishability of groups}, if for every  population $\D$ with uniform relevance, it holds that $\vl(m,\D)<\vopt(m)<\vf(m,\D).$
\end{definition}

\noindent 
If \emph{distinguishability of groups} is satisfied, we will consider ranking $r$ with $m(r)<\vopt$ to be unfair towards the protected group, and a ranking $r'$ with $m(r')>\vopt$ unfair toward the non-protected group.

\subheader{\underline{Property 2: Boundedness.}}
The second property concerns the range of values that a fairness metric can take on, and requires this range to be universally bounded. This is particularly necessary to quantify the degree of unfairness of a given ranking and to objectively compare the fairness of different rankings.
\begin{definition}
A ranking metric $m$ satisfies \emph{boundedness}, if there are values $v_{\min}(m), v_{\max}(m)\in\mathbb{R}$ such that for any candidate population $\D$, all candidate sets $D\subseteq\mathcal{D}$, and all rankings $r\in\mathcal{R}(D)$, the fairness score $m(r)$ is well-defined and it holds that $v_{\min}(m)\leq m(r) \leq v_{\max}(m)$.
\end{definition}

\subheader{\underline{Property 3: Monotonicity.}} 
This property is partly motivated by similar properties defined in the context of evaluating classification \cite{sebastiani2015} and document filtering results \cite{amigo2019comparison}.
In our context, we consider a ranking in which a candidate from the non-protected group is ranked higher than a protected candidate with a relevance score at least as high as the non-protected candidate. 
Monotonicity represents the intuitive notion that swapping the ranks of these two candidates should yield an outcome that is more in favor of the protected group.
\begin{definition}
A ranking metric $m$ satisfies \emph{monotonicity}, if for any ranking $r$ with candidates $r(i) \in G_0, r(j) \in G_1$, such that $i<j$ and $y(r(i)) \leq y(r(j))$, it holds that $m(r_{i\leftrightarrow j}) > m(r)$.
\end{definition}

\subheader{\underline{Property 4: Deepness.}}
This property is inspired by a similar property proposed by \citet{amigo2013general}, albeit in the context of document ranking.
Building on the observation that users tend to focus on higher ranks \cite{Baeza2018}, \emph{deepness} formalizes the notion that swapping candidates at higher ranking positions should have more impact on fairness scores than swapping candidates at lower ranking positions.
\begin{definition}
A ranking metric $m$ satisfies \emph{deepness}, if for any ranking $r$ with ranks $i,j$ such that such that $i<j$, $y(r(i)) = y(r(j))$, $y(r(i+1)) = y(r(j+1))$, and either $r(i),r(j)\in G_0$, $r(i+1), r(j+1)\in G_1$,  or $r(i),r(j)\in G_1$, $r(i+1), r(j+1)\in G_0$, it holds that
$
|m(r) - m(r_{i\leftrightarrow i+1}) | > |m(r) - m(r_{j\leftrightarrow j+1}) |.
$
\end{definition}

\subheader{\underline{Property 5: Intra-Group Fairness.}}
This property requires a metric to capture whether higher relevance within a group leads to higher ranks.
\begin{definition}
A ranking metric $m$ satisfies \emph{intra-group fairness}, if for any ranking $r$ with a pair of candidates $r(i), r(j)$ such that $i<j$ but $y(r(i)) < y(r(j))$, it holds that $m(r_{i\leftrightarrow j}) > m(r)$ if both $r(i), r(j) \in G_1$, and $m(r_{i\leftrightarrow j}) < m(r)$ if both $r(i), r(j) \in G_0$. 
\end{definition}

\subheader{\underline{Property 6: Invariance to Linear Transformation of Relevance Scores.}}
The final universal property states that rescaling or translating relevance scores of a candidate set should not affect the fairness score of a given metric $m$. 
This is particularly relevant in settings where relevance scores for rankings are preprocessed by common techniques such as min/max-scaling or z-score normalization, which are both linear transformations. 
Satisfying property 6 ensures that such preprocessing does not change the resulting values of a ranking metric.
\begin{definition}
For any $a\in \mathbb{R}_{>0}$, $c\in\R$, let $f_{a,c}: \mathbb{R} \rightarrow\mathbb{R}, y  \mapsto ay + c$ be a linear function. 
For a given population $\D$, we let $f_{a,c}(\D) := \big\{\big(\mathbf{x},f_{a,c}(y(\mathbf{x}))\big): d=(\mathbf{x},y(\mathbf{x}))\in\D\big\}$ denote the same population in which the relevance scores have been transformed according to $f_{a,c}$, and for any candidate set $D\subseteq\D$ we let $f_{a,c}(D)\subseteq f_{a,c}(\D)$ denote the transformed candidate set. 
Then, a ranking metric satisfies \emph{invariance to linear transformations of relevance scores}, if for any population $\D$ and any candidate set $D\subseteq\D$, for every ranking $r\in\mathcal{R}(D)$, and all $a\in \mathbb{R}_{>0}$, $c\in\R$, it holds that $m\big(r(D)\big) = m\big(r(f_{a,c}(D))\big)$.
\end{definition}
 
\noindent Note that if this property only holds for all $f_{a,c}$ where $a=1$, we call $m$ \emph{invariant to translations of relevance scores}. Conversely, if this property only holds for all $f_{a,c}$ where $c=0$, we call $m$ \emph{invariant to rescaling of relevance scores}.

\subsection{Properties for Setting 1: Ranking Full Population}
\label{sec:prop_set_1}

Next, we consider a set of properties that is specifically targeted at Setting 1, in which we assume that all candidates from a given population $\mathcal{D}$ are ranked. Moreover, we assume uniform relevance for all the properties in this setting, and specifically consider the values $\vf(m,\D)$ and $\vl(m,\D)$ that a metric $m$ takes on varying populations $\D$.

If these extreme values, and hence the range of values, are inconsistent across candidate populations, interpreting individual values of a metric can become extremely difficult. 
With this intuition in mind, we introduce four properties that are representative of consistency across extreme cases as well as the average case. The properties are shown in Figure \ref{fig:setting1}. 

\begin{figure}
     \centering
     \begin{subfigure}[t]{0.495\textwidth}
         \centering
         \includegraphics[width=\textwidth]{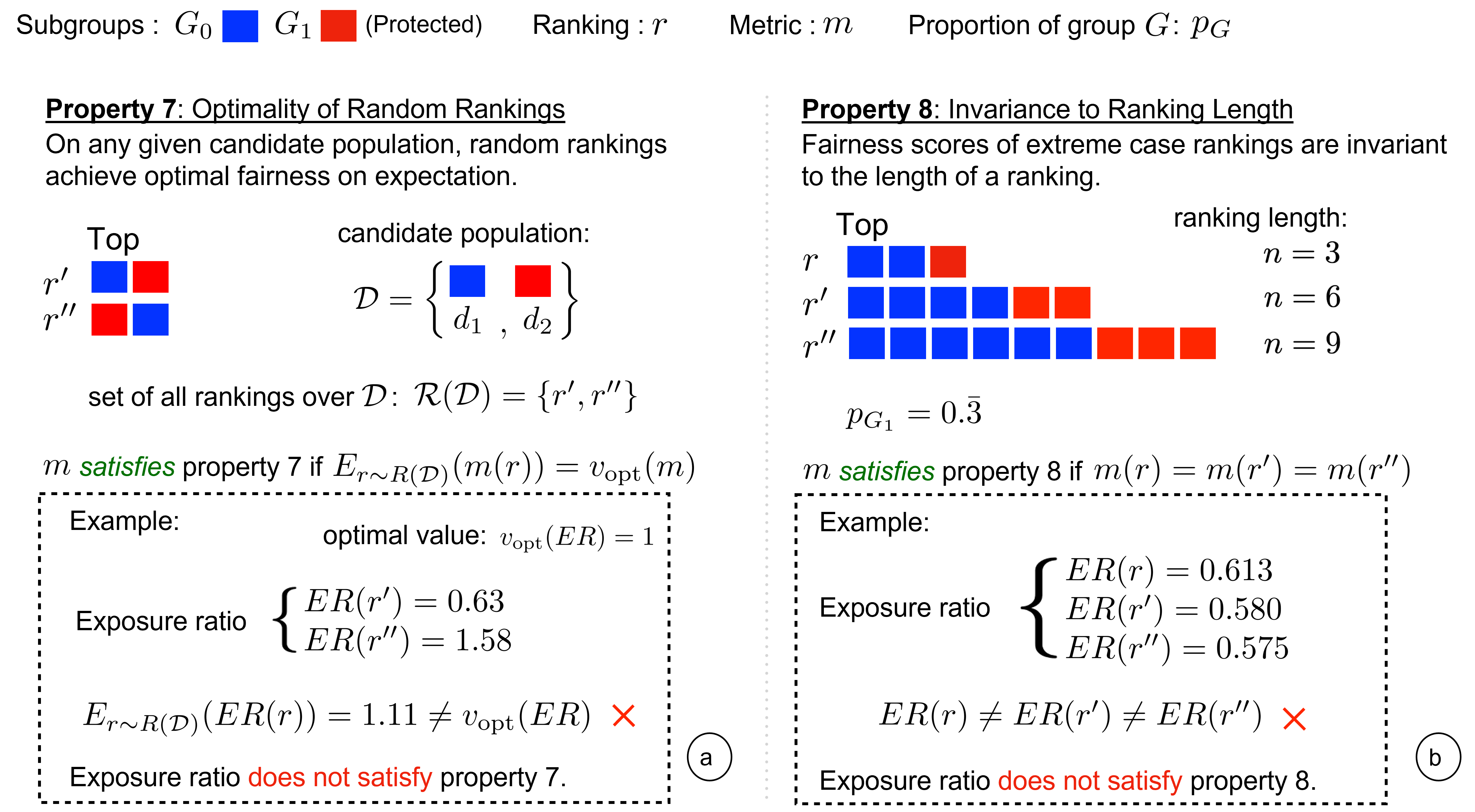}
         \label{fig:2.1}
     \end{subfigure}
     \begin{subfigure}[t]{0.495\textwidth}
         \centering
         \includegraphics[width=\textwidth]{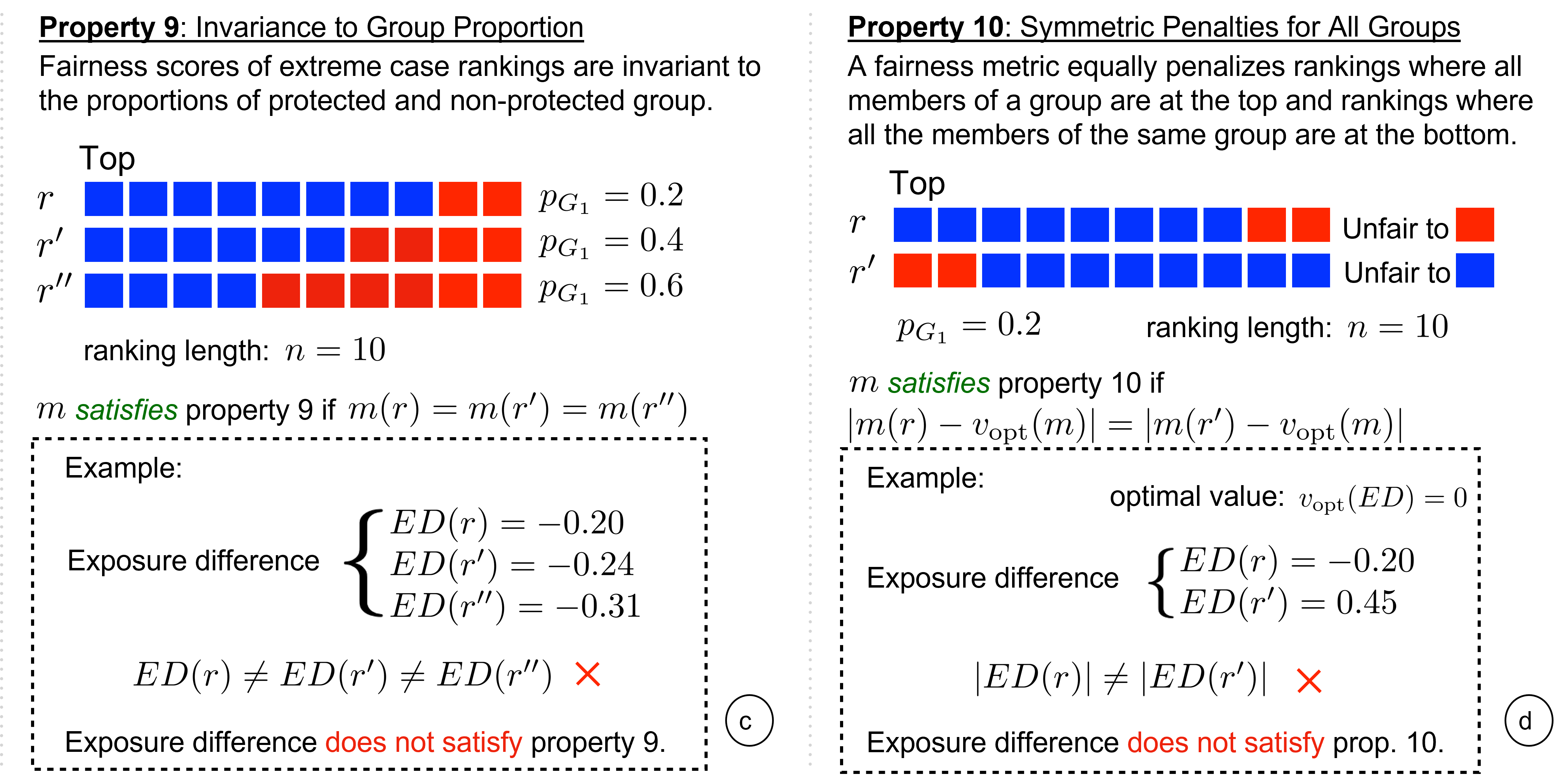}
         \label{fig:2.2}
     \end{subfigure}\vspace{-5mm}
     \caption{ \emph{Properties for ranking the full population.} We illustrate the four properties for Setting 1 in which the full population is ranked, using exemplary ranking scenarios in panels (a) - (d). In each scenario, a ranking candidate belongs to either a \textcolor{red}{protected} or a \textcolor{blue}{non-protected} group.
     In (a), property 7 requires that when sampling uniformly over all rankings of a candidate population $\D$, the ranking metric $m$ should, on expectation, yield the optimal fairness score $\vopt(m)$. We provide a minimal example in which exposure ratio (\er{}) does not obtain its optimal value $\vopt(ER) = 1$ on expectation.
     In (b), property 8 requires that a fairness metric $m$ is invariant to ranking length. The example shows that \er{} fails to satisfy this criterion, as it assigns different values to rankings that only differ in length. 
     In (c), property 9 requires that $m$ is invariant to group proportions. While the rankings in our example are intuitively similar (with the only difference being the size of the protected group), exposure difference (\ed) assigns different fairness scores to them. 
     In (d), property 10 stipulates $m$ to assign symmetric penalties to all groups. In the given example, $r$ disadvantages $G_1$ in the strongest possible way, while $r'$ disadvantages $G_0$ in the strongest possible way. However, \ed{} assigns asymmetric scores to the two rankings, thereby failing to satisfy property 10.}
     \label{fig:setting1}
\end{figure}

\subheader{\underline{Property 7: Optimality of Random Rankings.}}
Naturally, one would expect random rankings to be unbiased.
Following this intuition, this property states that under the premise of uniform relevance, random rankings should achieve optimal fairness in expectation.
\begin{definition}
Given a candidate population $\D$, we let $R:=R(\D)$ denote the random variable which draws from all rankings $r\in\mathcal{R}(\D)$ with equal probability.
Further, we let $v_E(m,\D):=E_{r\sim R(\D)}(m(r))$ denote the expected value of the metric $m$ over all rankings on $\D$.
Then, a ranking metric $m$ satisfies \emph{optimality of random rankings}, if on any population $\D$ with uniform relevance it holds that $v_E(m,\D) = \vopt(m)$.
\end{definition}


\subheader{\underline{Property 8: Invariance to Ranking Length.}} 
This property represents the intuitive notion that a metric should not produce different fairness scores for two rankings that differ only in the number of candidates. Violation of this property can lead to imprecise interpretation of the fairness scores. 
An example of such a scenario would be when the worst case value of a metric on a candidate population $\D$ is 1 while the worst case value of the same metric on a candidate population $\D'$ is 0.1, just because the candidate population $\D'$ contains more candidates.
\begin{definition}
A ranking metric $m$ satisfies \emph{invariance to ranking length}, if for any two candidate populations $\D, \D'$ with uniform relevance and $P_{G_1}(\D) = P_{G_1}(\D')$, but $|\D| \not= |\D'|$, it holds that $\vf(\D) = \vf(\D')$ and $\vl(\D) = \vl(\D')$.
\end{definition}

\subheader{\underline{Property 9: Invariance to Group Proportions.}}
Following a similar motivation as property 8 (\emph{invariance to ranking length}), this property states that the extreme values of a metric should not depend on the relative proportions of protected and non-protected candidates. 

\begin{definition}
A ranking metric $m$ satisfies \emph{invariance to group proportions}, if for any two candidate populations $\D, \D'$ with uniform relevance and $|\D| = |\D'|$, but $P_{G_1}(\D) \not= P_{G_1}(\D')$, it holds that $\vf(\D) = \vf(\D')$ and $\vl(\D) = \vl(\D')$.
\end{definition}

\subheader{\underline{Property 10: Symmetric Penalties for All Groups.}}
This property states that the extreme ranking values $\vf(m,\D)$ and $\vl(m,\D)$ of a metric $m$ should be equidistant from the optimum value $\vopt(m)$.

\begin{definition}
A metric $m$ satisfies symmetric penalties for all groups, if on any population $\D$ with uniform relevance it holds that either $|\vf(\D) - \vopt| = |\vopt - \vl(\D)|$, or $\vf(\D)/ \vopt = \vopt / \vl(\D)$.
\end{definition}
\noindent Note that the latter equations in this definition intend to accommodate  metrics that are based on ratios, such as \er{}.

\subsection{Properties for Setting 2: Ranking Subset of Population}
\label{sec:prop_set_2}
In Setting 2, the representation of the protected group in the ranked candidate set $D$ may strongly differ from its representation in the full candidate population $\D$.
The following three properties, which are illustrated in Figure \ref{fig:setting2}, specifically consider such disparities, and have been adapted from properties proposed by \citet{amigo2013general} in the context of more general document ranking tasks.


The first two properties intuitively analyze a quality versus quantity trade-off in the retrieval of protected candidates.
For that purpose, we assume uniform relevance and for a given $N\in\N$, we let $D_N, D_N' \subseteq \D$ denote two candidate sets with $|D_N| = |D_N'| = 2N$, $|D_N\cap G_1| = 1$, and $|D_N'\cap G_1| = N$.
Using these candidate sets, we compare rankings $r\in\mathcal{R}_{\operatorname{first}}(D_N)$ in which only a single protected candidate is retrieved, but ranked on top, with rankings $r'\in\mathcal{R}_{\operatorname{last}}(D_N')$ in which half of the retrieved candidates are from the protected group, but ranked at the bottom of the ranking.
One might argue that for smaller candidate sets, the first ranking should be considered more favorable, whereas for longer rankings, having retrieved more candidates from the protected group should yield higher fairness scores.

\begin{figure}[t]
     \centering
     \begin{subfigure}[t]{0.653\textwidth}
         \centering
         \includegraphics[width=\textwidth]{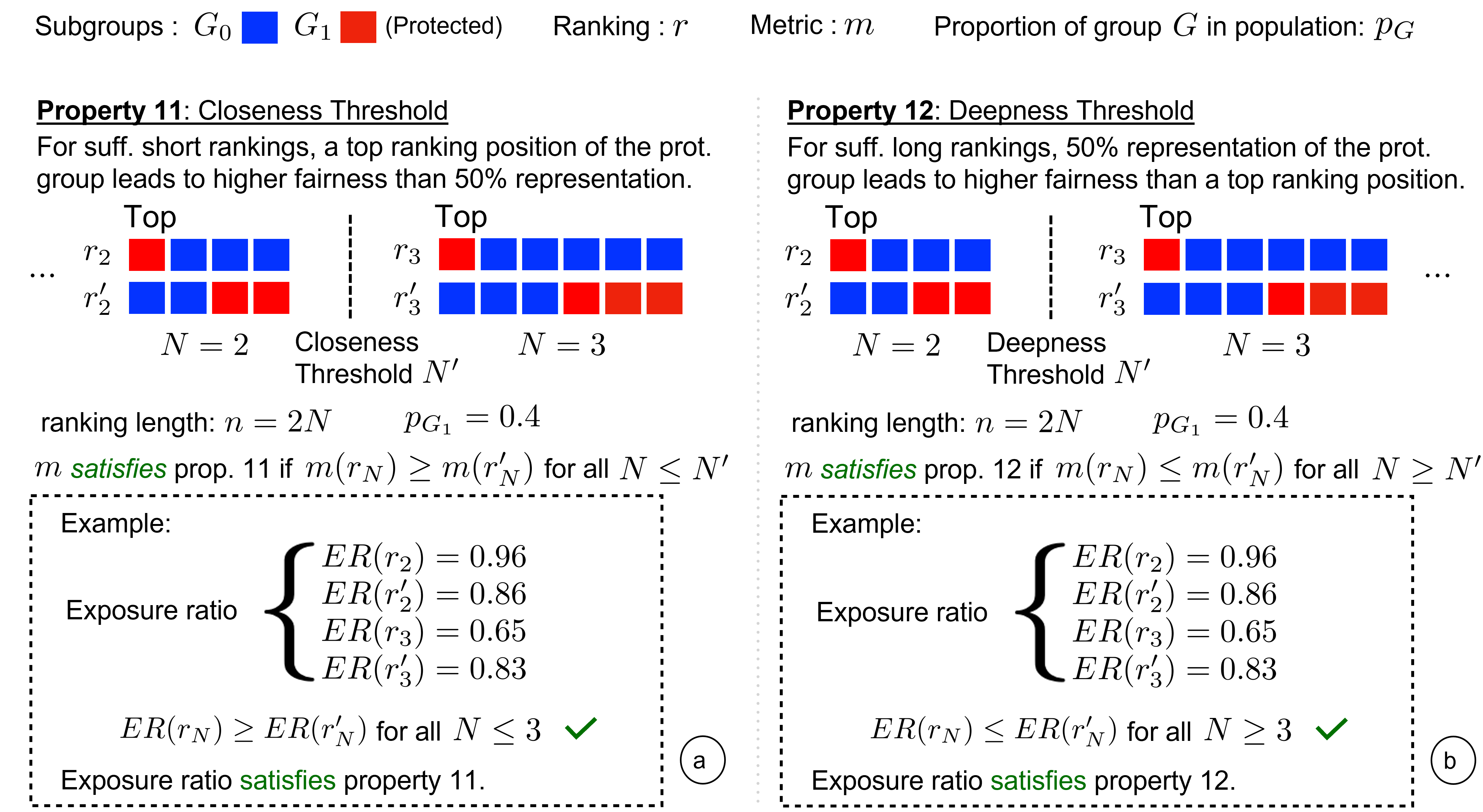}
         \label{fig:3.1}
     \end{subfigure}
     \begin{subfigure}[t]{0.335\textwidth}
         \centering
         \includegraphics[width=\textwidth]{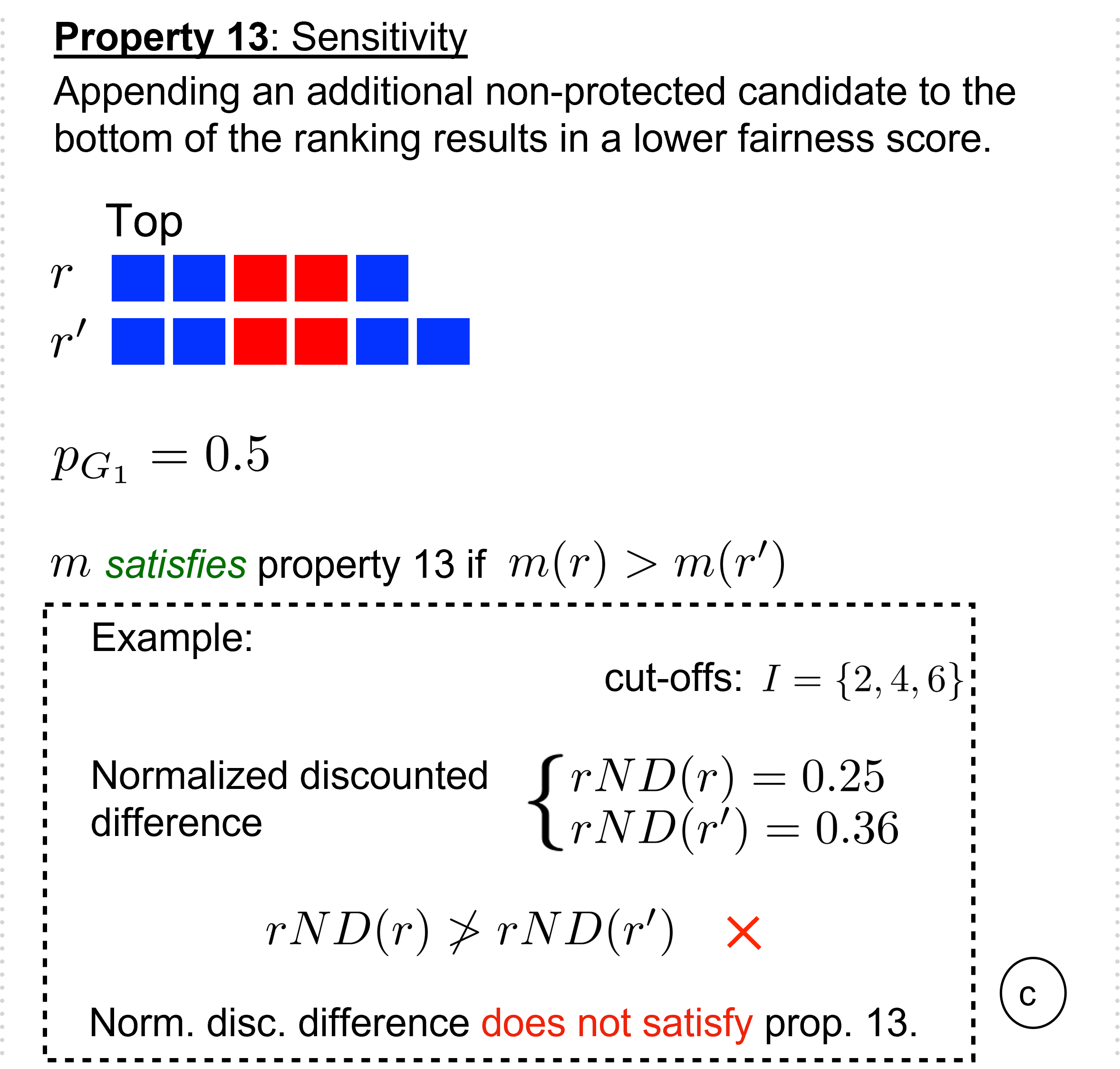}
         \label{fig:3.2}
     \end{subfigure}\vspace{-5mm}
     \caption{\emph{Properties for ranking subsets of a population.} Provided that only a subset of the candidate population is ranked, we illustrate three properties using exemplary ranking scenarios in panels (a)-(c). In each scenario, a ranking candidate belongs to either a \textcolor{red}{protected} or a \textcolor{blue}{non-protected} group.
     In (a) and (b), the given properties 11 and 12 consider and compare two kinds of rankings with identical, but varying length $n=2N$. The first type of ranking contains only a single protected candidate that is however ranked at position 1, whereas in the second type of ranking, 50\% of its candidates are from the protected group, but contrarily, all of these placed at the latter half of the ranking. 
     Now property 11 requires that for sufficiently small $N$, a metric $m$ always assigns a higher score for the first kind of rankings.
     Contrarily, property 12 requires that for sufficiently large $N$, a metric $m$ always assigns a higher score for the second kind of rankings.
     Our examples show that both of these properties are satisfied by exposure ratio (\er). 
     In (c), property 13 requires that appending a candidate of the non-protected group at the end of any ranking always results in lower fairness with respect to $m$. 
     The given example shows that this does not hold for \rnd{}. }
     \label{fig:setting2} 
\end{figure}

\subheader{\underline{Property 11: Closeness Threshold.}}
This property states that for sufficiently short rankings, placing a single candidate of the protected group at the top of a ranking should lead to a higher fairness score than having retrieved many protected candidates that are placed at the end of a ranking. 
\begin{definition}
A ranking metric $m$ has a \emph{closeness threshold}, if for any population $\D$ with uniform relevance and candidate sets $D_N,D_N'\subsetneq\D$ as defined above, there exists a $N'\in\N$ such that for all $N\leq N'$, $r\in\mathcal{R}_{\operatorname{first}}(D_N)$, and $r'\in\mathcal{R}_{\operatorname{last}}\big(D_N'\big)$ it holds that 
$ m(r(D_N)) > m\big(r'\big(D'_N\big)\big)$.
\end{definition}


\subheader{\underline{Property 12: Deepness Threshold.}}
Complementing the \emph{closeness threshold}, this property requires that for sufficiently long rankings, retrieving many candidates of the protected group should yield a higher fairness score than only retrieving a single protected candidate but placing it in the top ranking position.
\begin{definition}
A ranking metric $m$ has a \emph{deepness threshold}, if for any population $\D$ with uniform relevance and candidate sets $D_N,D_N'\subsetneq\D$ as defined above, there exists a $N'\in\N$ such that for all $N\geq N'$, $r\in\mathcal{R}_{\operatorname{first}}(D_N)$, and $r'\in\mathcal{R}_{\operatorname{last}}\big(D_N'\big)$, it holds that 
$ m(r(D_N)) < m\big(r'\big(D'_N\big)\big)$.
\end{definition}

\subheader{\underline{Property 13: Sensitivity.}}
Our final property represents the intuition that appending an additional non-protected candidate to the bottom of any ranking should reduce the fairness score with respect to a given metric.
\begin{definition}
A ranking metric $m$ satisfies \emph{sensitivity}, if for any population $\D$ with uniform relevance, any candidate set $D\subsetneq \D$, any ranking $r\in\mathcal{R}(D)$, and any candidate $d'\in G_0$, $d'\notin D$, the ranking $r':=\langle r(1),\dots,r(n),d'\rangle$ has a lower fairness score than $r$, i.e., it holds that $m(r') < m(r)$.
\end{definition}

\begin{figure}[t]
	\centering
	\begin{subfigure}[t]{0.35\textwidth}
		\centering
		\includegraphics[width=\textwidth]{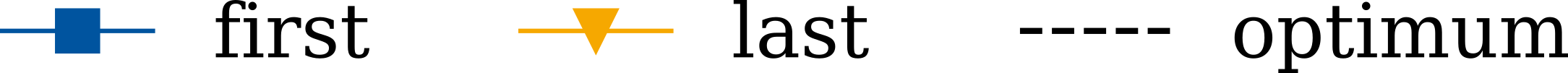}
	\end{subfigure} 
	
	\begin{subfigure}[t]{0.3\textwidth}
		\centering
		\includegraphics[width=\textwidth]{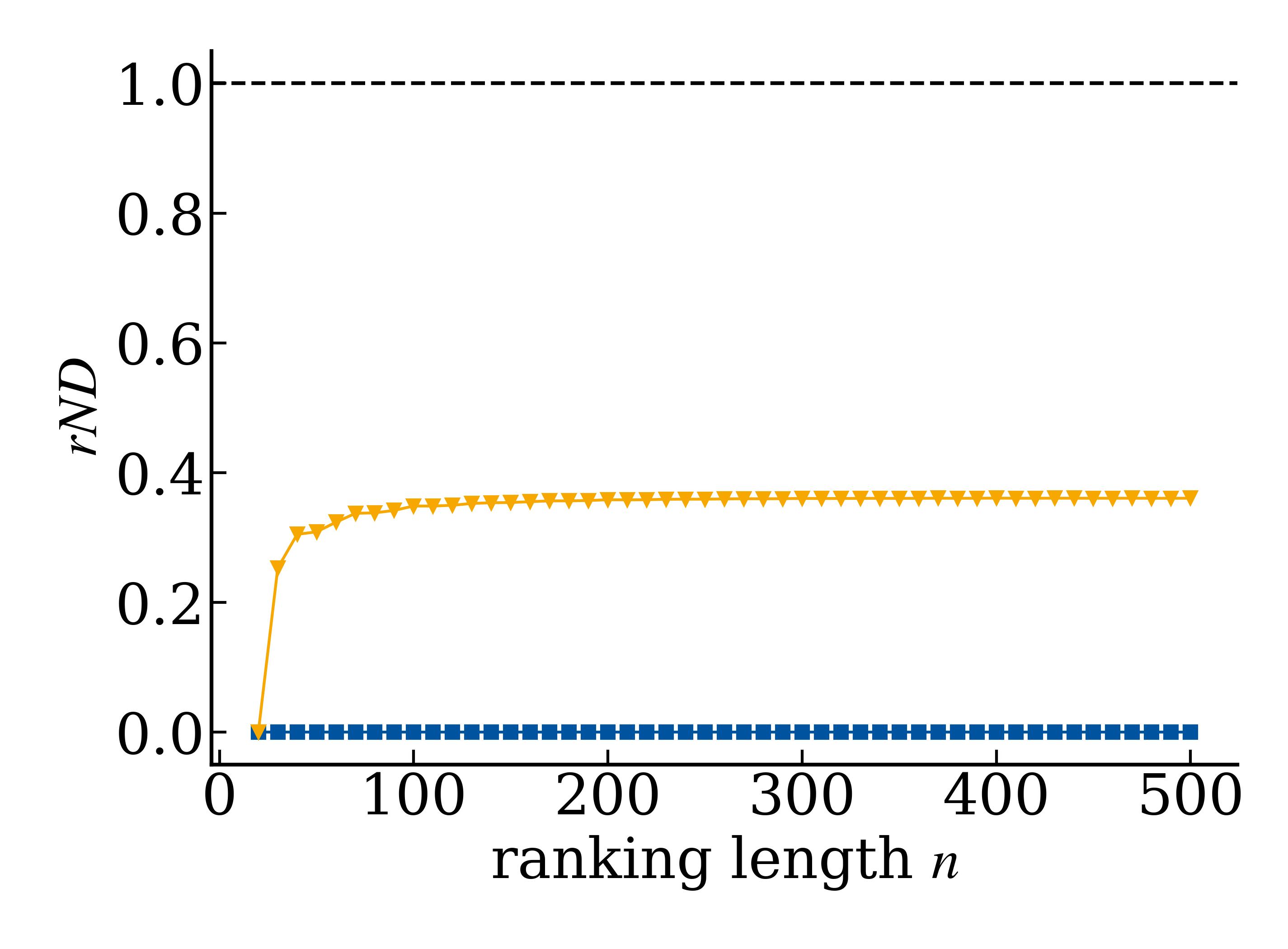}
		\subcaption{normalized disc. difference \rnd{} \hspace{0.6mm} \cross}
	\end{subfigure}
	\begin{subfigure}[t]{0.3\textwidth}
		\centering
		\includegraphics[width=\textwidth]{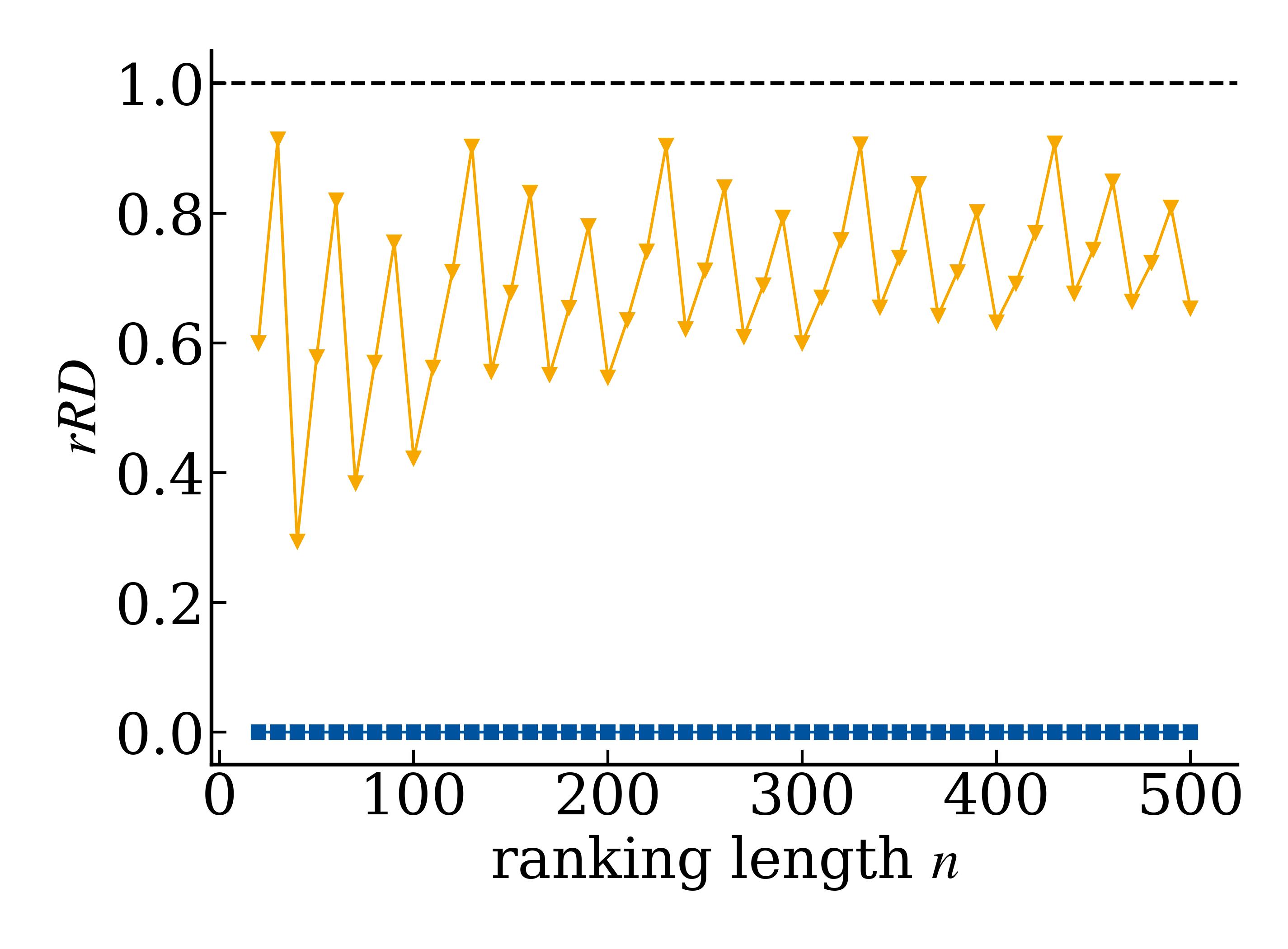}
		\subcaption{normalized disc. ratio \rrd{} \hspace{0.6mm} \cross}
	\end{subfigure}
	\begin{subfigure}[t]{0.3\textwidth}
		\centering
		\includegraphics[width=\textwidth]{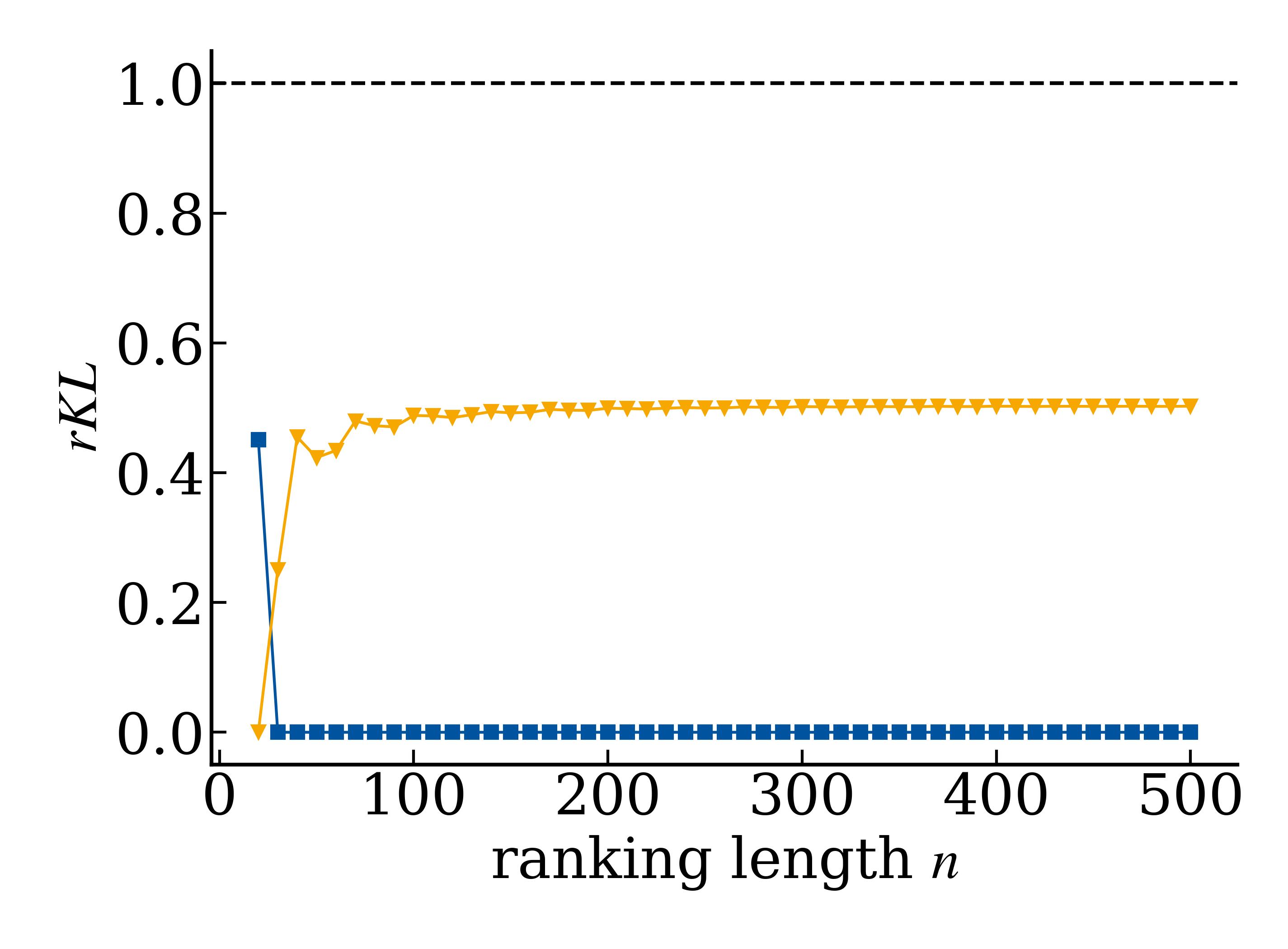}
		\subcaption{norm. disc. KL-divergence \rkl{} \hspace{0.6mm} \cross}
	\end{subfigure}
	\begin{subfigure}[t]{0.3\textwidth}
		\centering
		\includegraphics[width=\textwidth]{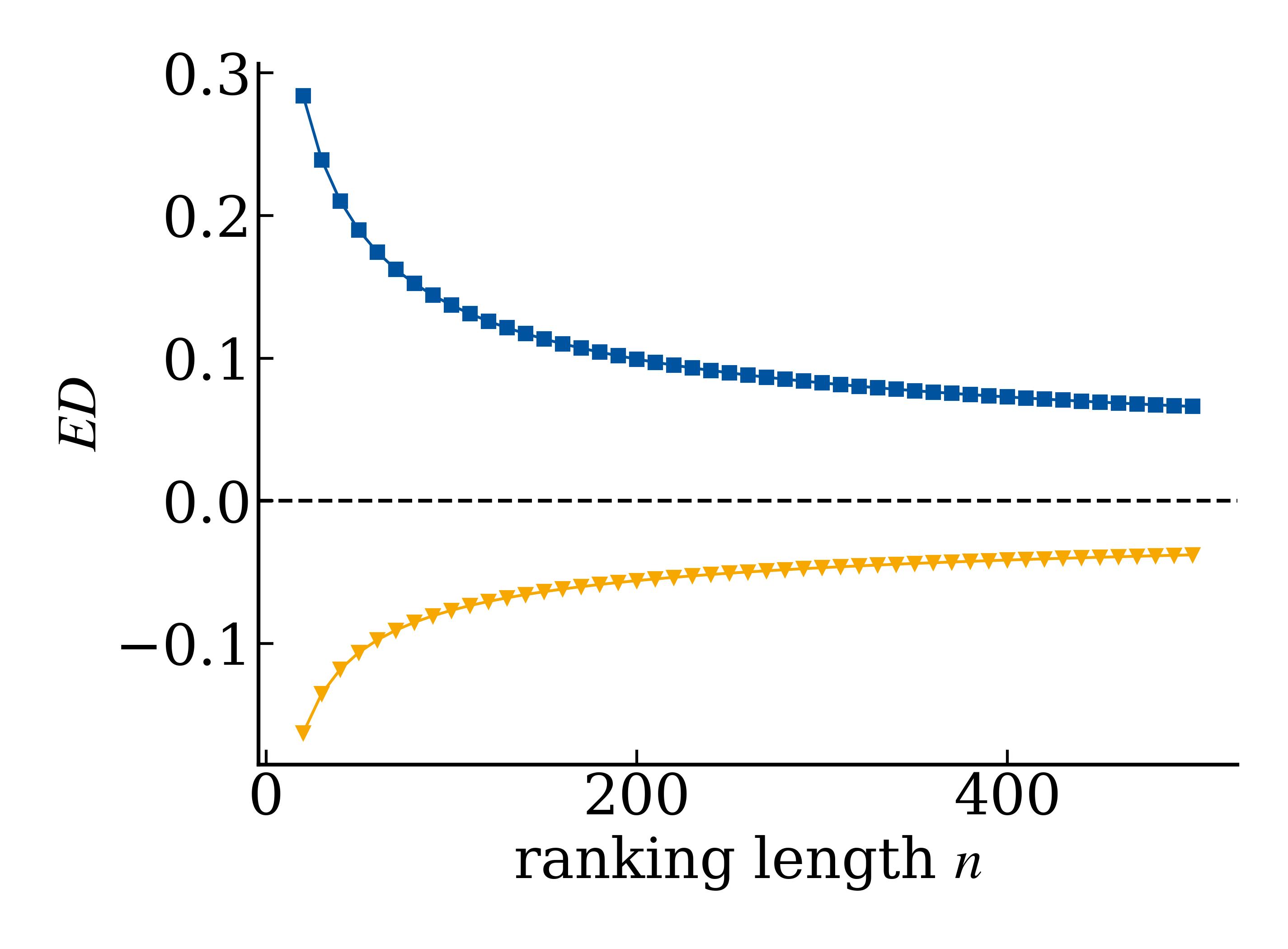}
		\subcaption{\ed{}, \dtd{} and \did{} \hspace{0.6mm} \cross}
	\end{subfigure}
	\begin{subfigure}[t]{0.3\textwidth}
		\centering
		\includegraphics[width=\textwidth]{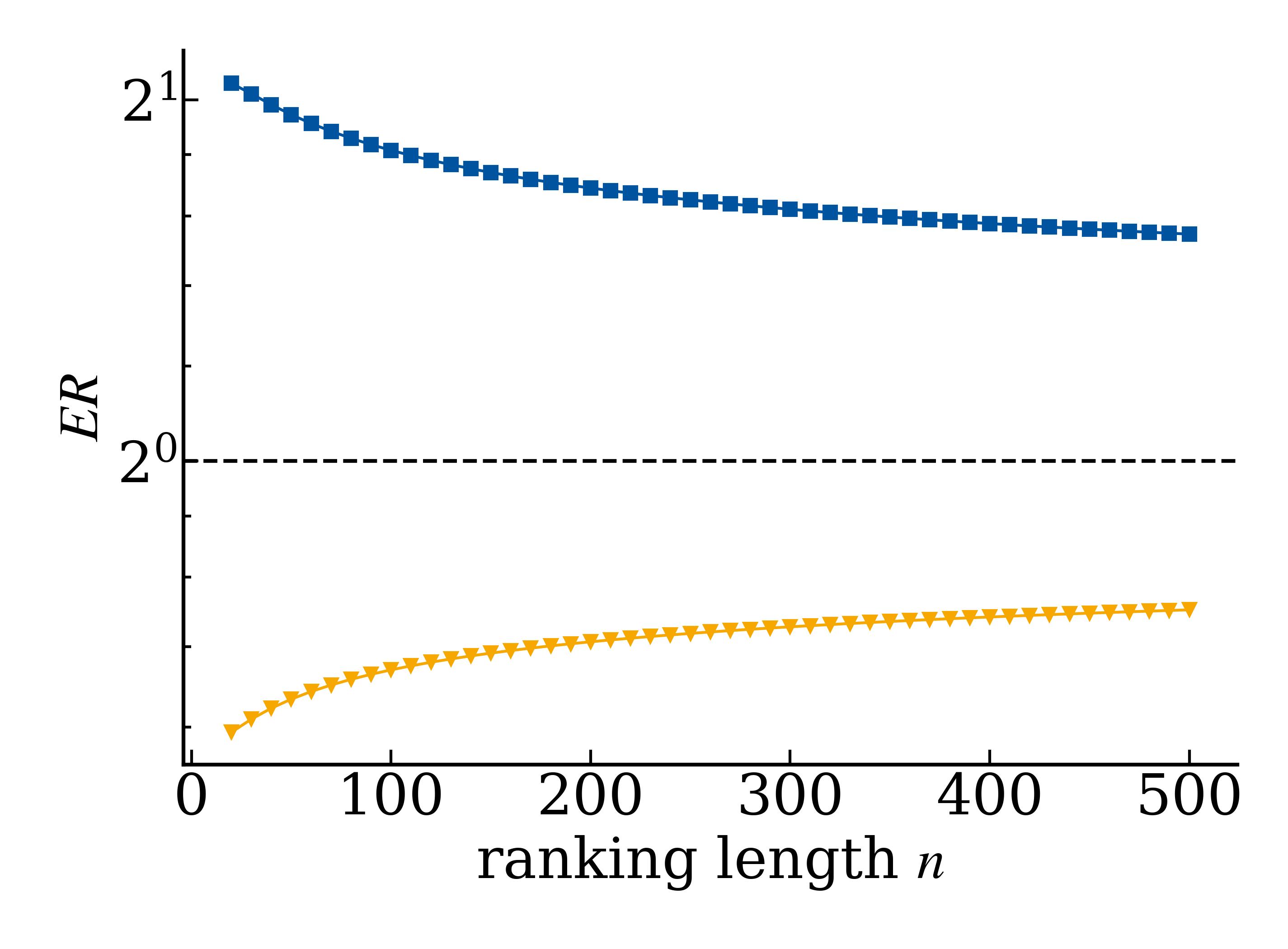}
		\subcaption{\er{}, \dtr{} and \did{} \hspace{0.6mm} \cross}
	\end{subfigure}
	\begin{subfigure}[t]{0.3\textwidth}
		\centering
		\includegraphics[width=\textwidth]{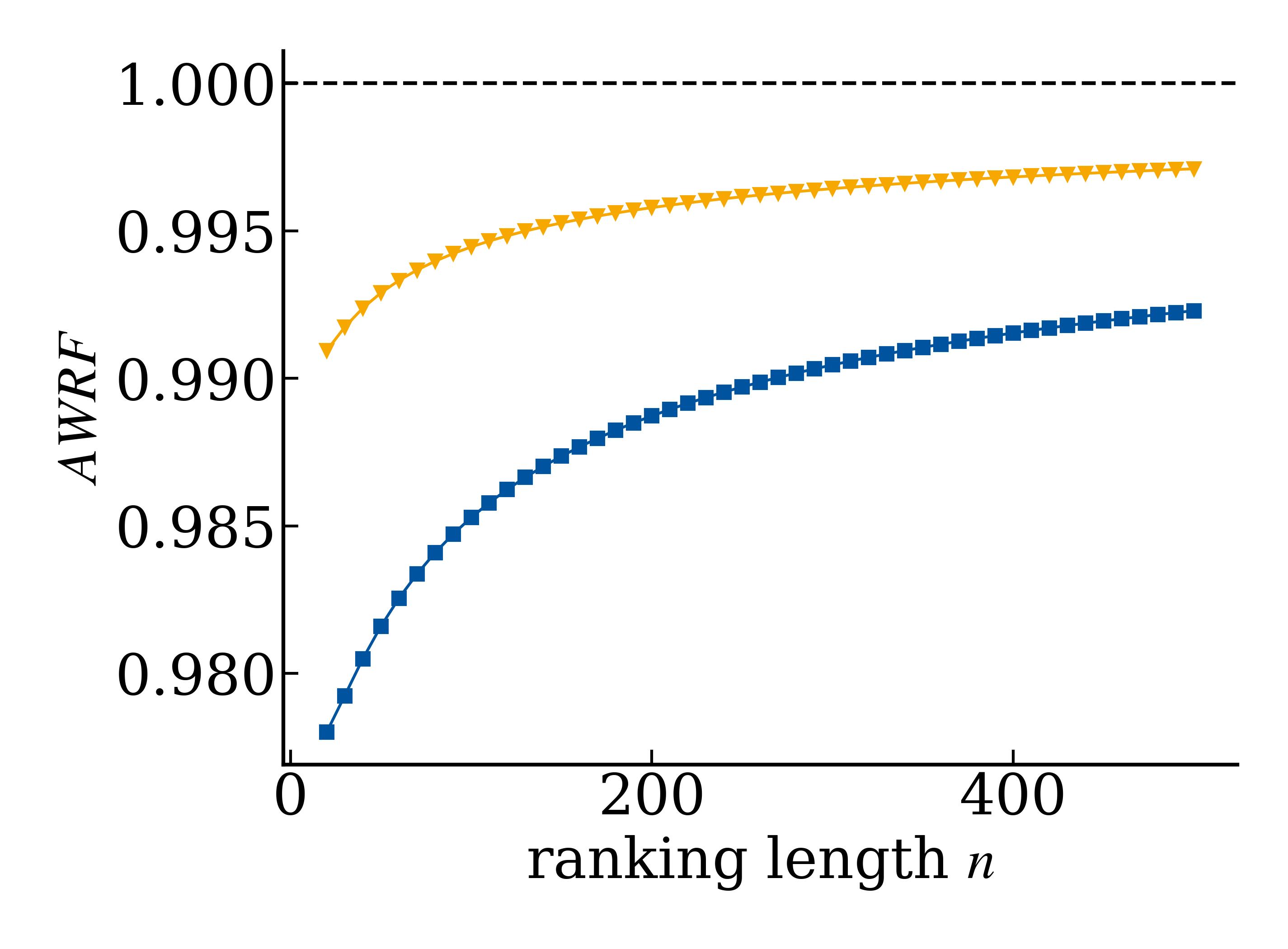}
		\subcaption{att.-weighted rank fairness \awrf{} \hspace{0.6mm} \cross}
	\end{subfigure}
    \caption{\emph{Property 8: invariance to ranking length.} We show the behavior of fairness metrics for varying ranking length $n$. 
    \texttt{Last} (blue) represents rankings in which all candidates from the protected group are ranked higher than each candidate from the non-protected group.
    Conversely, \texttt{first} (yellow) represents rankings in which all candidates from the protected group are ranked lower than each candidate from the non-protected group.
    The proportion of the protected group is fixed at $p_{G_1} = 0.3$, but qualitatively the results are the same for other values of $p_{G_1}$.
    If the markers of the different ranking types each remain at a constant value, the respective metric satisfies \emph{invariance to ranking length}. 
    We assume uniform relevance, therefore \dtd{}, \did{}, \dtr{} and \dir{} are equal to \ed{} and \er{}, respectively. 
    We observe that none of the shown metrics seem invariant to ranking length ($\textcolor{red}{\times}$). 
    This implies that for all these metrics, the values obtained on rankings of different length are hardly comparable.}
	\label{fig:ax1}
\end{figure}

\section{Experimental Analysis}
In principal, probing existing metrics for the proposed properties can be done in a completely theoretical manner.
However, we choose to also conduct and present experiments since they give us some more nuanced insights into the behavior of the given metrics, specifically with respect to properties 6 and 8-11.
Thus, we conduct experiments on synthetic and empirical ranking data data to assess whether the presented metrics satisfy the proposed properties.
First, we probe the metrics with respect to properties 8-11, which all assume uniform relevance and are thus easiest to assess via synthetic rankings.
Afterwards, we sample rankings from an empirical dataset that includes real world relevance scores to assess compliance with property 6 (\emph{invariance to linear transformations of relevance scores}). \\
We also mostly neglect \emph{pairwise statistical parity} (\psp) in the following, since this metric will be analyzed in Section \ref{sec:theory}.

\begin{figure}[t]
	\centering
	\begin{subfigure}[t]{0.35\textwidth}
		\centering
		\includegraphics[width=\textwidth]{img/prop/legend}
	\end{subfigure} 
	
	\begin{subfigure}[t]{0.3\textwidth}
		\centering
		\includegraphics[width=\textwidth]{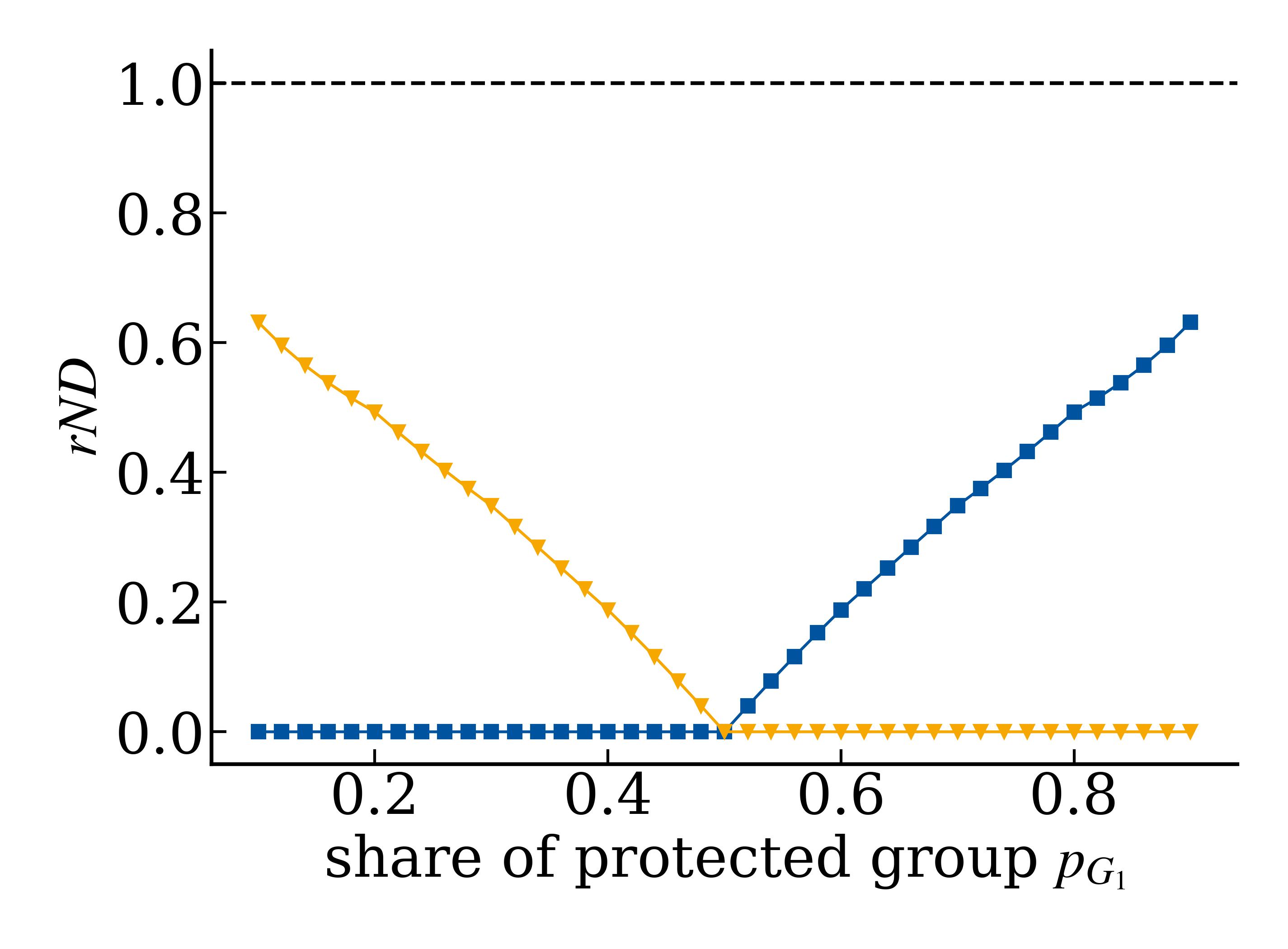}
		\subcaption{normalized disc. difference \rnd{} \hspace{0.8mm} \cross}
	\end{subfigure}
	\begin{subfigure}[t]{0.3\textwidth}
		\centering
		\includegraphics[width=\textwidth]{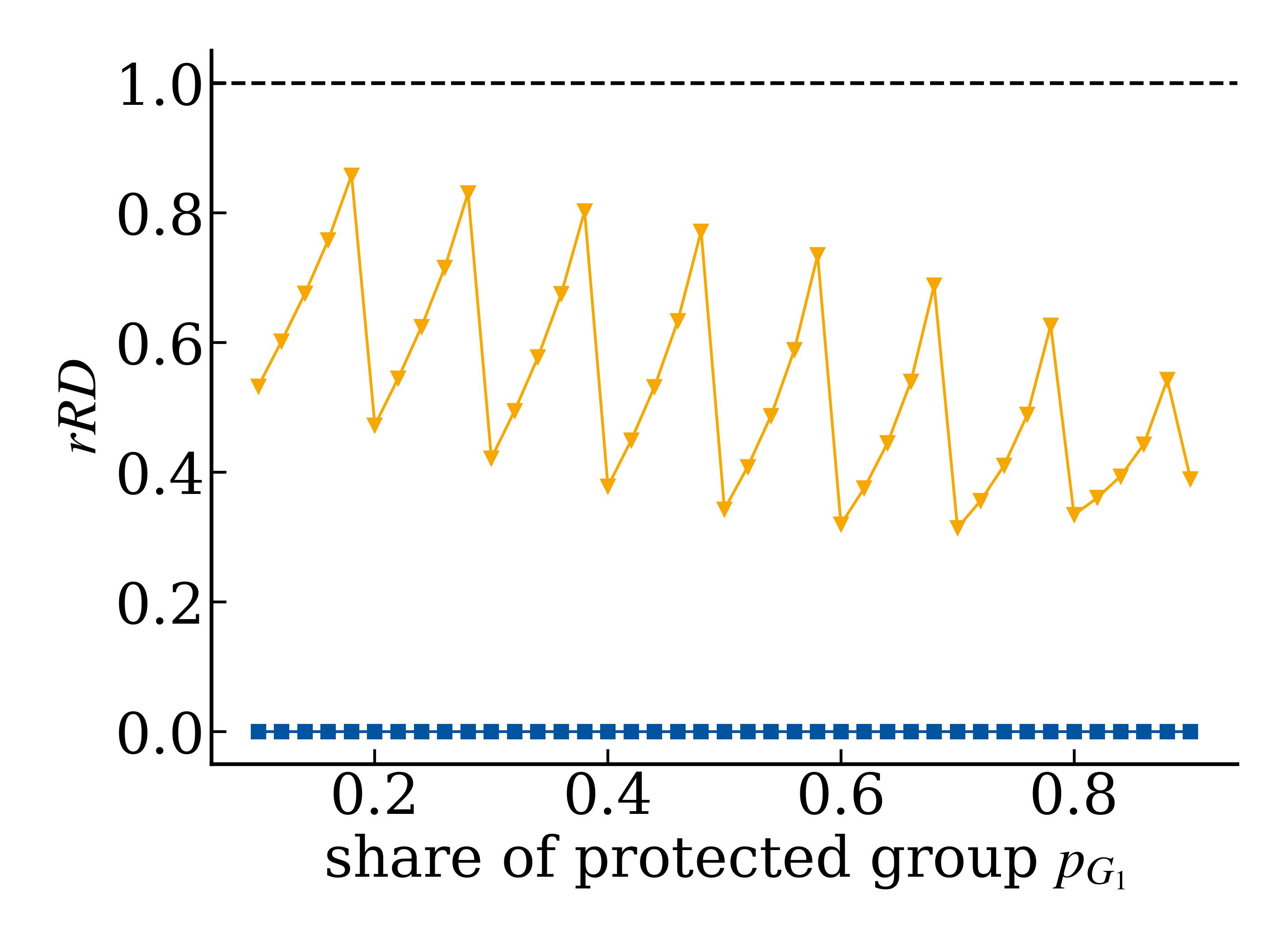}
		\subcaption{normalized disc. ratio \rrd{} \hspace{0.8mm} \cross}
	\end{subfigure}
	\begin{subfigure}[t]{0.3\textwidth}
		\centering
		\includegraphics[width=\textwidth]{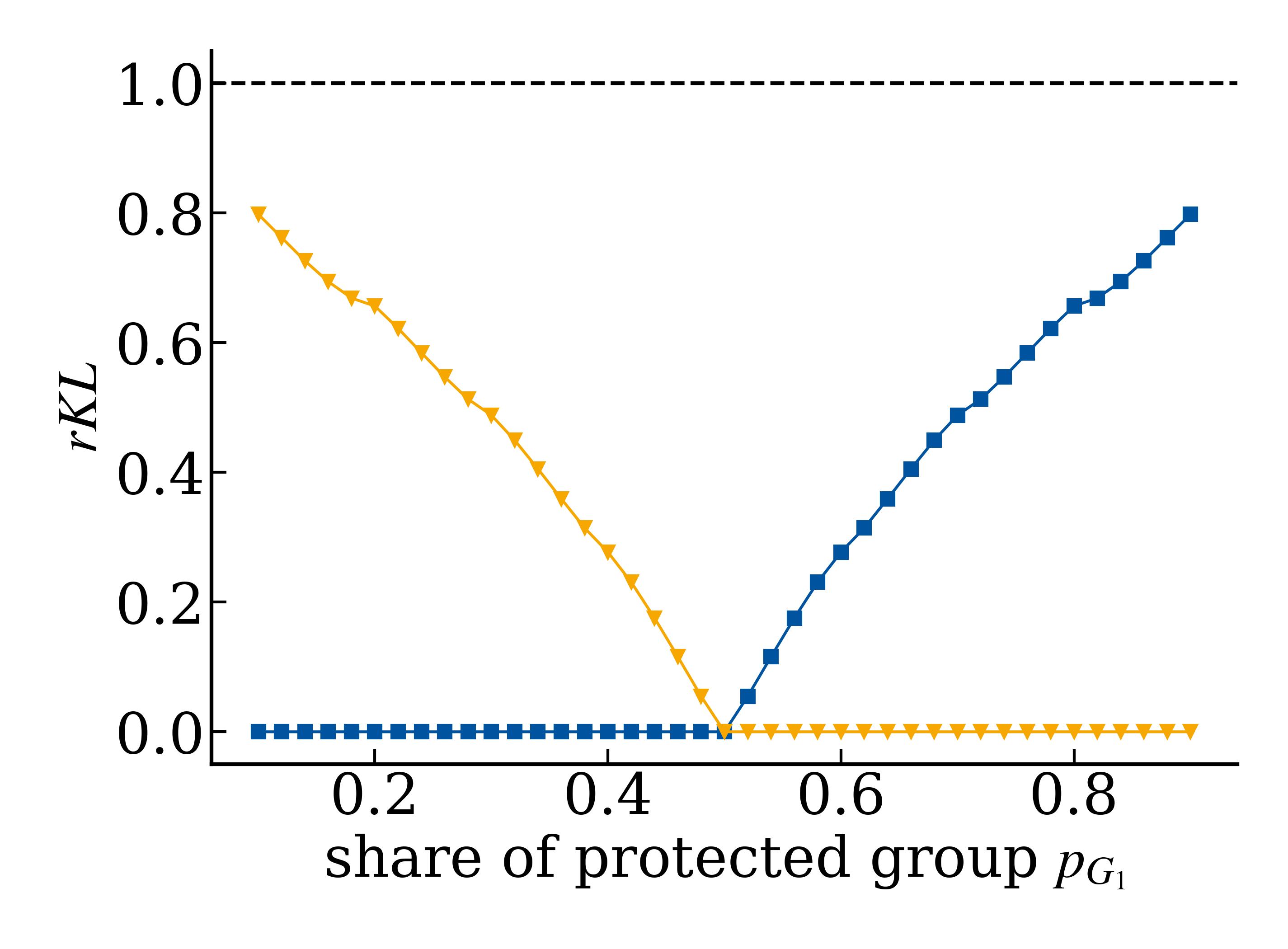}
		\subcaption{norm. disc. KL-divergence \rkl{} \hspace{0.8mm} \cross}
	\end{subfigure}
	\begin{subfigure}[t]{0.3\textwidth}
		\centering
		\includegraphics[width=\textwidth]{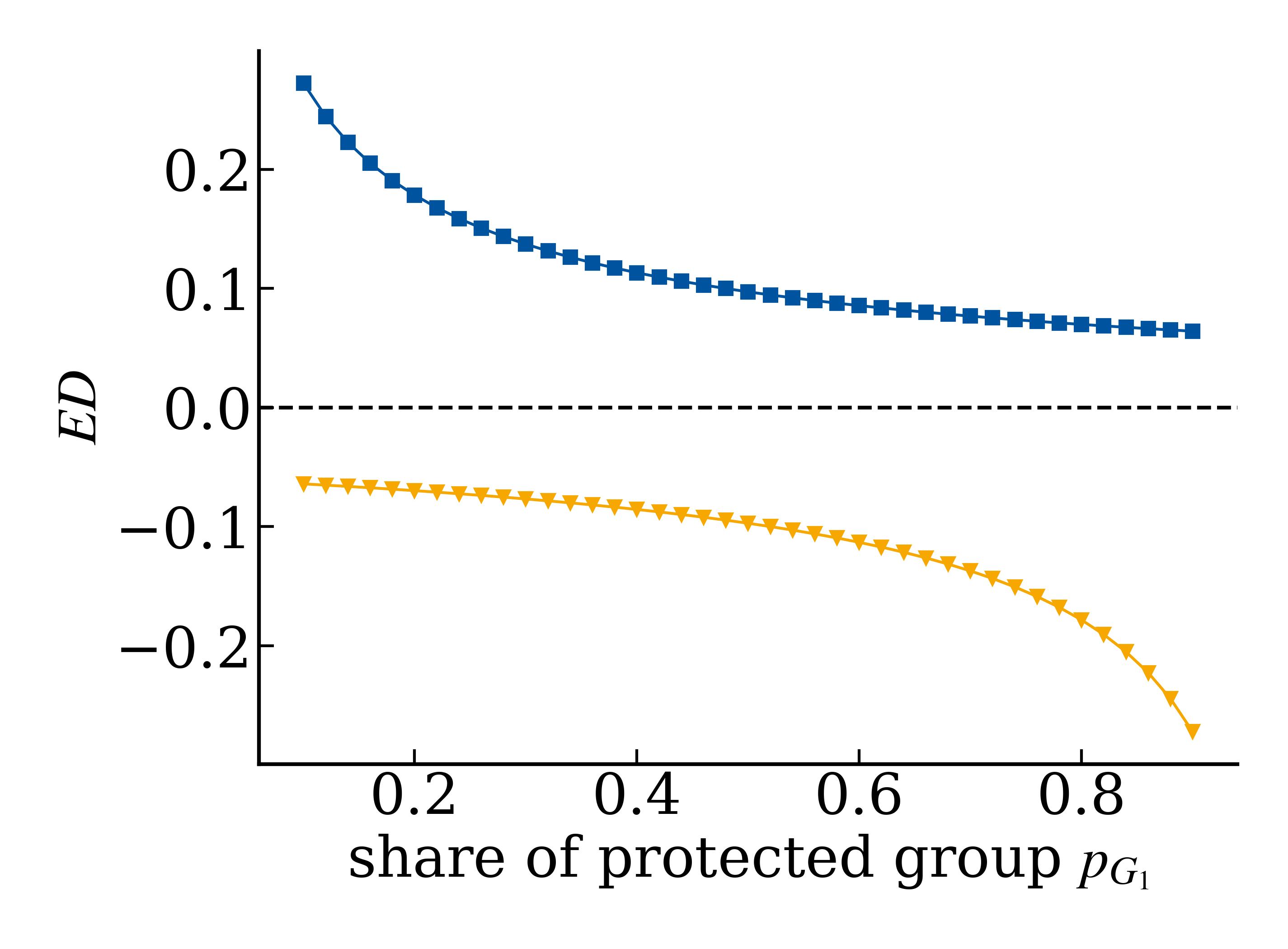}
		\subcaption{\ed{}, \dtd{} and \did{} \hspace{0.8mm} \cross}
	\end{subfigure}
	\begin{subfigure}[t]{0.3\textwidth}
		\centering
		\includegraphics[width=\textwidth]{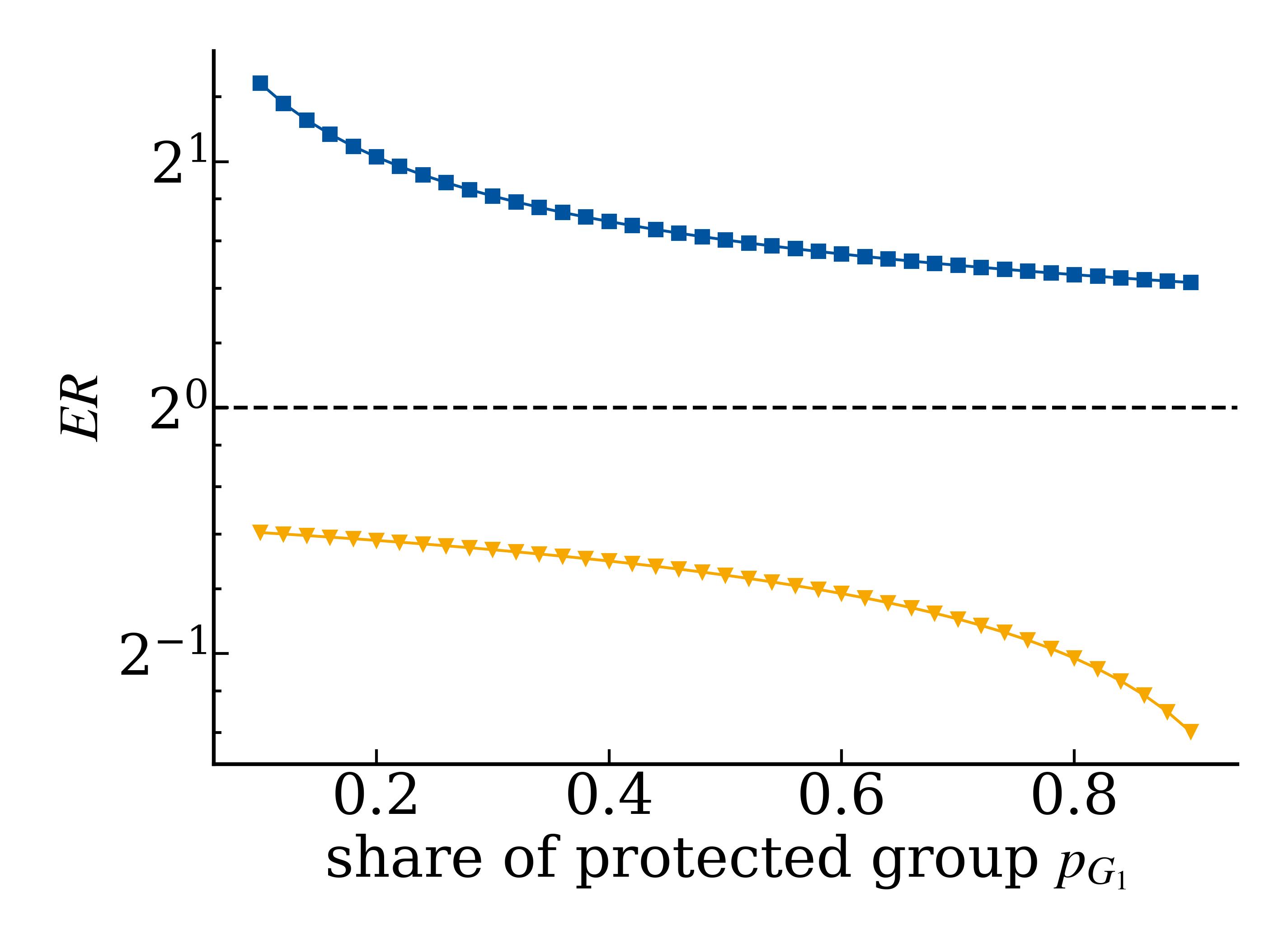}
		\subcaption{\er{}, \dtr{} and \dir{} \hspace{0.8mm} \cross}
	\end{subfigure}
	\begin{subfigure}[t]{0.3\textwidth}
		\centering
		\includegraphics[width=\textwidth]{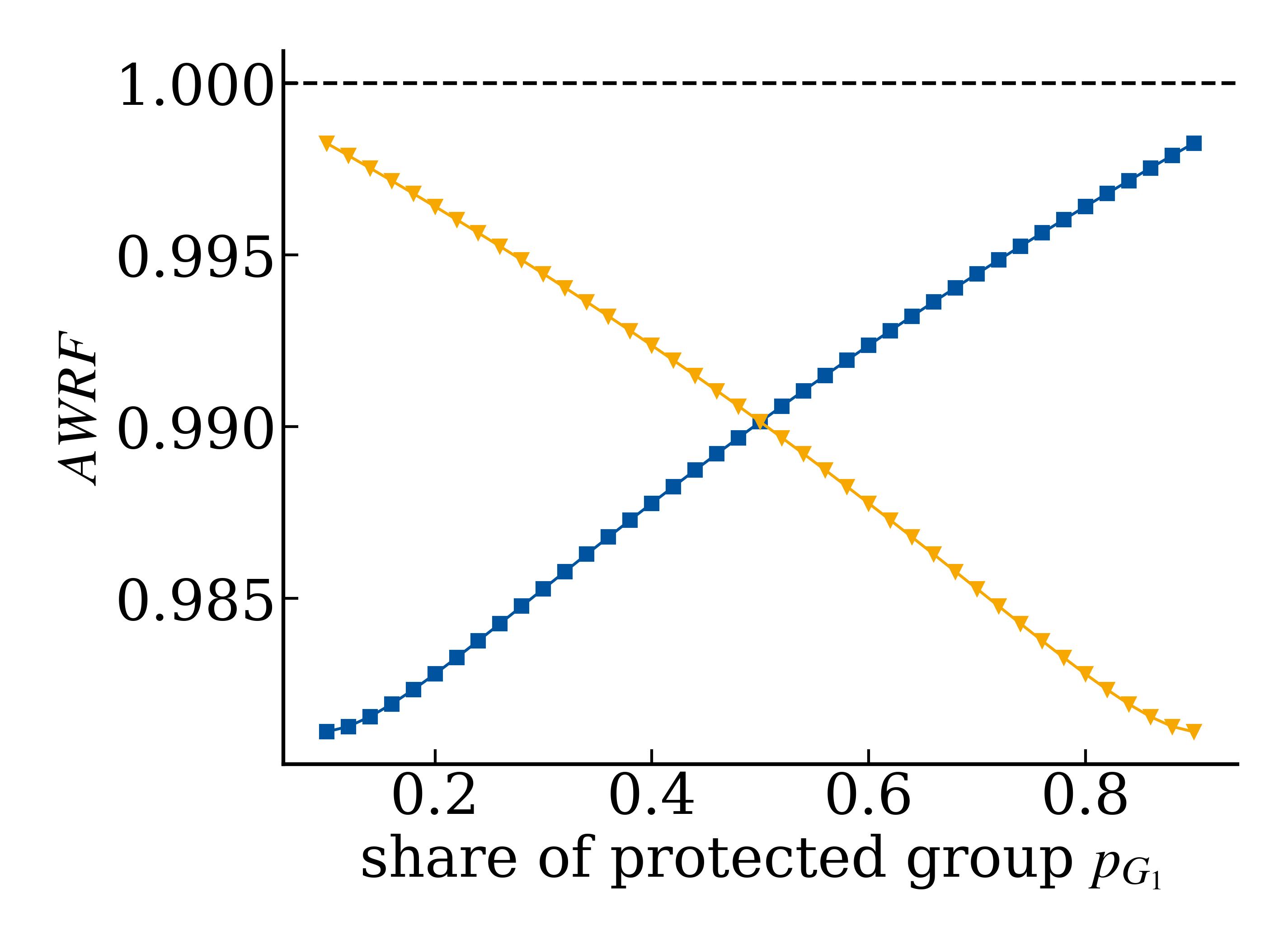}
		\subcaption{att.-weighted rank fairness \awrf{} \hspace{0.8mm} \cross}
	\end{subfigure}
	\caption{\emph{Property 9: invariance to group proportion, and property 10: symmetric penalties for all groups.} We illustrate the behavior of fairness metrics for varying proportions of the protected group $p_{G_1}$. 
    \texttt{Last} (blue) represents rankings in which all candidates from the protected group are ranked higher than each candidate from the non-protected group.
    Conversely, \texttt{first} (yellow) represents rankings in which all candidates from the protected group are ranked lower than each candidate from the non-protected group.
	The results are shown for a fixed ranking length of $n = 100$ but the results are qualitatively the same for other choices of $n$. Assuming uniform relevance, \dtd{}, \did{}, \dtr{} and \dir{} are equal to \ed{} and \er{}, respectively. If the markers of the different ranking types each remain at a constant value, the respective metric satisfies \emph{invariance to group proportions}.
	We can see that this is the case for none of the shown metrics ($\textcolor{red}{\times}$). 
	This implies that these metrics should not be used to compare the fairness of rankings that differ in terms of group shares. 
	With regard to property 10, a metric assigns \emph{symmetric penalties for all groups} if the distance to the optimum is the same for the blue and yellow markers for a fixed group proportion. 
	We observe that none of the shown metrics satisfy this property either.
	}
	\label{fig:ax2}
\end{figure}

\subsection{Experiments on Synthetic Rankings.}\label{sec:synex}

To assess whether the presented metrics satisfy property 8 (\emph{invariance to ranking length}), property 9 (\emph{invariance to group proportions}), property 10 (\emph{symmetric penalties for all groups}), and property 11 (\emph{closeness threshold}), we generate two different types of synthetic rankings:
\begin{enumerate}
	\item[1.]\texttt{first}: all protected candidates are placed at the top,
	\item[2.]\texttt{last}: all protected candidates are placed at the bottom.
\end{enumerate}
Rankings of type \texttt{first} and \texttt{last} correspond to rankings in which one group has the strongest advantage possible over the other group. We assume uniform relevance, hence the scores that each metric $m$ assigns to these ranking types on the candidate sets $D$ correspond to the extreme values $\vf(m,D)$ and $\vl(m,D)$, respectively.
As in this case \dtd{}, \dtr{}, \did{} and \dir{} are equivalent to \ed{} and \er{}, respectively, we do not explicitly report the results for them.

\subheader{Property 8: Invariance to Ranking Length.}
In order to probe the metrics for \emph{invariance to ranking length}, we generate rankings of type \texttt{first} and \texttt{last} and vary the size $n$ of the ranked population $\D$ from $20$ to $500$, with increments of $10$. The share of the protected group remains fixed at $p_{G_1}=0.3$, and the cut-offs of the prefix metrics are set to $I = \{10, 20, 30, ...\}$.
Figure \ref{fig:ax1} shows the behavior of all ranking metrics (aside from \psp{}) with respect to property 8. 
The plots indicate that none of the presented metrics are invariant to ranking length.
For the prefix metrics \rnd{} and \rkl{}, we observe that the extreme value $\vl(\D)$ is lower for shorter rankings than for longer rankings, and that it appears to converge to a fix value with increasing population size $n$. 
Regarding the third prefix metric \rrd{}, we observe much more variance in the extreme value $\vl(\D)$, presumably depending on whether the exact group shares of the overall ranking are numerically achievable in the respective ranking prefixes.
Further, for all prefix metrics we observe that in the given population $\D$, in almost all of the shown cases the extreme value $\vl(\D)$ is bigger than $\vf(\D)$, from which it directly follows that these metrics do not satisfy property 3 (\emph{monotonicity}).
For the exposure metrics, we observe that the extreme values $\vl(\D)$ and $\vf(\D)$ appear to be impacted significantly by the ranking length, seemingly converging toward their optimal values $\vopt$ with increasing $n$. 
For \awrf{}, it appears that $\vl(\D)$ and $\vf(\D)$ both increase along with increasing ranking length.
This increase appears to be the  most pronounced for smaller values of $n$, slowly decaying for longer rankings.
Further, we observe that under \awrf{}, the extreme value $\vl(\D)$ seems to be consistently higher than the value $\vf(\D)$, 
and thus this metric does not satisfy \emph{monotonicity} either.


\subheader{Property 9: Invariance to Group Proportions.}
To examine whether the presented metrics satisfy \emph{invariance to group proportions}, we fix the ranking length at $n = 100$ and vary the proportion of the protected group $p_{G_1}$ from $0.1$ to $0.9$ in increments of $0.02$. 
Again, the cut-offs of the prefix metrics are set to $I = \{10, 20, 30, ...\}$.
The behavior of the metrics is shown in Figure \ref{fig:ax2}, where we find that none of the presented metrics are invariant to group proportions. 
The plots illustrate that \emph{rND} and \emph{rKL} take on their highest value when the group that forms the majority is ranked last, regardless of whether it is the protected or non-protected group.
For the minority group, the corresponding values $\vf(\D)$ and $\vl(\D)$ appear to increase/decrease in an almost linear fashion.
For the third prefix metric \emph{rRD}, we observe that like in the case of varying ranking length, only the value $\vf(\D)$ appears invariant to such variations, whereas $\vl(\D)$ again shows strong variations. 
For all exposure metrics, we observe that both extreme values $\vf(\D)$ and $\vl(\D)$ decrease significantly with increasing share of the protected group. 
\awrf{} appears to behave similar to \rnd{} and \rkl{}, with the value $\vl(\D)$ decreasing and the value $\vf(\D)$ increasing with increasing share of the protected group $p_{G_1}$, both in a linear fashion, with these values appearing to be equal at $p_{G_1}=0.5$.

\begin{figure}[t]
	\centering
	\begin{subfigure}[b]{0.35\textwidth}
		\centering
		\includegraphics[width=\textwidth]{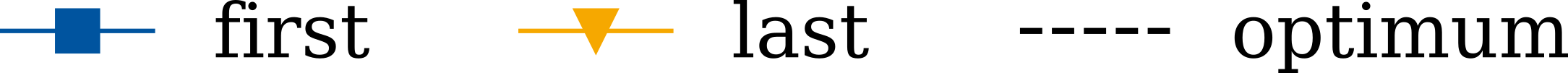}
	\end{subfigure} 
	
	\begin{subfigure}[t]{0.24\textwidth}
		\centering
		\includegraphics[width=\textwidth]{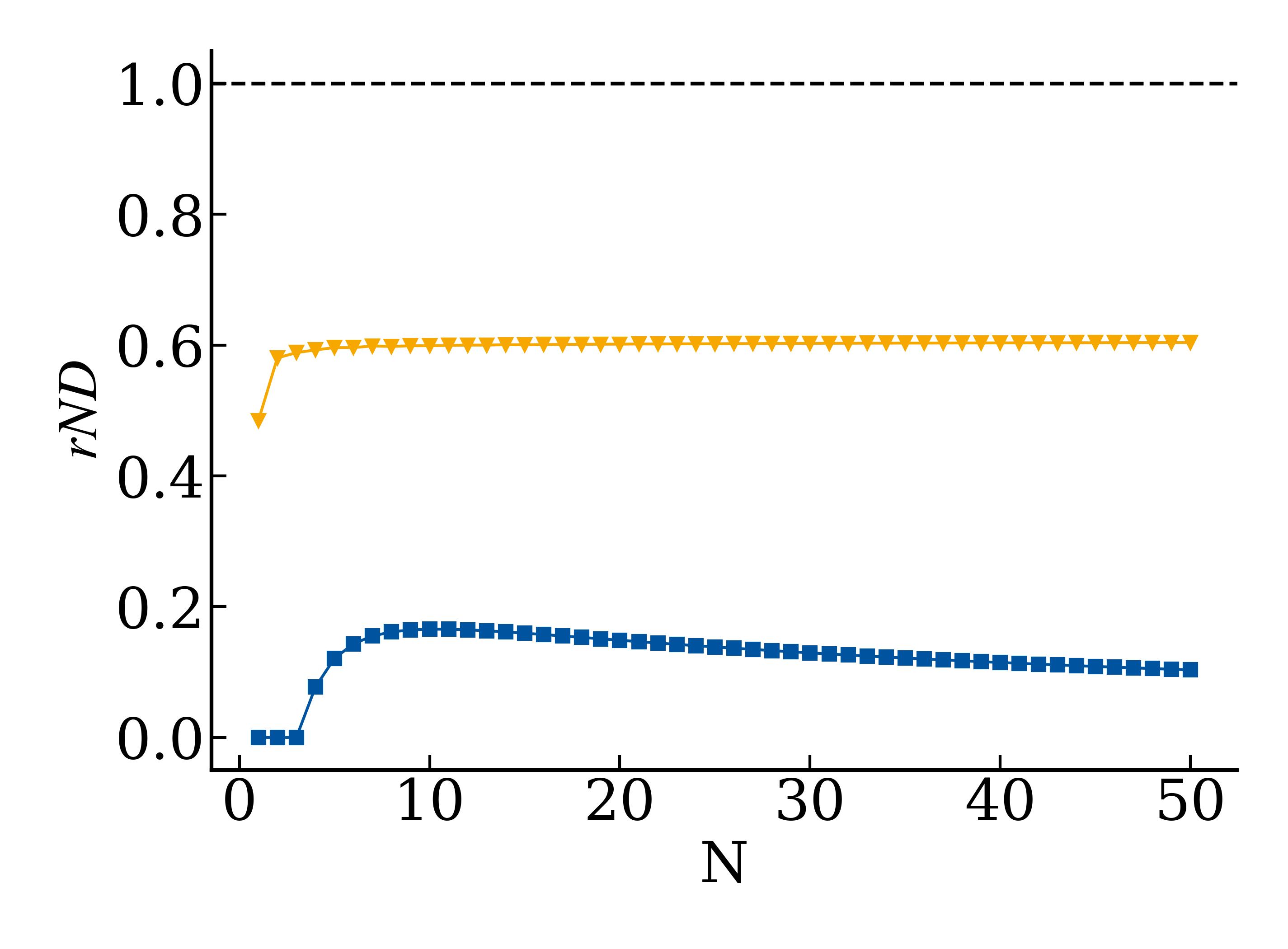}
		\subcaption{norm. disc. diff. \rnd{} \cross}
	\end{subfigure}
	\begin{subfigure}[t]{0.24\textwidth}
		\centering
		\includegraphics[width=\textwidth]{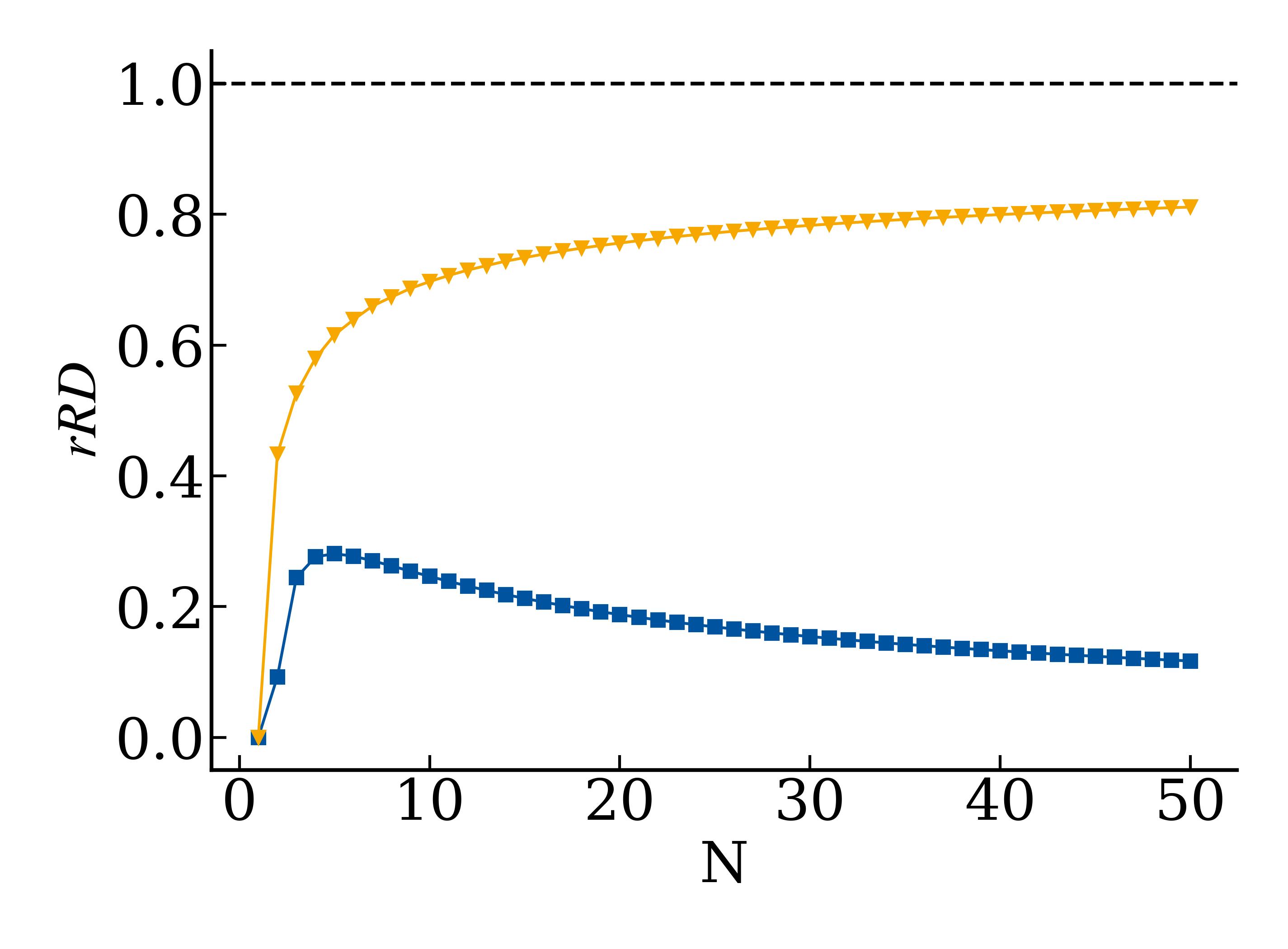}
		\subcaption{norm. disc. ratio \rrd{} \cross}
	\end{subfigure}
	\begin{subfigure}[t]{0.24\textwidth}
		\centering
		\includegraphics[width=\textwidth]{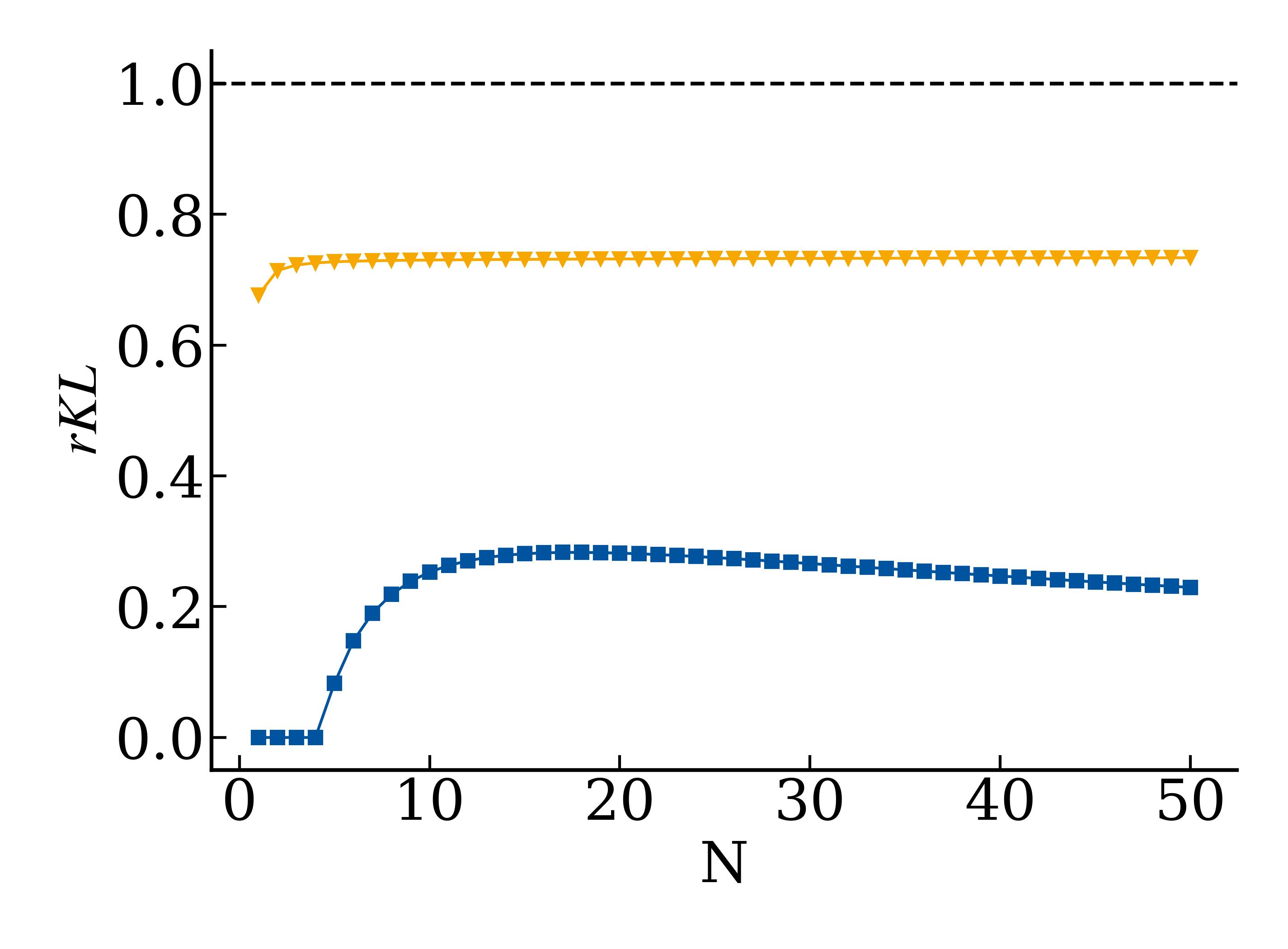}
		\subcaption{norm. disc. KL-div. \rkl{} \cross}
	\end{subfigure}
	\begin{subfigure}[t]{0.24\textwidth}
		\centering
		\includegraphics[width=\textwidth]{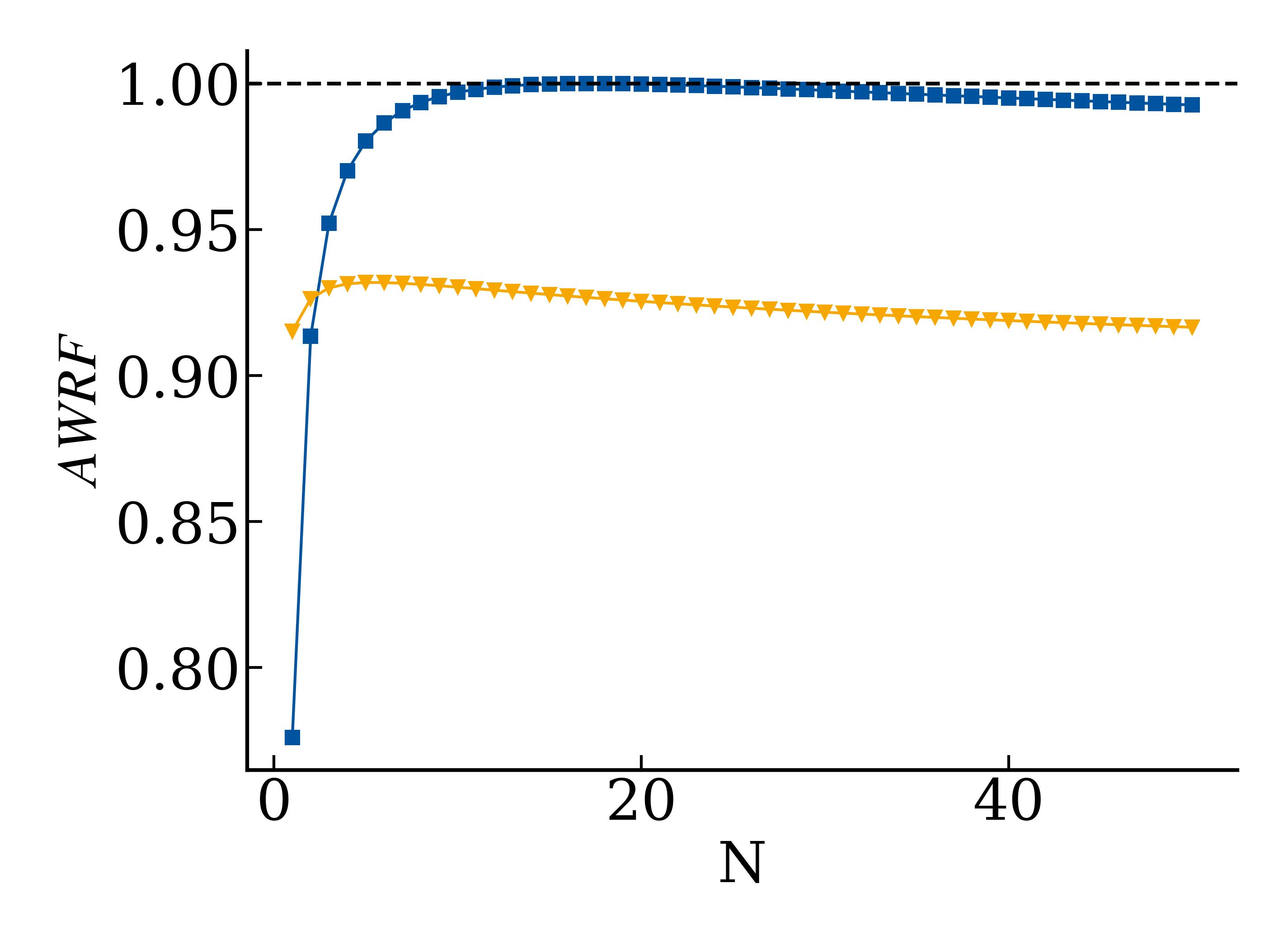}
		\subcaption{att. rank fairness \awrf{} \cross}
	\end{subfigure}
	\caption{\emph{Property 11: closeness threshold.} We illustrate two kinds of rankings for varying ranking length $n = 2N$. In the \texttt{first} rankings (blue), there is only a single candidate from the protected group present, but this candidate is placed at the first ranking position. In the \texttt{last} rankings (yellow), there are $N$ candidates from the protected group, all placed at the bottom half of the ranking. The fraction of the protected group in the population is at $P_{G_1}=0.3$ for the prefix metrics (\rnd, \rrd{} and \rkl{}) and at $P_{G_1}=0.1$ for attention-weighted rank fairness (\awrf). 
	In the plots we see that even for the lowest values $N\rightarrow 1$, the \texttt{last} ranking still yields higher fairness scores with respect to \rnd{}, \rrd{}, \rkl{}, and \awrf{} than the \texttt{first} ranking. This implies that property 11 is violated ($\textcolor{red}{\times}$).}
	\label{fig:ax11}
\end{figure}

\subheader{Property 10: Symmetric Penalties for Both Groups.}
Our experiments probing existing metrics for property 9 allow us to gain insights regarding \emph{symmetric penalties for both groups} as well.
Consequently, in Figure \ref{fig:ax2} one can also see that 
none of the depicted metrics assign \emph{symmetric penalties for both groups}. 
In general, we note that except for \rrd{}, it does not seem to matter whether a group is protected or not, but only whether it forms the minority or majority of candidates. Even more, we observe that all illustrated metrics except \rrd{}, assign rankings in which the minority is disadvantaged values that are closer to the optimal value $\vopt$ and thus suggest a fairer outcome than when the majority group is disadvantaged. It appears that symmetric penalties can only be achieved when both groups have the same cardinality.
For \rrd{}, we find that rankings in which the protected group is ranked last always achieve values closer to the optimum than rankings in which the protected group is ranked first.

Overall, these results illustrate that for most metrics, unbalanced group ratios lead to strong, asymmetric discrepancies in the fairness scores of the two extreme case rankings. In cases where one group forms a very small minority, disadvantaging this very group can still lead to fairness values close to the optimum, which implies that in settings where the minority is already experiencing discrimination, these metrics should be applied with caution.

\subheader{Property 11: Closeness Threshold.}
Figure \ref{fig:ax11} illustrates the behavior of \rnd{}, \rrd{}, \rkl{}, and \awrf{} on rankings $r\in\mathcal{R}_{\operatorname{first}}(D_N)$ and $r'\in\mathcal{R}_{\operatorname{last}}(D_N')$ for varying ranking length $n=2N$. 
For the prefix metrics, the proportion of the protected group in the full population equals $P_{G_1} = 0.3$, and for \awrf{}, we have $P_{G_1} = 0.1$.
We observe that for each of the shown metrics, at $N=1$ the ranking $r'\in\mathcal{R}_{\operatorname{last}}(D_N')$ does not have a lower fairness score than the ranking $r\in\mathcal{R}_{\operatorname{first}}(D_N)$.
From this, it follows that none of these metrics have a \emph{closeness threshold}. 
Even more, for \awrf{} it also appears that when $N>2$, the ranking $r\in\mathcal{R}_{\operatorname{first}}(D_N)$ has a higher value than the ranking $r'\in\mathcal{R}_{\operatorname{last}}(D_N')$, which would mean that property 12 \emph{(deepness threshold)} is not satisfied either.
In Section \ref{sec:theory}, we will confirm that this observation indeed holds for all $N>2$, and also prove that the exposure metrics (not shown here) satisfy properties 11 and 12.

\subsection{Experiments on Empirical Ranking Data.}
Property 6 (\emph{invariance to linear transformation of relevance scores}) can only be tested on rankings that include ground truth relevance scores. 
Thus, we use real world ranking data to obtain a realistic relevance score distribution. 
We use ranking data from the TREC Fair Ranking Track 2021 \cite{trec-fair-ranking-2021}. 
This dataset considers a set of Wikipedia articles that have to be ranked according to how much editing work they need.
The articles are grouped based on continents that are associated with the article content.
We select articles associated with only Northern America as protected group and articles associated with a single different continent as non-protected group. 
Each article can be assigned to multiple topics which are considered the queries in this dataset. 
For each query, we rank the associated set of articles $d$ according to their ground-truth relevance scores $y(d)\in[0,1]$ that indicate the quality of each article.
Using this dataset, we could then both analyze \emph{invariance to translation of relevance scores} and \emph{invariance to rescaling of relevance scores} experimentally, by applying corresponding linear functions on the relevance scores.

\begin{figure}[t]
	\centering
	\begin{subfigure}[b]{0.85\textwidth}
		\centering
		\includegraphics[width=\textwidth]{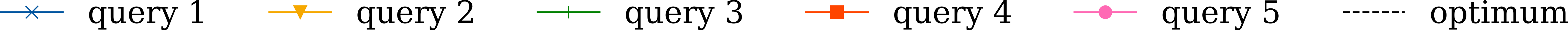}
	\end{subfigure} 
	
	\begin{subfigure}[t]{0.24\textwidth}
		\centering
		\includegraphics[width=\textwidth]{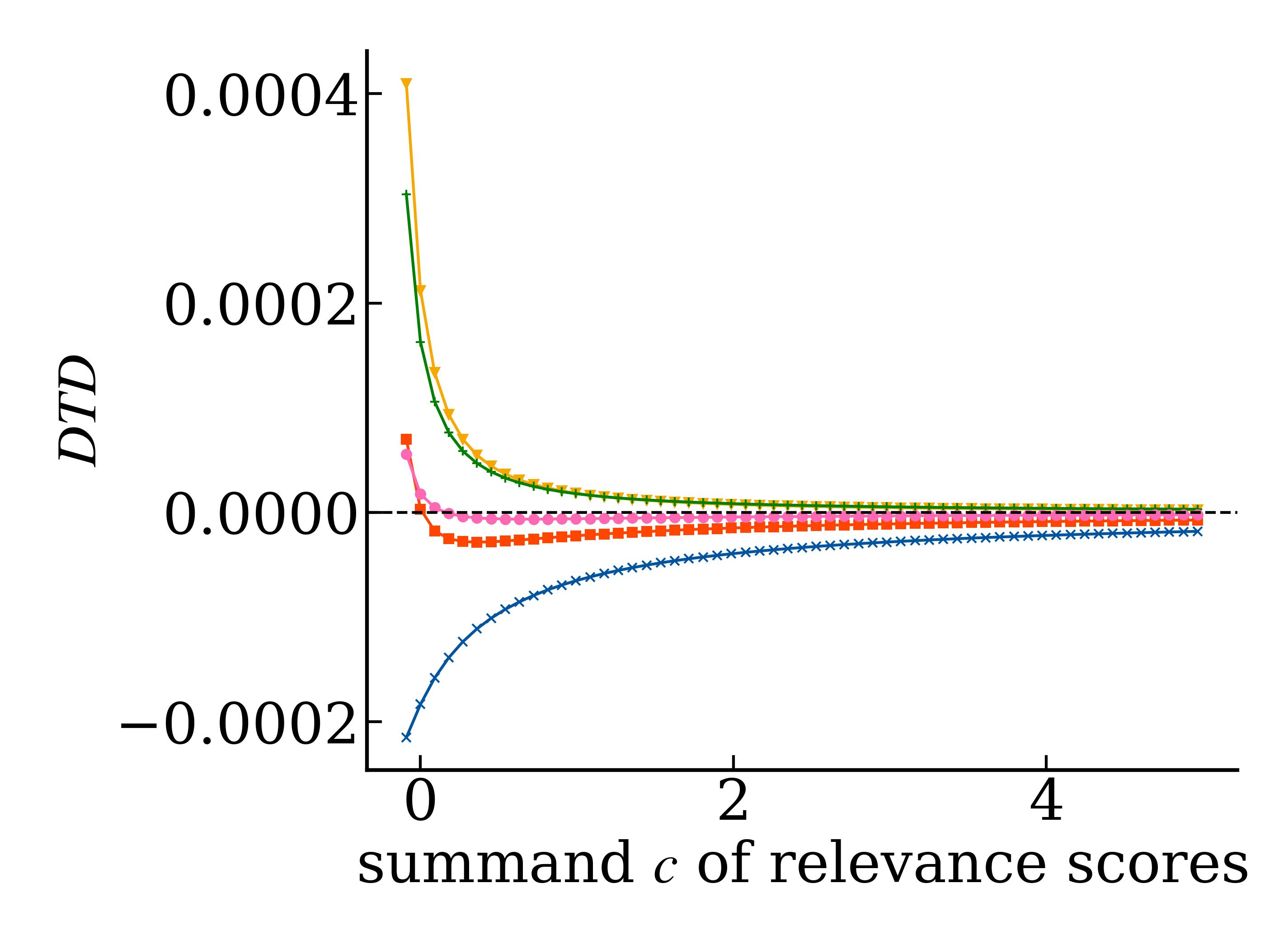}
		\subcaption{disp. treatment diff. \dtd{}  \cross}
	\end{subfigure}
	\begin{subfigure}[t]{0.24\textwidth}
		\centering
		\includegraphics[width=\textwidth]{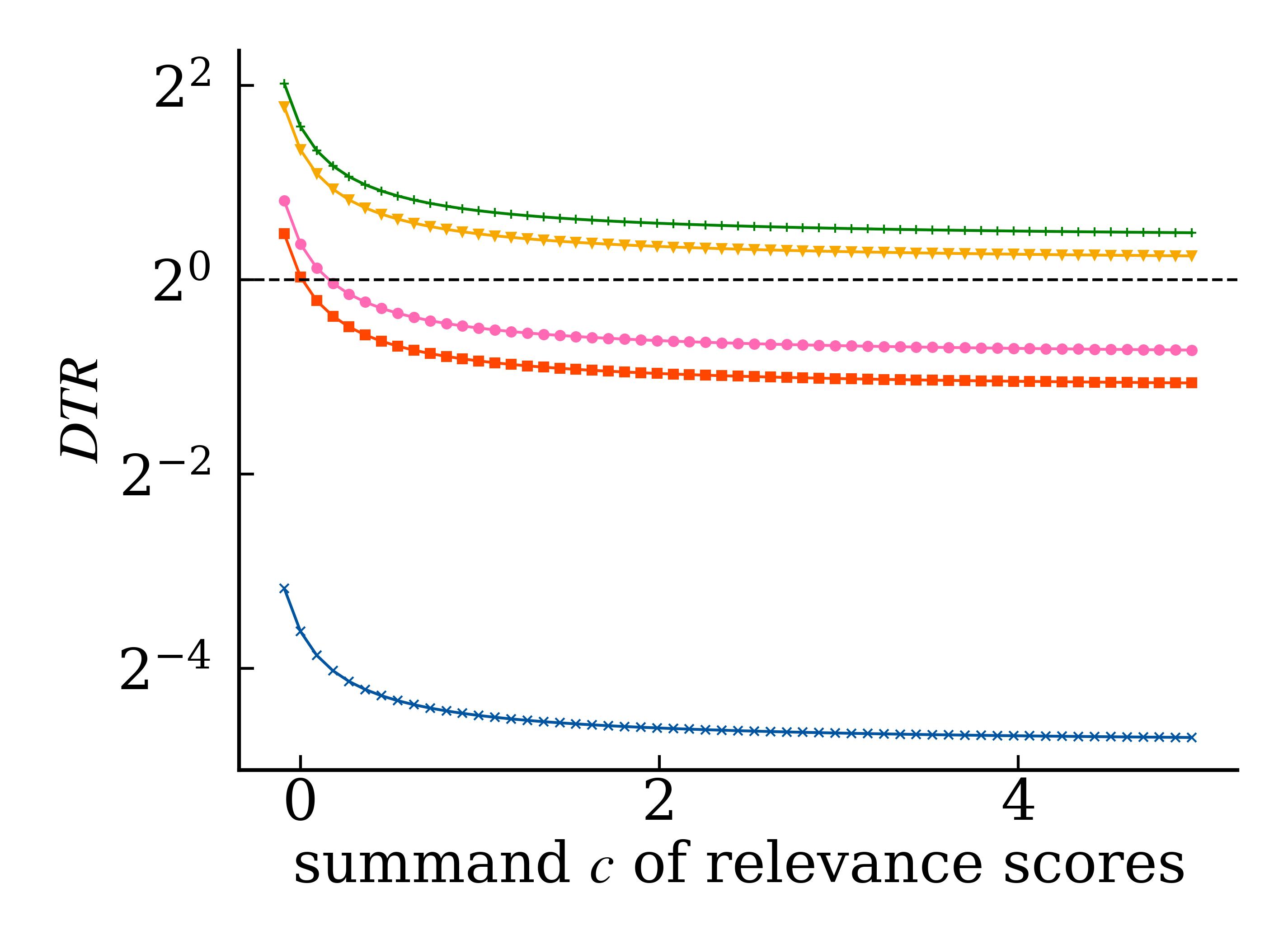}
		\subcaption{disp. treatment ratio \dtr{} \cross}
	\end{subfigure}
	\begin{subfigure}[t]{0.24\textwidth}
		\centering
		\includegraphics[width=\textwidth]{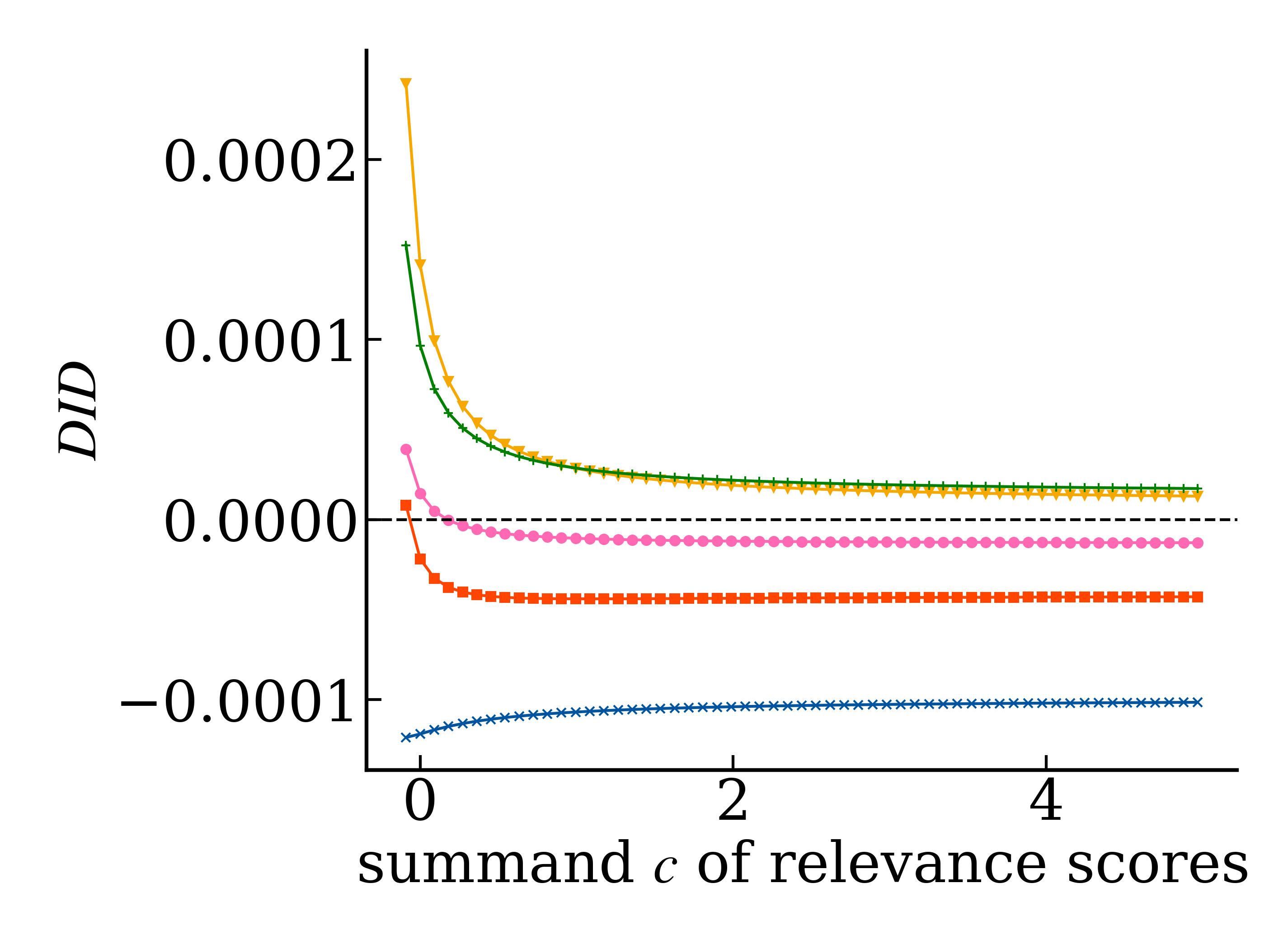}
		\subcaption{disp. impact diff. \did{} \cross}
	\end{subfigure}
	\begin{subfigure}[t]{0.24\textwidth}
		\centering
		\includegraphics[width=\textwidth]{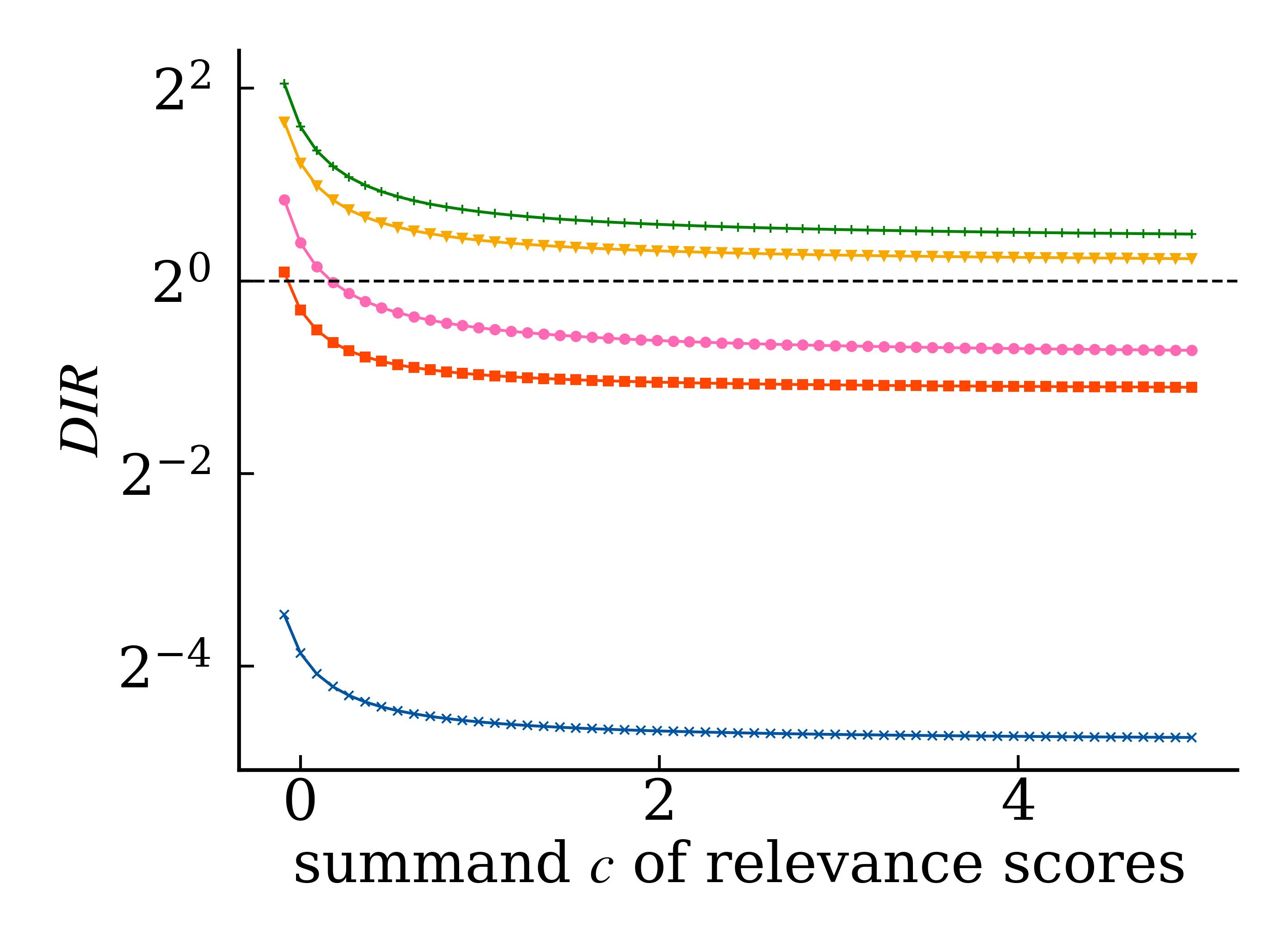}
		\subcaption{disp. impact ratio \dir{} \cross}
	\end{subfigure}
	\caption{\emph{Property 6: invariance to linear transformation of relevance scores.} We show the behavior of \dtd{}, \dtr{}, \did{} and \dir{} for translations of relevance scores by adding different values $c\in[-0.09,5]$ to each ground truth relevance score. The figure shows their fairness scores on the ground truth rankings of five queries retrieved from the TREC 2021 dataset \cite{trec-fair-ranking-2021}, where the ground truth relevance scores are in a range of $[0, 1]$. 
	If the metrics satisfied \emph{invariance to linear transformation of relevance scores}, their values should remain unaffected from these translations.
	We can observe that none of these four metrics are linear transformation of relevance scores ($\textcolor{red}{\times}$). Even more, the fairness scores of queries 4 and 5 always shift from being bigger than the optimum to smaller than the optimum. The results therefore support the usage of ground truth relevance scores without processing.
}
	\label{fig:ax4}
\end{figure}

\subheader{Property 6: Invariance to Linear Transformation of Relevance Scores.}
In Figure \ref{fig:ax4} we illustrate the behavior of \dtd{}, \dtr{}, \did{} and \dir{} on the TREC 2021 dataset under \emph{linear transformations of relevance scores}, or more specifically, under \emph{translations of relevance scores}, for five exemplary queries.
These are the only four metrics for which property 6 is applicable, hence the other metrics are omitted here.
We add a fixed value $c$ to all relevance scores and study whether the fairness scores change when varying $c$ in the interval $[-0.09,5]$.
It can be seen that none of these metrics is invariant to such translations. 
Even more, for queries 4 and 5 the fairness scores of all metrics are shifted from values $m(r)>\vopt$ to values $m(r)<\vopt$.
Further, we observe that for all these metrics, translating the relevance scores has the strongest effect the closer the relevance scores are to 0, which for the given dataset is the case at $b\approx-0.09$.
Consequently, none of these metrics is \emph{invariant to linear transformations of relevance scores}. 
We omit experimental results regarding \emph{invariance to rescaling of relevance scores} for brevity. 
In Section \ref{sec:theory}, we will show that \dtr{}, \did{} and \dir{} satisfy this property, whereas \dtd{} changes in proportion to the rescaling factor $a$. 

\section{Theoretical Analysis}\label{sec:theory}

Next, we provide theoretical results to complement and verify observations from our experiments.
To keep the analysis more short and concise, we mostly provide sketches of proofs for the following results.
Detailed versions of these proofs can be found in Appendix \ref{apdx:proofs}.

\subheader{Properties of Prefix Metrics.}
As their optimal value $\vopt=1$ is also their maximum value, the prefix metrics do not satisfy property 1 (\emph{distinguishability of groups}).
However, it is easy to see that the prefix metrics are \emph{bounded} in the interval $[0, 1]$ due to the normalization of the corresponding sum and the use of absolute values, and thus satisfy property 2.
Further, our experiments have provided examples where on a fixed candidate set $D$, rankings $r\in\mathcal{R}_{\operatorname{last}}(D)$ have lower fairness scores than rankings $r\in\mathcal{R}_{\operatorname{first}}(D)$ and thus, these metrics do not satisfy property 3 (\emph{monotonicity}) either.
Regarding property 4 (\emph{deepness}), we can see that this criterion can be violated if the candidates that are swapped within the ranking are not candidates of the index set $I$ that is summed over.
Since these metrics do not consider relevance scores, properties 5 and 6 are not applicable.
Property 7 (\emph{Optimality of random rankings}) cannot be satisfied since the optimal value $\vopt = 1$ is also the upper bound value of the metric. Regarding properties 8-11, our experiments in Section \ref{sec:synex} have shown that these are not satisfied either.
Finally, for properties 12 and 13, one can show that these properties are not satisfied by any of the prefix metrics either. 

\begin{restatable}{theorem}{prefixthm}\label{thm:prefix1213}
None of the prefix metrics satisfy property 12 (deepness threshold) and property 13 (sensitivity).
\end{restatable}

\begin{proof}[Proof (sketch)]
To show that property 12 is not satisfied, we construct a counterexample on a population $\D$ where $p_{G_1}(\D) = 0.8$ and the index set is chosen as $I=\{N\}$.
Using the same candidate population and index set, we consider a ranking $r'\in\mathcal{R}_{\operatorname{last}}\big(D_N'\big)$ to show that property 13 is violated. 
\end{proof}

\subheader{Properties of Exposure Metrics.}
In our experiments on properties 8-10, the exposure metrics already appeared to satisfy \emph{distinguishability of groups} on the given datasets. We show that this observation indeed holds in general.

\begin{theorem}
All of the exposure metrics satisfy property 1 (distinguishability of groups). 
\end{theorem}

\begin{proof}
Due to uniform relevance we only need to consider \ed{} and \er{}. For any population $\D$ and all rankings $r\in\mathcal{R}_{\operatorname{first}}(\D)$, and $r'\in\mathcal{R}_{\operatorname{last}}(\D)$, from the position bias $b(k)$ being strictly monotonically decreasing it follows that 
$$
\textstyle{
\operatorname{Exposure}(G_1|r) = \frac{1}{|G_1|}\cdot\sum_{k=1}^{|G_1|} b(k) > b(|G_1|) > b(|G_1|+1) > \frac{1}{|G_0|}\cdot\sum_{k=|G_1|+1}^{|\D|} b(k) = \operatorname{Exposure}(G_0|r).}
$$
Analogously, it can be shown that $\operatorname{Exposure}(G_0|r') > \operatorname{Exposure}(G_1|r')$.
Then, by definition of \ed{} and \er{}, it directly follows that these metrics satisfy property 1.
\end{proof}

We proceed with property 2 \emph{(boundedness)}.

\begin{theorem}
Except for $ED$, none of the exposure metrics satisfy property 2 (boundedness).
\end{theorem}

\begin{proof}
Since it holds that $0 < b(k) \leq 1$ for all $k\in\N$, it follows that $0 < \operatorname{Exposure}(G | r) \leq 1$ for each group $G\in\{G_0,G_1\}$ and all rankings $r\in\mathcal{R}(D)$.
Therefore, the exposure difference has to be bounded in the interval $[-1,1]$.
Regarding \er{}, we consider a population $\D$ in which $|G_0|=|G_1|$, a candidate set $D$ with $|D\cap G_0| = 1$, and a ranking $r\in\mathcal{R}_{\operatorname{first}}(D)$.
In that setting it holds that 
$$
\textstyle
ER(r) = \frac{|G_0|\cdot\sum_{k=1}^{n-1} \frac{1}{\log(k+1)}}{|G_1|\cdot \frac{1}{\log(n+1)}} = \log(n+1) \cdot \sum_{k=1}^{n-1} \frac{1}{\log(k+1)} \geq \log(n+1)\xrightarrow{n\to\infty}\infty
$$
and thus, \er{} does not satisfy \emph{boundedness}. Since \er{} can be considered a special case of both \dtr{} and \dir{}, it directly follows that these two metrics do not satisfy \emph{boundedness} either. For \did{}, it holds that this metric is not well-defined when $Y(G_0)=0$ or $Y(G_1)=0$ due to division by zero. Thus, this metric does not satisfy \emph{boundedness} either.
\end{proof}

We continue with discussing \emph{monotonicity} and \emph{deepness}.

\begin{restatable}{theorem}{monodeepthm}
All exposure metrics satisfy property 3 (monotonicity) and property 4 (deepness).
\end{restatable}


\begin{proof}[Proof (sketch)]
For both \emph{monotonicity} and \emph{deepness}, the considered candidate swaps only affect the $\operatorname{Exposure}(G|r)$ and $CTR(G|r)$ terms, essentially swapping a pair of position biases $b(i),b(j)$. 
The satisfaction of both properties by all exposure metrics then follows from the position bias function being strictly monotonically decreasing and convex.
\end{proof}

\dtd{}, \dtr{}, \did{} and \dir{} are the only four metrics under study that also consider relevance scores. 
We find that two of these metrics satisfy \emph{intra-group fairness}.

\begin{restatable}{theorem}{intragroupthm}
$DID$ and $DIR$ satisfy property 5 (intra-group fairness), while this is not the case for \dtd{} and \dtr{}.
\end{restatable}


\begin{proof}[Proof (sketch)]
From their definition we can directly infer that \dtd{} and \dtr{} remain unaffected by intra-group swaps.
For \did{} and \dir{}, we solve the inequation $CTR(G|r) < CTR(G| \rij) $ for the relevance scores of the swapped candidates to show that \emph{intra-group fairness} is satisfied.
\end{proof}

In the experimental results in Section 4.1, we have omitted plots regarding the \emph{invariance to rescaling of relevance scores} of exposure metrics, since for three metrics, namely \dtr{}, \did{}, and \dir{}, the values appeared constant under such transformations.
In Appendix \ref{apdx:proofs}, we show that these metrics indeed satisfy \emph{invariance to rescaling of relevance scores}, whereas this is not the case for \dtd{}.

\begin{restatable}{theorem}{scalethm}\label{thm:prop6}
$DTR$, $DID$, and $DIR$ satisfy property 5 (invariance to rescaling of relevance scores).
For $DTD$, it holds that $DTD\big(r\big(f_{a,0}(D)\big)\big) = \tfrac{1}{a}\cdot DTD\big(r(D)\big)$.
\end{restatable}

For Setting 1 where the full population is ranked, our experiments have shown that the exposure metrics do not satisfy properties 8-10. 
Next, we investigate property 7 (optimality of random rankings).

\begin{restatable}{theorem}{randomthm}
$ED$, $DTD$, and $DID$ are the only exposure metrics that satisfy property 7 (optimality of random rankings). 
\end{restatable}
\begin{proof}[Proof (sketch)]
Due to uniform relevance we only need to consider \ed{} and \er{}.
Since \ed{} is essentially a linear combination of position biases $b(k)$, \emph{optimality of random rankings} follows from the linearity of expected values.
For \er{}, on a population $\D = \{d_0, d_1\}$ with $d_0\in G_0$ and $d_1\in G_1$ this property is not satisfied.
\end{proof}

Finally, we look into the properties for Setting 2, in which subsets of a population are ranked.

\begin{restatable}{theorem}{expothm}
 All exposure metrics satisfy property 11 (closeness threshold),  property 12 
(deepness threshold), and property 13 (sensitivity).
\end{restatable}

\begin{proof}[Proof (sketch)]
Due to uniform relevance, we only need to consider \ed{} and \er{}. 
When computing the fairness scores of these metrics $m\in\{ED,ER\}$ on rankings $r\in\mathcal{R}_{\operatorname{first}}(D_N)$, $r'\in\mathcal{R}_{\operatorname{last}}(D_N')$, the group proportions $p_{G_0},p_{G_1}$  cancel out, and one can directly solve the inequalities $m(r(D_N))\lessgtr m(r'(D_N')$ for $N$ to obtain the sought-for thresholds in properties 11 and 12.
\emph{Sensitivity} follows from $\operatorname{Exposure}(G_0|r)$ being the only term in both \ed{} and \er{} that is affected from appending a single candidate $d'\in G_0$ to a ranking $r$.
\end{proof}

\subheader{Properties of Attention-weighted Rank Fairness.}
\awrf{} by definition only takes on values in $[0,1]$, with $\vopt=1$. Thus, this metric does not satisfy \emph{distinguishability of groups}, but does satisfy \emph{boundedness}.
Our experiments on synthetic data have shown that \awrf{} does not satisfy \emph{monotonicity}, and next we show that it does not satisfy \emph{deepness} either.

\begin{theorem}
$AWRF$ does not satisfy property 4 (deepness).
\end{theorem}

\begin{proof}
Let $\mathcal{D}$ a population with $p_\text{groups}(\mathcal{D}) = (0.56, 0.44)$, $D\subseteq \mathcal{D}$ a candidate set consisting of $n=6$ candidates, and $r\in\mathcal{R}(D)$ a ranking where $r(1),r(3),r(5)\in G_0$, and $r(2),r(4),r(6)\in G_1$.
If we choose $i=3$ and $j=5$, it holds that
\[
|AWRF(r) - AWRF(r_{i\leftrightarrow i+1}) | \approx 1.51\cdot 10^{-5} < 8.62\cdot 10^{-5} \approx |AWRF(r) - AWRF(r_{j\leftrightarrow j+1}) |,
\]
which shows that \emph{deepness} is not satisfied.
\end{proof}

\awrf{} does not consider relevance scores, so properties 5 and 6 are not applicable.
Since the optimal value $\vopt(AWRF)=1$ is also its upper bound, it is trivial that \awrf{} does not satisfy \emph{optimality of random rankings}.
Our experiments have further shown that properties 8-11 are not satisfied either.
To close this part of the analysis, we now show that this metric does not satisfy properties 12 and 13 either.

\begin{restatable}{theorem}{awrfthm}\label{thm:awrf}
$AWRF$ does not satisfy property 12 (deepness threshold) and property 13 (\emph{sensitivity}).

\end{restatable}

\begin{proof}[Proof (sketch)]
Regarding the \emph{deepness threshold}, one can consider a population $\D$ where $p_\text{groups}(\D) = (0.9,0.1)$, and then compare the values $AWRF(r(D_N)), AWRF\big(r'\big(D_N'\big)\big)$ for $N\rightarrow\infty$.
For \emph{sensitivity}, one can construct a counterexample on a candidate set $D'\subseteq \mathcal{D}'$ of $n'=2$ candidates from a population $\mathcal{D}'$ where $p_\text{groups}(\mathcal{D}') = (0.75, 0.25)$.
\end{proof}

\subheader{Properties of Pairwise Metrics.}
By its definition, \psp{} 
is \emph{bounded} in the interval $[-1,1]$. 
We will now show that \psp{} also satisfies \emph{monotonicity}, but does not satisfy \emph{deepness}.
\begin{theorem}
$PSP$ satisfies property 3 (monotonicity), but it does not satisfy property 4 (deepness).
\end{theorem}
\begin{proof}
For all $i<j\leq|\D|$ it holds that swapping candidates $r(i)\in G_1$ and $r(j)\in G_0$ will increase the set cardinality $\big|\{(d, d')\in G_0 \times G_1 : d' \succ_r d\}\big|$, and decrease the set cardinality $\big|\{(d, d')\in G_0 \times G_1 : d \succ_r d'\}\big|$.
Thus, \emph{PSP} satisfies \emph{monotonicity}. In the case that $j = i+1$, the increase/decrease in the cardinalities of the sets 
$\big|\{(d, d')\in G_0 \times G_1 : d' \succ_r d\}\big|$ and $\big|\{(d, d')\in G_0 \times G_1 : d \succ_r d'\}\big|$, respectively, is always exactly 1, independent from the exact value of $i$.
Therefore, \emph{PSP} does not satisfy \emph{deepness}.
\end{proof}

Since \psp{} does not consider relevance scores, properties 5 and 6 are not applicable to it.
We neglected the results of \psp{} with respect to setting 1 in section \ref{sec:synex}, since all associated properties seemed to be satisfied.
We now show that this claim holds in general, and that \psp{} also satisfies \emph{distinguishability of groups}.

\begin{theorem}\label{thrm:psp}
$PSP$ satisfies properties 1, 8, 9, and 10. More precisely, for any candidate population $\D$ it holds that $\vf(PSP,\D) = 1$ and $\vl(PSP,\D) = -1$. 
\end{theorem}

\begin{proof}
By definition of $\mathcal{R}_{\operatorname{first}}$, for every $r\in\mathcal{R}_{\operatorname{first}}(\D)$ it holds that $\big|\{(d, d')\in G_0 \times G_1 : d' \succ_r d\}\big| = 0$ and thus $\big|\{(d, d')\in G_0 \times G_1 : d \succ_r d'\}\big| = \big|G_0 \times G_1 \big|$. 
From that, it directly follows that $PSP(r)=1$.
Analogously, it can be shown that $PSP(r)=-1$ for all $r\in\mathcal{R}_{\operatorname{last}}(\D)$.
\end{proof}

The only property left for consideration with respect to \psp{} is \emph{optimality of random rankings}, given that \psp{} is not applicable in Setting 2. For this property, we obtain the following result.

\begin{restatable}{theorem}{pspthm}
$PSP$ satisfies property 7 (optimality of random rankings).
\end{restatable}

\begin{proof}[Proof (\emph{sketch})]
One can show that inverting a ranking $r$, i.e., moving all candidates from their rank $k$ to rank $n+1-k$, yields the negative \psp{} score of the original ranking. 
From this relationship it follows that $v_E(PSP)=\vopt=0$.
\end{proof}


\begin{sidewaystable}
        \centering
	    \resizebox{\linewidth}{!}{\tiny
		\begin{tabularx}{\textwidth}{XlXccccccccccccc}
			\hline
			\textbf{Metric} & \textbf{Definition} & \textbf{Description} & \textbf{P 1} & \textbf{P 2} & \textbf{P 3} & \textbf{P 4} & \textbf{P 5} & \textbf{P 6} & \textbf{P 7} & \textbf{P 8} & \textbf{P 9} & \textbf{P 10} & \textbf{P 11} & \textbf{P 12} & \textbf{P 13}\\ 
			\hline
			
			Normalized discounted \newline difference ($rND$) \cite{yang2016measuring} & $\mathlarger{1} - \tfrac{1}{Z_\text{ND}(D)}\cdot \sum_{k\in I} b(k)\cdot \left|p_{G_1}^k(r) - p_{G_1}(r)\right|$ & difference in proportion of protected group in top-$k$ and overall ranking 
			&\cross & \gcheck & \cross & \cross & N/A & N/A & \cross & \cross & \cross & \cross &\cross &\cross &\cross \\ \rowcolor[gray]{.9}
			
			Normalized discounted \newline ratio ($rRD$) \cite{yang2016measuring} &
			$\mathlarger{1} - \frac{1}{Z_\text{RD}(D)} \cdot \sum_{k \in I} b(k) \cdot  \left|\tfrac{p_{G_1}^k(r)}{p_{G_0}^k(r)} - \tfrac{p_{G_1}(r)}{p_{G_0}(r)}\right|$
			& difference of ratios of protected to non-protected candidates in top-$k$ and overall ranking 
			& \cross & \gcheck & \cross & \cross & N/A & N/A & \cross & \cross & \cross & \cross & \cross & \cross & \cross \\ 
			
			Normalized discounted \newline KL-divergence ($rKL$) \cite{yang2016measuring} & $\mathlarger{1} - \frac{1}{Z_\text{KL}(D)} \cdot \sum_{k\in I} b(k) \cdot \sum_{j\in\{0,1\}}p^k_{G_j}(r)\cdot \log\tfrac{p_{G_j}^k(r)}{p_{G_j}(r)} $ & expectation of difference in proportion of protected group in top-$k$ and overall ranking 
			&\cross & \gcheck & \cross & \cross & N/A & N/A & \cross & \cross & \cross & \cross  & \cross & \cross & \cross \\ \rowcolor[gray]{.9}
			
			Exposure difference ($ED$) \cite{Singh2018} & $\operatorname{Exposure}(G_1|r) - \operatorname{Exposure}(G_0|r)$ 
			& difference between group exposures & \gcheck & \gcheck & \gcheck & \gcheck & N/A & N/A & \gcheck & \cross & \cross & \cross & \gcheck & \gcheck & \gcheck \\ 
			
			Exposure ratio ($ER$) \cite{Singh2018} & $\frac{\operatorname{Exposure}(G_1|r)}{\operatorname{Exposure}(G_0|r)}$ & ratio of group exposures 
			& \gcheck & \cross & \gcheck & \gcheck & N/A & N/A & \cross & \cross & \cross & \cross & \gcheck & \gcheck & \gcheck \\ \rowcolor[gray]{.9}
			
			Disparate treatment \newline difference ($DTD$) \cite{Singh2018} & $\frac{\operatorname{Exposure}(G_1|r)}{Y(G_1)} - \frac{\operatorname{Exposure}(G_0|r)}{Y(G_0)}$ & difference of ratios of group exposure to group relevance 
			& \gcheck & \cross & \gcheck & \gcheck & \cross & \cross & \gcheck & \cross & \cross & \cross &\gcheck & \gcheck & \gcheck \\ 
			
			Disparate treatment \newline ratio ($DTR$)\cite{Singh2018} & $\frac{\operatorname{Exposure}(G_1|r)}{\operatorname{Exposure}(G_0|r)}\cdot \frac{Y(G_0)}{Y(G_1)}$ & quotient of ratios of group exposure to group relevance 
			& \gcheck & \cross & \gcheck & \gcheck & \cross & \cross & \cross & \cross & \cross & \cross  & \gcheck & \gcheck & \gcheck \\ 	\rowcolor[gray]{.9}		
			
			Disparate impact \newline difference ($DID$) \cite{Singh2018} & $\frac{CTR(G_1|r)}{Y(G_1)} - \frac{CTR(G_0|r)}{Y(G_0)}$ & difference of ratios of group click-through rate to group relevance 
			& \gcheck & \cross & \gcheck & \gcheck & \gcheck & \cross & \gcheck & \cross & \cross &\cross & \gcheck & \gcheck & \gcheck \\
			
			Disparate impact ratio ($DIR$) \cite{Singh2018} & $\frac{CTR(G_1|r)}{CTR(G_0|r)}\cdot \frac{Y(G_0)}{Y(G_1)}$ & quotient of ratios of group click-through rate to group relevance 
			& \gcheck &\cross & \gcheck & \gcheck & \gcheck & \cross & \cross & \cross & \cross & \cross & \gcheck & \gcheck & \gcheck \\ \rowcolor[gray]{.9}
			
			Attention-weighted Rank\newline Fairness ($AWRF$) \cite{sapiezynski2019quantifying} & $\mathlarger{1} - \tfrac{1}{2}\cdot\left[\mathlarger{\Delta_\text{KL}}\left(\pexp(r)\big\| p_\text{sum}(r)\right) + \cdot\mathlarger{\Delta_\text{KL}}\left(p_\text{groups}\big\| p_\text{sum}(r)\right)\right]$ & Jensen-Shannon divergence between average group exposure and relative group sizes
			& \cross &\gcheck & \cross & \cross & N/A & N/A & \cross & \cross & \cross & \cross & \cross & \cross & \cross \\ 
			
			Pairwise statistical \newline parity ($PSP$) \cite{narasimhan2020pairwise} & $\frac{\big|\{(d, d')\in G_0 \times G_1| d' \succ_r d\}\big| - \big|\{(d, d')\in G_0 \times G_1 | d \succ_r d'\}\big|}{\big| G_0 \times G_1\big|} $ & difference in chances of one group being ranked higher than the other 
			& \gcheck & \gcheck & \gcheck & \cross & N/A & N/A & \gcheck & \gcheck &  \gcheck &  \gcheck  & N/A & N/A & N/A\\
			\hline
		\end{tabularx}
		}
	\caption{\emph{Overview of existing group fairness metrics in rankings.} For each of the group fairness metrics considered in this work, we give the definition along with a brief description, and illustrate which of the four properties they satisfy. 
	We find that the majority of of existing group fairness metrics only satisfy a fraction of our properties.
	In particular, established metrics such as \rnd{}, \rrd{}, \rkl{} and \awrf{} do not satisfy any of the given properties (\cross) aside from property 2 (\gcheck).
	Conversely, \psp{} is the only metric that satisfies properties 8-10, but this metric is also the only metric not applicable in Setting 2.
	 N/A marks cases where a property is not applicable for a given metric.}
	\label{table:overview}
\end{sidewaystable}

\section{Discussion}
We present an overview of which properties are satisfied by the reviewed group fairness metrics in Table \ref{table:overview}.
Our analysis shows that most of these metrics satisfy only a subset of the proposed properties.
In particular, well-established metrics such as the prefix metrics or \awrf{} only satisfy one property.
We note that a metric does not need to satisfy all the properties in order to be considered appropriate for a given ranking context.
However, not satisfying specific properties can have practical implications when interpreting or comparing results. 
The choice of a ranking metric for a problem at hand should therefore consider the properties that are deemed relevant for the given context. 
In the following, we discuss potential implications of our analysis and point out limitations of our study.

\subheader{Universal Properties.}
Overall, the prefix metrics and \awrf{} only appear to satisfy property 2 (\emph{boundedness}).
Furthermore, these are the only metrics that do not satisfy property 1 (\emph{distinguishability of groups}).
Hence, if any of these metrics generates a non-optimal value, it is difficult to conclude if the protected or the non-protected group is at an advantage.
However, a metric can only practically be used as (part of a) loss function in optimizing for fairness, if the optimal value $\vopt(m)$ is also its minimum or maximum, which is one of the primary objective when designing these metrics~\cite{yang2016measuring}.
By taking away the absolute values within the sums of \rnd{} and \rrd{}, these metrics could be 
made to satisfy \emph{distinguishability of groups}, but that may lead to violation of \emph{boundedness}.

Concerning the other universal properties, the prefix metrics and \awrf{} do not satisfy \emph{monotonicity} and \emph{deepness} either.
This is in principle due to their usage of absolute differences or Jensen-Shannon divergence as internal distance functions, which mathematically counters  design choices such as using position bias to assign more weight to higher ranking positions.

In contrast to the prefix metrics, exposure metrics generally satisfy properties 1-4 but not \emph{boundedness}, except for \ed{}. 
\dtd{}, \dtr{}, \did{}, and \dir{} are also the only metrics that consider relevance scores, but only the latter two satisfy property 5 (\emph{intra-group fairness}). 
Hence, when it is also of concern to rank items of the same group in concordance with their relevance, these metrics should not be used. 

None of the exposure metrics satisfy property 6 (\emph{invariance to linear transformation of relevance scores}). Hence, standard normalization operations such as min-max scaling or z-score normalization (often used in practice~\citep{zehlike2020matching}) should only be deployed with proper justifications. 
Moreover, our experiments indicate that in some cases, such a transformation could even alter the score from indicating an unfair outcome to indicating fair outcome for the protected group.  
This opens a new discussion on how ground truth relevance scores should be gathered in the first place.
Given that they can fundamentally influence the fairness assessment of a metric, exact relevance scores and value ranges should be consciously chosen and justified. 

\subheader{Properties for Setting 1.}
For Setting 1 in which the full candidate population is ranked, \psp{} is the only metric that does satisfy any of the given properties aside from property 7 (\emph{optimality of a random ranking}).
This again has strong implications for practical use cases.
When optimizing fair ranking models based on such metrics, the corresponding loss functions will be particularly sensitive for settings where the range of values of a metric is particularly large, e.g., when there are very few protected instances in the candidate set.
Therefore, for this specific setting, \psp{} might be best suited as an evaluation metric, unless \emph{deepness} is also a requirement.
In that regard, we suspect that the choice of the inverse logarithm as a position bias function \eqref{eq:posbias} that ensures satisfaction of \emph{deepness} is a large factor in skewing the values of these ranking metrics so that they do not satisfy properties 8-10. Further investigations are necessary in this direction.

Further, in line with the findings of \citet{kuhlman2021measuring}, our experiments on properties 8-10 demonstrate that several metrics such as \rnd{}, \rkl{}, or the exposure metrics are particularly sensitive to cases where the minority group is at a disadvantage. 
The fact that disadvantaging such a minority group can result in fairness values close to the optimum may, in practice, lead to a lack of recognition of discrimination against them. 
Since in the real world minorities often correspond to groups that already experience discrimination, this effect requires particular attention.

\subheader{Properties for Setting 2.}
We found that in Setting 2, when only a subset of a candidate population is ranked, \psp{} is not applicable and both the prefix metrics and \awrf{} do not satisfy any properties.
In contrast, all exposure metrics satisfy each of the properties for Setting 2.
This behavior bears some similarity to the universal properties, and overall suggests that the exposure metrics might be more suitable for such application scenarios than the prefix metrics or \awrf{}.

\subheader{Limitations.}
We have deliberately restricted ourselves to the setting of a binary sensitive attribute, in which only one group needs protection. Moreover, we only considered the simple setting of single deterministic rankings, whereas given metrics are often applied in context of stochastic rankings.
We argue that these issues however pose no real restriction, since our properties as well as our findings for the binary setting would mostly generalize to non-binary settings as well.
Further, when measuring performance of stochastic rankings, one computes expected fairness values over a given metric, where performance of single rankings are ultimately averaged.
Therefore, we argue that our findings regarding inconsistencies in the values of such metrics are particularly relevant for this scenario. 

We have focused on group fairness metrics, given a vast majority of the existing metrics are based on this notion \cite{Zehlike2022a, Zehlike2022b}. 
An obvious extension would be to include other notions of fairness (e.g., individual, causality-driven). 
Given these metrics are fundamentally different, we believe such extensions call for separate research efforts and are beyond the scope of this paper.

\section{Related Work}

We provide a brief overview of related research on metrics for fairness in rankings.

\noindent\textbf{Metrics for Fairness in Rankings.} 
In addition to the metrics for assessing group fairness which have been presented in Section \ref{sec:ex_metrics}, there are also several metrics that evaluate rankings based on other definitions of fairness which we briefly summarize here.
\emph{Equity of attention}~\cite{Biega2018}, which is based on the notion of individual fairness, requires that the  ranked subjects receive attention that is proportional to their relevance in a given search task. 
The notion of \emph{equal error rates}, which stipulates that a ranking model attains similar error rates for each group, is represented by metrics such as pairwise group accuracy~\cite{narasimhan2020pairwise, beutel2019fairness}. 
Finally, fairness can also be evaluated by means of causal inference, and this approach has been taken by \citet{wu2018discrimination}.    

\subheader{Comparative Evaluations of Metrics for Information Retrieval Problems.} 
Given that evaluation plays a key role in determining the efficiency of information retrieval systems, several attempts have been made towards comparing different evaluation metrics, albeit in context of the quality and diversity of the retrieved items~\cite{amigo2009comparison,hofmann2016online,sakai2007reliability,amigo2013general, amigo2018axiomatic, amigo2019comparison, parapar2021towards}.
\citet{amigo2013general} point out a set of general constraints that any evaluation measure concerning tasks like document clustering, filtering, ranking or retrieval should satisfy. 
As noted above, the \emph{deepness}, \emph{closeness threshold}, \emph{deepness threshold}, and \emph{sensitivity} properties are also directly adapted from their work.
Similarly, \citet{amigo2018axiomatic} introduce constraints which concern the metrics for quantifying diversity and quality of the retrieval results. \citet{parapar2021towards} further develop diversity and relevance constraints for evaluation metrics in the context of recommender systems. 
With respect to document filtering, \citet{amigo2019comparison} provide a comparative study of evaluations metrics through a set of defined properties such as \emph{monotonicity}, which stipulates that relabeling an item to the correct category should increase the score, and was also adapted to the group fairness setting in our work. 
In the context of argument retrieval, \citet{pathiyan2021evaluating} also provide a comparative analysis of fairness metrics for this problem.

\subheader{Comparative Evaluations of Metrics for Fairness in Rankings.} 
In our work, we primarily focus on analyzing existing metrics for group fairness in rankings. 
In this regard, the works by \citet{kuhlman2021measuring} and \citet{raj2022measuring} are the closest to our work. 
\citet{raj2022measuring} qualitatively compare different metrics with regard to scenarios such as their ability to handle missing data (e.g., relevance scores, group memberships) or dealing with edge cases among others. 
Their empirical evaluation further demonstrates that the metrics do not show clear consensus across datasets. 
Additionally, they perform a sensitivity analysis of the metrics with respect to their parameter settings as well as design choices which include size of ranking.
\citet{kuhlman2021measuring} compare the behavior of fairness metrics in expectation over distributions of rankings characterized by functions of group advantage.
Additionally, they provide recommendations for practitioners when selecting a suitable fairness metric for a given use case. 
In contrast to the existing work, our comparison of the metrics is through the lens of a set of properties which are inspired from real-world scenarios, including the size of rankings as explored by \citet{raj2022measuring}. We further provide a detailed theoretical analysis of the metrics with respect to their satisfiability of the proposed properties.   
Such comparative analysis, to the best of our knowledge, has largely remained unexplored in the existing literature. 

\section{Conclusion}

This work proposes 13 properties of group fairness metrics for rankings and evaluates to what extent existing metrics satisfy them. 
These properties are categorized based on settings where (i) a full population of candidates is ranked, and (ii) only a subset of a bigger population ends up being ranked.
Through extensive empirical as well as theoretical analysis, we demonstrate that existing fairness metrics only satisfy a subset of these properties. 
Our work not only highlights limitations of existing metrics, but also contributes to assessing and interpreting different fairness metrics for rankings when deployed in real-world applications. 
Given the plethora of available ranking fairness metrics and the unavailability of principled ways for selecting fairness metrics, our work could also act as a guide for practitioners in choosing appropriate metrics for evaluating fairness in a given application scenario.

\bibliographystyle{plainnat}
\bibliography{references}

\appendix
\section{Table of Notations}
\vspace{0.5cm}

\begin{center}
    \begin{tabular}{|c|l|}
    \toprule
        $b(k)$ & position bias at rank $k$, defined as $b(k)=\tfrac{1}{\log_2(k+1)}$\\
        $d$ & candidate\\
        $D$ & candidate set \\
        $\D$ & population of candidates \\
        $D_N$ & candidate set where $|D_N| = 2N$, $|D_N\cap G_1| = 1$, and $y(d) = 1$ for all $d\in D_N$ \\ 
        $D_N'$ & candidate set where $|D_N'| = 2N$, $|D_N'
        \cap G_1| = N$, and $y(d) = 1$ for all $d\in D_N'$ \\ 
        $G_0$ & non-protected group\\
        $G_1$ & protected group\\
        $p_{G_0}(\D)$ & relative proportion of non-protected group in population $\D$\\
        $p_{G_1}(\D)$ & relative proportion of protected group in population $\D$\\
        $p_{G}^k(r)$ & relative proportion of candidates from group $G\in\{G_0,G_1\}$ in top-$k$ positions of ranking $r$ \\
        $p_\text{groups}(\D)$ & vector of relative group proportions  $\left(p_{G_0}(\D), p_{G_1}(\D)\right)$ \\
        $m$ & ranking metric\\
        $m(r)$ & fairness score of ranking metric $m$ on ranking $r$\\
        $n$ & cardinality of candidate set/length of ranking\\
        $q$ & query\\
        $\mathcal{Q}$ & set of queries\\
        $r$ & ranking\\
        $r(k)$ & candidate at position $k$ in ranking $r$\\
        $r^{-1}(d)$ & rank of candidate $d$ in ranking $r$\\
        $\succ_r$ & ranking relation, if $d$ is \textit{ranked higher} than $d'$ we write $d \succ_r d'$ \\
        $r_{i \leftrightarrow j}$ & ranking resulting from swapping the candidates at positions $i$ and $j$ in a ranking $r$\\
        $\mathcal{R}(D)$ & set of all rankings on candidate set $D$\\
        $\mathcal{R}_{\operatorname{first}}(D)$ & set of all rankings on $D$ where every candidate $d\in G_1$ is ranked higher than every $d\in G_0$\\
        $\mathcal{R}_{\operatorname{last}}(D)$ & set of all rankings on $D$ where every candidate $d\in G_1$ is ranked lower than every $d\in G_0$\\
        $\vopt(m)$ & optimal value of fairness metric $m$\\
        $\vf(m,D)$ & value that metric $m$ takes on every ranking $r\in\mathcal{R}_{\operatorname{first}}(D)$ \\
        $\vl(m,D)$ & value that metric $m$ takes on every ranking $r\in\mathcal{R}_{\operatorname{last}}(D)$ \\
        $\mathbf{x}$ & feature vector of a candidate\\ 
        $y(d), y(\mathbf{x})$ & relevance score of a candidate $d$ with feature vector $\mathbf{x}$\\
        $Y(G)$ & average relevance of all candidates from group $G$  \\\bottomrule      
    \end{tabular}
\end{center}

\newpage
\section{Supplementary Proofs for Theoretical Analysis}\label{apdx:proofs}

In the following, we provide detailed proofs to the theorems given in Section \ref{sec:theory}.

\prefixthm*
\begin{proof}
We consider a population $\D$ where $p_{G_1}(\D) = 0.8$, and, for the \emph{deepness threshold}, consider a minimal example with the index set $I=\{N\}$.
Then for each $r'\in\mathcal{R}_{\operatorname{last}}\big(D_N'\big)$ it holds that
$$
\rnd\big(r'\big(D_N'\big)\big) = \rrd\big(r'\big(D_N'\big)\big) = \rkl\big(r'\big(D_N'\big)\big) = 0,
$$
since the rankings $r'\big(D_N'\big)$ maximize the distance terms at the single cut-off point $N$  due to the proportion of the protected group being $p^N_{G_1}\big(r'\big(D_N'\big)\big)=0$. 
Conversely, in each ranking $r\in\mathcal{R}_{\operatorname{first}}(D_N)$, we have $p^N_{G_1}(r)=\frac{1}{N}$, which does not maximize the distance terms in the prefix metrics. 
In consequence, each of the prefix metrics $m\in\{\rnd{},\rrd{},\rkl{}\}$ will have a value $m(r(D_N))>0$ for all $N\in\N$, and therefore, these metrics do not satisfy property 12.

Regarding \emph{sensitivity}, we again consider a ranking $r'\in\mathcal{R}_{\operatorname{last}}\big(D_N'\big)$ in the same population $\D$ with $I=\{N\}$. When appending a single candidate $d\in G_0$ at the end of this ranking, this will not be detected by the prefix metrics due to both groups having $N$ candidates contained in the ranking.
Therefore, these metrics do not satisfy \emph{sensitivity} either.
\end{proof}

\monodeepthm*
\begin{proof}
Regarding \emph{deepness}, due to the uniform relevance assumption we just have to consider \ed{} and \er{}.
Since these metrics just consider differences or ratios in group-wise averaged exposure, which is measured in terms of position bias $b(k)$,
the claim follows due to $b(k)$ being strictly monotonically decreasing and convex.

Regarding \emph{monotonicity}, we note that swapping a pair of candidates $r(i) \in G_0$ and $r(j) \in G_1$ with $i<j$ and $y(r(i)) \leq y(r(j))$ will not affect the average group relevance scores $Y(G_0)$ and $Y(G_1)$, whereas by monotonicity of the position bias 
the average exposure of the protected group $\operatorname{Exposure}(G_1|r)$ will increase and the average exposure of the non-protected group $\operatorname{Exposure}(G_0|r)$ will decrease.
Thus, for \ed{}, \er{}, \dtd{}, and \dtr{} it directly follows that these metrics satisfy \emph{monotonicity}.
Regarding the click-through-rate ($CTR$), it holds that
\begin{talign}
         CTR(G|r) < CTR(G| \rij) \nonumber
    \Leftrightarrow& \,   \frac{1}{|G|} \sum_{d \in G\cap D} b\big(r^{-1}(d)\big) \cdot y(d) < \frac{1}{|G|} \sum_{d \in G\cap D} b\big(\rij^{-1}(d)\big) \\ \nonumber
    \Leftrightarrow&\,  b(i)\cdot y(r(i)) +  b(j)\cdot y(r(j)) < b(j)\cdot y(r(i)) +  b(i)\cdot y(r(j)) \\ \nonumber
    \Leftrightarrow&\,  y(r(i))\cdot \big[b(i) - b(j)\big] < y(r(j))\cdot \big[b(i) - b(j)\big] \\
    \Leftrightarrow &\,  y(r(i)) < y(r(j)), \label{eq:ctr}
\end{talign}
and thus the $CTR$ increases if and only if candidates are swapped in concordance with their relevance scores.
Therefore, the \dtd{} and \dtr{} satisfy \emph{monotonicity} as well.
\end{proof}

\intragroupthm*

\begin{proof}
It is trivial that both the $\operatorname{Exposure}(G|r)$ and $Y(G)$ terms are not affected when swapping candidates $r(i)$ and $r(j)$ that are from the same group $G$. 
Thus, \dtd{} and \dtr{}, which just compute ratios and differences of these values, are invariant to such swaps.
Regarding the click-through-rate ($CTR$), by Equation \eqref{eq:ctr} the $CTR$ increases when switching candidates in concordance with their relevance scores.
Thus, by the definitions of \did{} and \dir{}, and the invariance of $Y(G)$, it follows that indeed these metrics satisfy \emph{intra-group fairness}.
\end{proof}

\scalethm*

\begin{proof}
Let $r$ denote an arbitrary ranking on a candidate set $D\in\mathcal{D}$, and $f_{a,0}(Y(G))$ denote the average relevance scores of group $G$ after rescaling according to $f_{a,0}$.
Then it holds that 
$${\textstyle
f_{a,0}(Y(G)) = \frac{1}{|G|}\cdot \sum_{d \in G} f_{a,0}(y(d))= \frac{1}{|G|}\cdot \sum_{d \in G} a\cdot y(d) = a \cdot Y(G).
}
$$
Thus, it follows that
\[\textstyle
DTD\big(r\big(f_{a,0}(D)\big)\big) = \frac{\operatorname{Exposure}(G_1|r)}{f_{a,0}(Y(G_1))} - \frac{\operatorname{Exposure}(G_0|r)}{f_{a,0}(Y(G_0))} 
 = \frac{\operatorname{Exposure}(G_1|r)}{a\cdot Y(G_1)} - \frac{\operatorname{Exposure}(G_0|r)}{a\cdot Y(G_0)}
  = \frac{1}{a}\cdot DTD(r(D)).
\]
Conversely, the \dtr{} is \emph{invariant to rescaling of relevance scores} since
\[
\textstyle
    DTR\big(r\big(f_{a,0}(D)\big)\big) = \frac{\operatorname{Exposure}(G_1|r)}{\operatorname{Exposure}(G_0|r)}\cdot \frac{f_{a,0}(Y(G_0))}{f_{a,0}(Y(G_1))} 
    = \frac{\operatorname{Exposure}(G_1|r)}{\operatorname{Exposure}(G_0|r)}\cdot \frac{aY(G_0)}{a Y(G_1)} 
    = \frac{\operatorname{Exposure}(G_1|r)}{\operatorname{Exposure}(G_0|r)}\cdot \frac{Y(G_0)}{ Y(G_1)} 
   = DTR(r(D)).
\]
For \did{} and \dir{}, the factor $a$ cancels out analogously in the fractions of relevance scores and click-through rates.
\end{proof}

\randomthm*
\begin{proof}
Let $R:=R(\D)$ denote the random variable that draws uniformly from all rankings $r\in\mathcal{R}(\D)$, and $\mathds{1}_{G}$ the indicator function for a group $G$. By linearity of the expected value, for each group $G\in\{G_0,G_1\}$ it holds that
\begin{talign*}
    E_{r\sim R}\big(\operatorname{Exposure}(G|r)\big) &= E_{r\sim R}\left(\frac{1}{|G|}\cdot\sum_{d \in G\cap \D} b\big(r^{-1}(d)\big)\right) = E_{r\sim R}\left(\frac{1}{|G|}\cdot\sum_{k = 1}^{|\D|} \mathds{1}_{G}(r(k))\cdot b(k)\right) \\
    &= \frac{1}{|G|}\cdot\sum_{k = 1}^{|\D|} E_{r\sim R}\big(\mathds{1}_{G}(r(k))\big)\cdot b(k) 
    = \frac{1}{|G|}\cdot\sum_{k = 1}^{|\D|} p_{G} \cdot b(k) 
    = \frac{1}{|\D|} \cdot\sum_{k = 1}^{|\D|} b(k).
\end{talign*}
Thus, once again using linearity of the expected value, it directly follows that 
$$
E_{r\sim R} \big(ED(r)\big) = E_{r\sim R}\big(\operatorname{Exposure}(G_1|r) - \operatorname{Exposure}(G_0|r)\big) = 0 = \vopt(ED).
$$
For \er{}, we can provide a minimal example for which this property is not satisfied. 
Let $\D = \{d_0, d_1\}$ with $d_0\in G_0$ and $d_1\in G_1$.
There are exactly two rankings in $\mathcal{R}(\D)$, namely $r_0 = \langle d_0,d_1\rangle$, and $r_1 = \langle d_1,d_0\rangle$.
Then it holds that
$$
E_{r\sim R} \big(ER(r)\big) = \tfrac{1}{2} \cdot \tfrac{\operatorname{Exposure}(G_1|r_0)}{\operatorname{Exposure}(G_0|r_0)} + \tfrac{1}{2} \cdot \tfrac{\operatorname{Exposure}(G_1|r_1)}{\operatorname{Exposure}(G_0|r_1)}
= \tfrac{1}{2} \cdot \left(\tfrac{b(2)}{b(1)} + \tfrac{b(1)}{b(2)}\right)
\approx 1.11 > 1 = \vopt(ER),
$$
and thus \er{} does not satisfy property 7.
\end{proof}

\expothm*

\begin{proof}
Since we assume uniform relevance, we only need to prove the assertion for \ed{} and \er{}.
Regarding the \emph{closeness threshold} and \emph{deepness threshold}, for all $r\in\mathcal{R}_{\operatorname{first}}(D_N)$, $r'\in\mathcal{R}_{\operatorname{last}}(D_N')$ holds that
\begin{talign*}
    ED(r) = \frac{1}{|G_1|} - \frac{1}{|G_0|}\cdot\sum_{k=2}^{2N} \frac{1}{\log(k+1)}, \quad  
    & ED(r') = \frac{1}{|G_1|} \cdot \sum_{k=N+1}^{2N} \frac{1}{\log(k+1)} - \frac{1}{|G_0|}\cdot\sum_{k=1}^{N} \frac{1}{\log(k+1)}, \\
    ER(r) = \frac{|G_0|}{|G_1|\cdot \sum_{k=2}^{2N} \frac{1}{\log(k+1)}}, \quad\text{and}\quad 
    & ER(r') = \frac{|G_0|\cdot \sum_{k=N+1}^{2N} \frac{1}{\log(k+1)}}{|G_1|\cdot \sum_{k=1}^{N} \frac{1}{\log(k+1)}}.
\end{talign*}
Therefore, we obtain
\begin{talign*}
    ED(r) \geq ED(r') 
    \Leftrightarrow &\,\frac{1}{|G_1|} - \frac{1}{|G_0|}\cdot\sum_{k=2}^{2N} \frac{1}{\log(k+1)} \geq \frac{1}{|G_1|}\cdot \sum_{k=N+1}^{2N} \frac{1}{\log(k+1)} - \frac{1}{|G_0|}\cdot\sum_{k=1}^{N} \frac{1}{\log(k+1)} \\
    \Leftrightarrow &\, \frac{1}{|G_1|} + \frac{1}{|G_0|} - \frac{1}{|G_0|}\cdot \sum_{k=N+1}^{2N} \frac{1}{\log(k+1)} \geq \frac{1}{|G_1|} \cdot \sum_{k=N+1}^{2N} \frac{1}{\log(k+1)}  \\
    \Leftrightarrow &\, \frac{1}{|G_1|} + \frac{1}{|G_0|} \geq \left( \frac{1}{|G_1|} + \frac{1}{|G_0|} \right) \cdot \sum_{k=N+1}^{2N} \frac{1}{\log(k+1)} \\
    \Leftrightarrow &\, \mathlarger{1} \geq \sum_{k=N+1}^{2N} \frac{1}{\log(k+1)}, 
\end{talign*}
and this inequality holds for $N=1$, since $1 > \frac{1}{\log(3)}$. Therefore, by choosing $N'=1$, we can conclude that \er{} satisfies property 11. Further, for $N > 2$ it holds that
\begin{equation}\label{eq:deepineq}
\mathsmaller{ \sum_{k=N+1}^{2N} \frac{1}{\log(k+1)} > \frac{N}{\log(2N+2)} = \frac{N}{ \log(N+1) + 1} > 1}, 
\end{equation}
and thus, \ed{} has a \emph{deepness threshold} as well.
Similarly, for \er{} we obtain that 
\begin{talign*}
    ER(r) \geq ER(r') 
    \Leftrightarrow &\,\frac{|G_0|}{|G_1| \sum_{k=2}^{2N} \frac{1}{\log(k+1)}} \geq  \frac{|G_0| \sum_{k=N+1}^{2N} \frac{1}{\log(k+1)}}{|G_1| \sum_{k=1}^{N} \frac{1}{\log(k+1)}}\\
    \Leftrightarrow &\, \sum_{k=1}^{N} \frac{1}{\log(k+1)} \geq \sum_{k=2}^{2N} \frac{1}{\log(k+1)} \cdot \sum_{k=N+1}^{2N} \frac{1}{\log(k+1)}, 
\end{talign*}
which again is satisfied for $N=1$ since $1 > \frac{1}{\log(3)^2}$, and thus \er{} has a \emph{closeness threshold}. 
Conversely, since by \eqref{eq:deepineq} we have that $\sum_{k=N+1}^{2N} \frac{1}{\log(k+1)} > 1$ for $N>2$, it follows that 
$$
\mathsmaller \sum_{k=2}^{2N} \tfrac{1}{\log(k+1)}
\cdot 
\mathsmaller \sum_{k=N+1}^{2N} \tfrac{1}{\log(k+1)}
> {\mathsmaller \sum_{k=1}^{N} \tfrac{1}{\log(k+1)}},
$$
and thereby \er{} satisfies property 12.

Regarding \emph{sensitivity}, we note that adding an additional candidate $d'\in G_0$ to the ranking will always increase the exposure of $G_0$, i.e., $\operatorname{Exposure}(G_0|r') > \operatorname{Exposure}(G_0|r)$. Since adding $d'$ to the ranking will not affect the exposure of group $G_1$, it directly follows that \ed{} and \er{} satisfy property 13.
\end{proof}

To properly prove Theorem \ref{thm:awrf}, we need to establish some initial results.

\begin{proposition}\label{prop:log_ineq}
For all $x\geq 3$ it holds that $x+1 > (x+2)\cdot\log(x+2)-(x+1)\cdot\log(x+1)$.
\end{proposition}

\begin{proof}
Let $g, h: \R_+ \rightarrow \R$ two functions defined as 
$$
g(x):= x+1 \quad \text{and}\quad h(x) := (x+2)\cdot\log(x+2)-(x+1)\cdot\log(x+1).
$$
Then it holds that $h(3)\approx 3.6 < 4 = g(3)$.
Further, both $g$ and $h$ are differentiable, and for all $x \geq 3$ we obtain that
$$
h'(x) = \log(x+2) - \log(x+1) = \log\left(\tfrac{x+2}{x+1}\right) < \log\left(\tfrac{2x+2}{x+1}\right) = \log(2) = 1 = g'(x).
$$
Therefore, the assertion follows.
\end{proof}

\begin{proposition}\label{prop:f_dec}
The function $f: \R_+ \rightarrow \R_+ : x \mapsto \log(x+2) \cdot \left( \frac{1}{\log(2x+2)} + \frac{1}{\log(2x+3)} \right)$ is strictly monotonically increasing for all $x\geq3$.
\end{proposition}

\begin{proof}
The function $f(x) = \log(x+2) \cdot \left( \frac{1}{\log(2x+2)} + \frac{1}{\log(2x+3)}\right) =  \frac{\log(x+2)}{\log(2x+2)} + \frac{\log(x+2)}{\log(2x+3)}$ is differentiable with
\begin{talign*}
f'(x) &= \frac{\frac{1}{x+2}\cdot \log(2x+2) - \log(x+2)\cdot \frac{2}{2x+2}}{\log(2x+2)^2} + \frac{\frac{1}{x+2}\cdot \log(2x+3) - \log(x+2)\cdot \frac{2}{2x+3}}{\log(2x+3)^2} \\
&> \frac{1}{\log(2x+3)^2}\cdot \left( 2\cdot \frac{\log(2x+2)}{x+2} - 2 \cdot \frac{\log(x+2)}{x+1}\right) \\
&= \frac{2}{(x+1)\cdot(x+2)\cdot\log(2x+3)^2} \cdot \left[(x+1)\cdot(1+\log(x+1)) - (x+2)\cdot\log(x+2)\right] \\
&= \underbrace{\frac{2}{(x+1)\cdot(x+2)\cdot\log(2x+3)^2}}_{=:c_1(x)}\cdot \underbrace{\left[x+1 + (x+1)\cdot\log(x+1) - (x+2)\cdot\log(x+2)\right]}_{=:c_2(x)}
\end{talign*}
Since it holds that $c_1(x)>0$ for all $x \geq 0$ and, by Proposition \ref{prop:log_ineq}, also $c_2(x) >0$ for all $x\geq 3$, it also follows that $f'(x) > 0$ for all $x\geq3$ and thus, $f(x)$ is strictly monotonically increasing for all $x\geq 3$.
\end{proof}

\begin{proposition}\label{prop:seq_dec}
The sequence $B_N':= \frac{\sum_{k=1}^N \frac{1}{\log(k+1)}}{\sum_{k=1}^{2N} \frac{1}{\log(k+1)}}$ is strictly monotonically decreasing for all $N\geq 5$.
\end{proposition}

\begin{proof}
We need to show that for all $N\geq5$ it holds that
\[
B_{N+1}' = \tfrac{\sum_{k=1}^N \frac{1}{\log(k+1)} + \frac{1}{\log(N+2)}}{\sum_{k=1}^{2N} \frac{1}{\log(k+1)}+\frac{1}{\log(2N+2)} + \tfrac{1}{\log(2N+3)}} 
< \tfrac{\sum_{k=1}^N \tfrac{1}{\log(k+1)}}{\sum_{k=1}^{2N} \tfrac{1}{\log(k+1)}} = B_N'
\]
which is equivalent to 
\begin{equation}\label{eq:ind_ineq}
\textstyle
    \frac{1}{\log(N+2)} \cdot \sum_{k=1}^{2N} \frac{1}{\log(k+1)} < \left( \frac{1}{\log(2N+2)} + \frac{1}{\log(2N+3)} \right) \cdot \sum_{k=1}^N \frac{1}{\log(k+1)}.
\end{equation}
Thus, we show via mathematical induction that inequality \eqref{eq:ind_ineq} holds for all $N\geq5$. 
The base case can be verified via computer algebra systems - indeed it holds that the left-hand side of \eqref{eq:ind_ineq} is roughly equal to $1.618$, whereas the right-hand side roughly equals $1.619$.
As induction hypothesis (IH), we assume that the claim holds for a given $N\geq 5$, and try to prove the assertion for $N+1$ in the induction step.
It holds that
\begin{talign*}
\frac{1}{\log(N+3)} \cdot \sum_{k=1}^{2N+2} \frac{1}{\log(k+1)} 
&\, = \frac{\log(N+2)}{\log(N+3)}\cdot \frac{1}{\log(N+2)} \cdot \left( \sum_{k=1}^{2N} \frac{1}{\log(k+1)}  + \frac{1}{\log(2N+2)} + \frac{1}{\log(2N+3)}\right) \\
&\, = \frac{\log(N+2)}{\log(N+3)} \cdot \frac{1}{\log(N+2)} \cdot \sum_{k=1}^{2N} \frac{1}{\log(k+1)} + \frac{1}{\log(N+3)} \cdot \left( \frac{1}{\log(2N+2)} + \frac{1}{\log(2N+3)}\right) \\
&\overset{\text{(IH)}}{<} \frac{\log(N+2)}{\log(N+3)} \cdot \left( \frac{1}{\log(2N+2)} + \frac{1}{\log(2N+3)} \right) \cdot \sum_{k=1}^N \frac{1}{\log(k+1)} \\
& \quad + \frac{1}{\log(N+3)} \cdot \left( \frac{1}{\log(2N+2)} + \frac{1}{\log(2N+3)}\right)  \\
&\, = \left( \frac{1}{\log(2N+2)} + \frac{1}{\log(2N+3)} \right) \cdot \left[ \frac{\log(N+2)}{\log(N+3)} \cdot \left(\sum_{k=1}^{N+1} \frac{1}{\log(k+1)} - \frac{1}{\log(N+2)} \right) + \frac{1}{\log(N+3)}\right] \\
&\, = \underbrace{\log(N+2)\cdot\left( \frac{1}{\log(2N+2)} + \frac{1}{\log(2N+3)} \right)}_{=:f(N)} \cdot \frac{1}{\log(N+3)} \cdot \left(\sum_{k=1}^{N+1} \frac{1}{\log(k+1)} \right).
\end{talign*}
By Proposition \ref{prop:f_dec}, it holds that $f(N) < f(N+1)$ for all $N\geq 3$, and therefore we obtain
\begin{talign*}
\frac{1}{\log(N+2)} \cdot \sum_{k=1}^{2N+2} \frac{1}{\log(k+1)}
&< \log(N+3)\cdot\left( \frac{1}{\log(2N+4)} + \frac{1}{\log(2N+5)} \right) \cdot \frac{1}{\log(N+3)} \cdot \left(\sum_{k=1}^{N+1} \frac{1}{\log(k+1)} \right) \\
&= \left( \frac{1}{\log(2N+4)} + \frac{1}{\log(2N+5)} \right) \cdot \sum_{k=1}^{N+1} \frac{1}{\log(k+1)}.
\end{talign*}
Now the assertion follows by the principle of mathematical induction.
\end{proof}

\awrfthm*
\begin{proof}
To show that there is no \emph{deepness threshold} for \awrf{}, we consider a population $\D$ in which we have $p_\text{groups}(\D) = (0.9,0.1)$.
Letting $B_N := \tfrac{1}{\sum_{k=1}^{2N} b(k)}$ and $B_N' := \tfrac{\sum_{k=1}^{N} b(k)}{\sum_{k=1}^{2N} b(k)}$, it is easy to see that 
\[\textstyle
p_\text{Exp}(r(D_N)) =(1-B_N,B_N)\quad
\text{and}\quad
p_\text{Exp}\big(r'\big(D_N'\big)\big) = \big(B_N',1-B_N'\big).
\]
Due to
$
\sum_{k=1}^{2N} b(k) 
> \sum_{k=1}^{2N} \tfrac{1}{k+1}, 
$
which, as harmonic series, diverges, we have that $B_N \xrightarrow{N\to\infty} 0$.
It follows that 
\begin{talign*}
\Delta_\text{KL}\left(p_\text{groups}\| p_\text{sum}\left(r(D_N)\right)\right) 
&= 0.9\cdot\left[\log(0.9) - \log\left(0.5\cdot\left(1-B_N+0.9\right)\right) \right]\\
&\quad + 0.1\cdot\left[\log(0.1) - \log\left(0.5\cdot\left(B_N
+0.1\right)\right)\right]\\
& \xrightarrow{N\to\infty} 0.9\cdot\log(0.9) - 0.9\cdot\log(0.95) +  0.1\cdot\log(0.1) - 0.1\cdot\log(0.05)  
\end{talign*}
and 
\begin{talign*}
\Delta_\text{KL}\big(p_\text{Exp}\left(r(D_N)\big) \| p_\text{sum}\left(r\left(D_N\right)\right)\right) 
&= \left(1-B_N
\right) \cdot \left[\log\left(1-B_N\right) - \log\left(0.5\cdot\left(1-B_N
+0.9\right)\right) \right]
\\
&\quad +B_N\cdot \left[\log\left(B_N\right) - \log\left(0.5\cdot\left(B_N+0.1\right)\right)\right]\\
&\xrightarrow{N\to\infty}   - \log(0.95),
\end{talign*}
where we use that $\lim_{x \to 0} x\cdot\log(x) = 0$.
With these results, we obtain
$$
\lim_{N\to\infty} AWRF(r(D_N)) = 1 - 0.5\cdot[ - \log(0.95) + 0.9\cdot\log(0.9) - 0.9\cdot\log(0.95) +  0.1\cdot\log(0.1) - 0.1\cdot\log(0.05)] \approx 0.948.
$$
Regarding $AWRF\big(r'\big(D_N'\big)\big)$, we again first consider the values of the exposure vector $p_\text{Exp}\big(r'\big(D_N'\big)\big) = \big(B_N',1-B_N'\big)$. Since the position bias $b(k)$ is strictly monotonically decreasing, 
it holds that $B_N' > \tfrac{1}{2}$.
Conversely, by Proposition \ref{prop:seq_dec} it holds that the $B_N'$ is strictly monotonically decreasing for $N\geq 5$, and even more, for $N\geq 5000$ one can compute that $B_N' < 0.55$.
Thus, for $N>5000$ it follows that
\begin{talign*}
\Delta_\text{KL}\left(p_\text{groups}\| p_\text{sum}\big(\big(r'(D_N'\big)\big)\right)
&= 0.9\left[\log(0.9) - \log\left(0.5\left(
B_N'
+0.9\right)\right) \right] + 0.1\left[\log(0.1) - \log\left(0.5\left(1-B_N'+0.1\right)\right)\right]\\
&> 0.9\left[\log(0.9) - \log\left(0.5\left(0.55+0.9\right)\right) \right]
+ 0.1\left[\log(0.1) - \log\left(0.5\left(1-0.5+0.1\right)\right)\right] \\
&= 0.9\cdot[\log(0.9) - \log(0.725)] + 0.1\cdot[\log(0.1) - \log(0.3)] \\
&\approx 0.1223 > 0.12.
\end{talign*}
Further, it holds that $\Delta_\text{KL}\left(p_\text{Exp}\| p_\text{sum}\big(r'(D_N'\big)\big)\right) \geq 0$ by the properties of the Kullback-Leibler divergence, and therefore, for $N>5000$ we finally obtain that
\begin{talign*}
AWRF\left(r_N'\big(D_N'\big)\right) &= 1 - \tfrac{1}{2}\cdot \left[\Delta_\text{KL}\left(p_\text{Exp}\| p_\text{sum}\big(r'\big(D_N'\big)\big)\right) + \Delta_\text{KL}\big(p_\text{groups}\| p_\text{sum}\big(r'\big(D_N'\big)\big)\big)\right] \\
& < 0.94 < \lim_{N \to\infty} AWRF(r(D_N)),
\end{talign*}
from which the claim follows.

Regarding \emph{sensitivity}, consider a population $\mathcal{D}'$ where $p_\text{groups}(\mathcal{D}') = (0.75, 0.25)$, a candidate set $D'\subseteq \mathcal{D}'$ of $n'=2$ candidates, and a ranking $r'\in\mathcal{R}(D')$ with $r'(1)\in G_0$ and $r'(2)\in G_1$. Further, let $d''\in G_0, d''\notin D'$ a candidate that has not been a part of $D'$, and $r''\in \mathcal{R}(D\cup\{d''\})$ a ranking with $r''(i)=r'(i)$ for $i\in\{1,2\}$ and $r''(3)=d''$.
Then it holds that 
$
AWRF(r'') \approx 0.998 > 0.984 \approx AWRF(r'),   
$
and therefore, \emph{sensitivity} is not satisfied.
\end{proof}

\pspthm*
\begin{proof}
Let 
$\pi_{\text{inv}}: \mathcal{R}(\D) \rightarrow \mathcal{R}(\D),\, r := \langle r(1), ..., r(n)\rangle \mapsto \langle r(n), ..., r(1)\rangle
$
denote the function that inverts every ranking on $\D$.
This mapping is bijective, and for every candidate pair $d,d'\in \D$ with $d \succ_r d'$ on a given ranking $r\in\mathcal{R}$, it holds that  $d' \succ_{\pi_{\text{inv}}(r)} d$.
In consequence, for every ranking $r\in\mathcal{R}$ we obtain 

\begin{talign*}
  PSP(r) &= \frac{\big|\big\{(d, d')\in G_0 \times G_1\, :\, d' \succ_r d\big\}\big| - \big|\big\{(d, d')\in G_0 \times G_1 \, :\, d \succ_r d' \big\}\big|}{\big|G_0 \times G_1\big|} \\
         &= \frac{\big|\big\{(d, d')\in G_0 \times G_1 \, :\, d \succ_{\pi_{\text{inv}}b(r)} d'\big\}\big| - \big|\big\{(d, d')\in G_0 \times G_1 \, :\, d' \succ_{\pi_{\text{inv}}(r)} d\big\}\big|}{\big|G_0 \times G_1\big|}  \\
         &= - PSP\big(\pi_{\text{inv}}(r)\big).
\end{talign*}
From this, it ultimately follows that
\begin{talign*}
 v_E(PSP,\D) &= \frac{1}{|\mathcal{R}(\D)|}\cdot \sum_{r\in \mathcal{R}(\D)} PSP(r)
           =\frac{1}{2\cdot|\mathcal{R}(\D)|}\cdot\left(\sum_{r\in \mathcal{R}(\D)}\ PSP(r) + \sum_{r\in \mathcal{R}(\D)}\ PSP\big(\pi_{\text{inv}}(r)\big)\right)\\
         &=\frac{1}{2\cdot|\mathcal{R}(\D)|}\cdot\sum_{r\in \mathcal{R}(\D)}\left(PSP(r) - PSP(r)\right)\ = \vopt(PSP),
\end{talign*}
which concludes the proof.
\end{proof}

\end{document}